\documentclass[11pt]{article}

\usepackage{booktabs}       
\usepackage{amsfonts}       
\usepackage{amsmath}
\usepackage{graphicx}
\usepackage{amsthm}
\usepackage{amssymb}
\usepackage{multirow}
\usepackage{makecell}
\usepackage{algpseudocode}
\usepackage{algorithm}
\usepackage[colorlinks,linkcolor=magenta,filecolor=blue,citecolor=blue,urlcolor=blue]{hyperref}%
\usepackage[font=footnotesize,labelsep=space,labelfont=bf]{caption}

\usepackage[top=1in, left=1in, right=1in, bottom=1in]{geometry}
\usepackage{subfigure}
\usepackage{amssymb}
\usepackage{pifont}
\usepackage{mathtools}

\newcommand{\pl}{Polyak-\L{}ojasiewicz}

\newtheorem{theorem}{Theorem}
\newtheorem{lemma}{Lemma}
\newtheorem{remark}{Remark}
\newtheorem{assumption}{Assumption}

\newtheorem{property}{Property}
\newtheorem{corollary}{Corollary}
\newtheorem{definition}{Definition}

\newtheorem{fact}[theorem]{Fact}

\allowdisplaybreaks

\begin{document}


\title{\huge\texttt{FedSKETCH}: Communication-Efficient and Private Federated Learning via Sketching}

\author{Farzin Haddadpour, Belhal Karimi, Ping Li, Xiaoyun Li\\
Cognitive Computing Lab\\
Baidu Research\\
10900 NE 8th St. Bellevue, WA 98004\\
\{farzin.haddadpour, belhal.karimi, lixiaoyun996, pingli98\}@gmail.com}

\date{}

\maketitle

\begin{abstract}
Communication complexity and privacy are the two key challenges in Federated Learning where the goal is to perform a distributed learning through a large volume of devices.
In this work, we introduce \texttt{FedSKETCH} and \texttt{FedSKETCHGATE} algorithms to address both challenges in Federated learning jointly, where these algorithms are intended to be used for homogeneous and heterogeneous data distribution settings respectively. The key idea is to compress the accumulation of local gradients using count sketch, therefore, the server does not have access to the gradients themselves which provides privacy. Furthermore, due to the lower dimension of sketching used, our method exhibits communication-efficiency property as well.
We provide, for the aforementioned schemes, sharp convergence guarantees.
 Finally, we back up our theory with various set of experiments.
\end{abstract}

\section{Introduction}
Federated Learning is a recently emerging setting for distributed large scale machine learning problems. In Federated Learning, data is distributed across devices (which could be any smartphone or IOT edge device)~\cite{mcmahan2016communication,konevcny2016federated} and due to privacy concerns, users are only allowed to communicate with parameter server. The parameter server orchestrates optimization among devices by aggregating gradient-related information of devices and broadcasts the average of received vectors. Additionally, moving data across the devices for the purpose of learning a global model can be impractical and could violate the privacy of users/devices~\cite{carlini2019secret,mcmahan2017learning}.

There are a number of challenges to be addressed in Federated Learning to efficiently learn a global model that performs well in average for all devices. The first challenge is the \emph{communication-efficiency} as there could be a million of devices communicating iteratively among them which can incur huge communication overhead.
The second challenge is \emph{data heterogeneity}.
Since the data in smartphones or devices are generated locally in Federated Learning, generated data may come from various probability distributions. Thus it is supposed that data distribution is non-iid. It is known that non-iid data distribution can lead to poor convergence error in practice~\cite{li2019federated,liang2019variance}. The last, yet important, issue is \emph{device privacy}~\cite{geyer2017differentially,hardy2017private}. It is important to make sure that the privacy of the sensitive information on each device is preserved during the training.

Almost all of the previous studies consider addressing the aforementioned challenges separately. One approach to deal with communication cost is the idea of \emph{local SGD with periodic averaging}~\cite{zhou2018convergence,stich2019local,yu2019parallel,wang2018cooperative} which asserts that instead of taking the average within each iteration, like baseline SGD~\cite{bottou-bousquet-2008}, one may take the average periodically and performs local update, see local SGD~\cite{lin2019don}. It is shown that local SGD with periodic averaging benefits from the same convergence rate as baseline SGD, while requiring less communication rounds. The second approach to deal with communication cost is aiming at reducing the size of communicated message per each communication round. Available methods reduce the size of the message by communicating compressed local gradients or models to parameter server via quantization~\cite{alistarh2017qsgd,bernstein2018signsgd,tang2018communication,wen2017terngrad,wu2018error}, sparsification~\cite{alistarh2018convergence,lin2017deep,stich2018sparsified,stich2019error}.

There are a number of research efforts such as~\cite{liang2019variance,karimireddy2019scaffold,horvath2019stochastic,haddadpour2020federated} aiming at mitigating the effect of data heterogeneity by exploiting variance reduction or gradient tracking techniques in distributed optimization settings where data distribution is non-iid.

Solving the privacy issue has been widely performed by injecting an additional layer of random noise in order to respect differential-privacy property of the method~\cite{mcmahan2017learning} or using cryptography based approaches under secure multi-party computation~\cite{bonawitz2017practical} framework.

Another promising recent approach with a potential to tackle all major issues in Federated Learning setting is based on sketching algorithms~\cite{DBLP:journals/tcs/CharikarCF04,cormode2005improved,kleinberg2003bursty,Proc:Li_Church_Hastie_NIPS08}. Sketches are built from independent hash tables (functions), needed to compress a high dimensional vector into a lower dimensional one and the corresponding estimation error of sketching are well studied. With the focus of communication-efficiency,~\cite{ivkin2019communication} proposes a distributed SGD algorithm using sketching and they provide the convergence analysis in homogeneous data distribution setting. Also with focus on privacy,  in~\cite{li2019privacy}, the authors derive a single framework in order to tackle these issues jointly and introduce \texttt{DiffSketch} based on the Count Sketch operator. Compression and privacy are performed using random hash functions such that no third parties are able to access the original data. Yet,~\cite{li2019privacy} does not provide the convergence analysis for the \texttt{DiffSketch} in Federated setting, and additionally the estimation error of the \texttt{DiffSketch} is relatively higher than the sketching scheme in~\cite{ivkin2019communication} which might end up in poor convergence error. Finally,~\cite{rothchild2020fetchsgd} considers using sketching technique for Federated Learning in heterogeneous setting from a communication-efficiency perspective. The proposed sketching schemes in~\cite{ivkin2019communication,rothchild2020fetchsgd} are based on a deterministic scheme which requires having access to the exact values of the gradient-related information, thus are not privacy-preserving.

In this work, we provide a thorough convergence analysis for the Federated Learning using sketching for both homogeneous and heterogeneous settings. Additionally, all of our sketching algorithms including a novel scheme, do not need to obtain exact values of gradient, hence are privacy preserving. Therefore, our proposed algorithms based on sketching addresses all the aforementioned three main challenges jointly.

The main contributions of this paper are summarized as follows:
\begin{itemize}
    \item Based on the current compression methods, we provide a new algorithm -- \texttt{HEAPRIX} -- that displays an unbiased estimator of the full gradient we ought to communicate to the central parameter server. We theoretically show that \texttt{HEAPRIX} jointly reduces the cost of communication between devices and server, preserves privacy and is unbiased.

    \item We develop a general algorithm for communication-efficient and privacy preserving federated learning based on this novel compression algorithm.
Those methods, namely \texttt{FedSKETCH} and \texttt{FedSKETCHGATE}, are derived under \textit{homogeneous} and \textit{heterogeneous} data distribution settings.

    \item Non asymptotic analysis of our method is established for convex, \pl\: (generalization of strongly-convex) and nonconvex functions in Theorem~\ref{thm:homog_case} and Theorem~\ref{thm:hetreg_case} for respectively the i.i.d. and non i.i.d. case,  and highlight an improvement in the number of iteration required to achieve a stationary point.

    \item We illustrate the benefits of \texttt{FedSKETCH} and \texttt{FedSKETCHGATE} over baseline methods through a set of experiments. In particular, we plot training loss and accuracy curves depending on the method used for training, the size of the sketches employed and the number of local updates performed at each round of communication. Numerical experiments show the advantages of, in particular, \texttt{FedSKETCH-HEAPRIX} algorithm that achieves comparable test accuracy as Federated SGD (\texttt{FedSGD}) while compressing the information exchanged between devices and server.

\end{itemize}

\section{Related Work}
In this section, we provide a summary of the prior related research efforts as follows:

\vspace{0.1in}\noindent\textbf{Local SGD with Periodic Averaging:}
Compared to baseline SGD where model averaging happens in every iteration, the main idea behind \emph{Local SGD with periodic averaging} comes from the intuition of variance reduction by periodic model averaging~\cite{zhang2016parallel} with purpose of saving communication rounds. While Local SGD has been proposed in~\cite{mcmahan2016communication,konevcny2016federated} under the title of Federated Learning Setting, the convergence analysis of Local SGD is studied in~\cite{zhou2018convergence,yu2019parallel,stich2019local,wang2018cooperative}. The convergence analysis of Local SGD is improved in the follow up works~\cite{haddadpour2019local,haddadpour2019trading,basu2019qsparse,haddadpour2019convergence,bayoumi2020tighter,stich2019error} in majority for homogeneous data distribution setting. The convergence analysis is further extended to heterogeneous setting, wherein studied under the title of \emph{Federated Learning}, with improved rates in~\cite{yu2019linear,li2019convergence,sahu2018convergence,liang2019variance,haddadpour2019convergence,karimireddy2019scaffold}.
Additionally, a few recent Federated Learning/Local SGD with adaptive gradient methods can be found in~\cite{reddi2020adaptive,chen2020toward}.


\vspace{0.1in}\noindent\textbf{Gradient Compression Based Algorithms for Distributed Setting:}~\cite{ivkin2019communication} develop a solution for leveraging sketches of full gradients in a distributed setting while training a global model using SGD~\cite{robbins1951stochastic, bottou-bousquet-2008}. They introduce \texttt{Sketched-SGD} and establish a communication complexity of order $\mathcal{O}(\log(d))$ (per round) where $d$ is the dimension of the vector of parameters, i.e. the dimension of the gradient.
Other recent solutions to reduce the communication cost include quantized gradient as developed in~\cite{alistarh2017qsgd,lin2017deep,stich2018sparsified,horvath2019stochastic}.
Yet, their dependence on the number of devices $p$ makes them harder to be used in some practical settings. Additionally, there are other research efforts such as~\cite{haddadpour2020federated,reisizadeh2020fedpaq,basu2019qsparse,horvath2019stochastic} that exploit compression in Federated Learning or distributed communication-efficient optimization.
Finally, the recent work in~\cite{horvath2020better} jointly exploits variance reduction technique with compression in distributed optimization.

\vspace{0.1in}\noindent\textbf{Privacy-preserving Setting:} Differentially private methods for federated learning have been extensively developed and studied in~\cite{li2019privacy,liu2019enhancing} recently.

\vspace{0.1in}

The remaining of the paper is organized as follows.
Section~\ref{sec:problem} gives a formal presentation of the general problem.
Section~\ref{sec:compression} describes the various compression algorithms used for communication efficiency and privacy preservation, and introduces our new compression method.
The training algorithms are provided in Section~\ref{sec:algos} and their respective analysis in the strongly-convex or nonconvex cases are provided Section~\ref{sec:cnvg-an}. Finally, in Section~\ref{sec:experimnt} we provide empirical results for our proposed algorithms.

\vspace{0.1in}\noindent\textbf{Notation:} For the rest of the paper we indicate the number of communication rounds and number of bits per round per device with $R$ and $B$ respectively. For the rest of the paper we indicate the count sketch of any vector $\boldsymbol{x}$ with $\mathbf{S}(\boldsymbol{x})$. We also denote $[p] =\{1,\dots,p\}$.

\section{Problem Setting}
\label{sec:problem}

The federated learning optimization problem across $p$ distributed devices is defined as follows:
\begin{align}\label{eq:main}
   \min_{\boldsymbol{x}\in \mathbb{R}^{d},\: \sum_{j=1}^pq_j=1} f(\boldsymbol{x})\triangleq \left[\sum_{j=1}^{p}q_jF_j(\boldsymbol{x})\right] \, ,
\end{align}
where $F_j(\boldsymbol{x})=\mathbb{E}_{\xi\in\mathcal{D}_j}\left[L_j\left(\boldsymbol{x},\xi\right)\right]$ is the local cost function at device $j$, $q_j\triangleq\frac{n_j}{n}$ with $n_j$ shows the number of data shards at device $j$ and $n=\sum_{j=1}^pn_j$ is the total number of data samples.
$\xi$ is a random variable with probability distribution $\mathcal{D}_j$, and $L_j$ is a loss function that measures the performance of model $\boldsymbol{x}$.
We note that, while for the homogeneous data distribution, we assume $\mathcal{D}_j$ for $1\leq j\leq p$ have the same distribution across devices and $L_1=L_2=\ldots=L_p$, in the heterogeneous setting these data distributions and loss functions $L_j$ can be different from device to device.

We focus on solving the optimization problem in Eq.~(\ref{eq:main}) for the homogeneous data distribution.
In the heterogeneous setting we consider the special case of $q_1=\ldots=q_p=\frac{1}{p}$.

\section{Count Sketch as a Compression Operation}\label{sec:compression}

A common sketching solution employed to tackle \eqref{eq:main} called \texttt{Count Sketch} ~(for more detail see~\cite{DBLP:journals/tcs/CharikarCF04}) is described Algorithm~\ref{alg:csketch}.
\begin{algorithm}[b]
\caption{\texttt{CS}~\cite{kleinberg2003bursty}: Count Sketch to compress ${\boldsymbol{x}}\in\mathbb{R}^{d}$. }\label{alg:csketch}
\begin{algorithmic}[1]
\State{\textbf{Inputs:} $\boldsymbol{x}\in\mathbb{R}^{d}, t, k, \mathbf{S}_{m\times t}, h_j (1\leq i\leq t), sign_j (1\leq i\leq t)$}
\State{\textbf{Compress vector $\boldsymbol{x}\in\mathbb{R}^{d}$ into $\mathbf{S}\left(\boldsymbol{x}\right)$:}}
\State{\textbf{for} $\boldsymbol{x}_i\in\boldsymbol{x}$ \textbf{do}}
\State{\quad\textbf{for $j=1,\cdots,t$ do}}
\State{\quad\quad $\mathbf{S}[j][h_j(i)]=\mathbf{S}[j-1][h_{j-1}(i)]+\text{sign}_j(i).\boldsymbol{x}_i$ }
\State{\quad\textbf{end for}}
\State{\textbf{end for}}
\State{\textbf{return} $\mathbf{S}_{m\times t}(\boldsymbol{x})$}
\end{algorithmic}
\end{algorithm}
The algorithm for generating count sketching is using two sets of functions that encode any input vector $\boldsymbol{x}$ \textbf{into a hash table} $\boldsymbol{S}_{m\times t}(\boldsymbol{x})$. We use hash functions $\{h_{j,1\leq j\leq t }:[d]\rightarrow m\}$ (which are pairwise independent) along with another set of pairwise independent sign hash functions $\{\text{sign}_{j,1\leq j\leq t}: [d]\rightarrow \{+1,-1\}\}$ to map every entry of $\boldsymbol{x}$ ($\boldsymbol{x}_i, \:1\leq i\leq d$) into $t$ different columns of hash table $\mathbf{S}_{m\times t}$.
These steps are summarized in Algorithm~\ref{alg:csketch}.

\subsection{Unbiased Compressor}
\begin{definition}[Unbiased compressor]
A randomized function, $\text{C}:\mathbb{R}^{d}\rightarrow\mathbb{R}^{d}$ is called an unbiased compression operator with $\Delta\geq 1$, if we have
\begin{align}\notag
\mathbb{E}\left[\text{C}(\boldsymbol{x})\right]&=\boldsymbol{x}\nonumber \quad \textrm{and} \quad    \mathbb{E}\left[\left\|\text{C}(\boldsymbol{x})\right\|^2_2\right] \leq \Delta\left\|\boldsymbol{x}\right\|^2_2\notag \, .
\end{align}
We indicate this class of compressors with $\text{C}\in\mathbb{U}(\Delta)$.
\end{definition}
We note that this definition leads to the following property
\begin{align}\notag
    \mathbb{E}\left[\left\|\text{C}(\boldsymbol{x})-\boldsymbol{x}\right\|^2_2\right]&\leq \left(\Delta-1\right)\left\|\boldsymbol{x}\right\|^2_2\, .
\end{align}
\begin{remark}
Note that if $\Delta=1$ then our algorithm reduces to the case of no compression.
This property allows us to control the noise of the compression.
\end{remark}

\subsection{An Example of Unbiased Compressor via Sketching}
An instance of such unbiased compressor is \texttt{PRIVIX} which obtains an estimate of input $\boldsymbol{x}$ from a count sketch noted $\boldsymbol{S}(\boldsymbol{x})$.
In this algorithm, to query the quantity $x_i$, the $i-th$ element of the vector, we compute the median of $t$ approximated values specified by the indices of $h_j(i)$ for $1\leq j\leq t$. These steps are summarized in Algorithm~\ref{Alg:privix}.

\begin{algorithm}[t]
\caption{\texttt{PRIVIX} \cite{li2019privacy}: Unbiased compressor based on sketching. }\label{Alg:privix}
\begin{algorithmic}[1]
\State{\textbf{Inputs:} $\boldsymbol{x}\in\mathbb{R}^{d}, t, m, \mathbf{S}_{m\times t}, h_j (1\leq i\leq t), sign_j (1\leq i\leq t)$}
\State{\textbf{Query} $\tilde{\boldsymbol{x}}\in\mathbb{R}^d$ \textbf{from $\mathbf{S(\boldsymbol{x})}$:}}
\State{\textbf{for} $i=1,\ldots,d$ \textbf{do}}
\State{\quad\quad ${\tilde{\boldsymbol{x}}}[i]=\text{Median}\{\text{sign}_j(i).\mathbf{S}[j][h_j(i)]:1\leq j\leq t\}$ }
\State{\textbf{end for}}
\State{\textbf{Output:} ${\tilde{\boldsymbol{x}}}$}
\vspace{- 0.1cm}
\end{algorithmic}
\end{algorithm}

Next, we review a few properties of \texttt{PRIVIX} as follows:

\begin{property}[\cite{li2019privacy}]
For the purpose of our proof, we will need the following crucial properties of the count sketch described in Algorithm~\ref{alg:csketch}.
For any real valued vector $\mathbf{x}\in \mathbb{R}^{d}$:
\begin{itemize}
    \item[1)] \emph{Unbiased estimation}: As it is also mentioned in~\cite{li2019privacy}, we have:
    \begin{align}\notag
        \mathbb{E}_{\mathbf{S}}\left[\texttt{PRIVIX}\left[\mathbf{S}\left(\mathbf{x}\right)\right]\right]=\mathbf{x}\, .
    \end{align}
    \item[2)] \emph{Bounded variance}: With $m=\mathcal{O}\left(\frac{e}{\mu^2}\right)$ and $t=\mathcal{O}\left(\ln \left(\frac{d}{\delta}\right)\right)$, we have the following bound with probability $1-\delta$:
    \begin{align}\notag
        \mathbb{E}_{\mathbf{S}}\left[\left\|\texttt{PRIVIX}\left[\mathbf{S}\left(\mathbf{x}\right)\right]-\mathbf{x}\right\|_2^2\right]\leq \mu^2 d\left\|\mathbf{x}\right\|_2^2\, .
    \end{align}
\end{itemize}
\end{property}
Therefore, $\texttt{PRIVIX}\in \mathbb{U}(1+\mu^2 d)$ with probability $1-\delta$.
\begin{remark}
We note that $\Delta=1+\mu^2d$ implies that if $m\rightarrow d$, $\Delta\rightarrow 1+1=2$, which means that the case of no compression is not covered. Thus, the algorithms based on this may converges poorly.
\end{remark}
In the following we provide a review of privacy property of count sketch:

\begin{definition}
A randomized mechanism $\mathcal{O}$ satisfies $\epsilon-$differential privacy, if for input data ${S}_1$ and ${S}_2$ differing by up to one element, and for any output $D$ of $\mathcal{O}$,
\begin{align}\notag
    \Pr\left[\mathcal{O}(S_1)\in D\right]\leq \exp{\left(\epsilon\right)}\Pr\left[\mathcal{O}(S_2)\in D\right] \, .
\end{align}
\end{definition}
 For smaller $\epsilon$, it will become more difficult to specify what is the input for the algorithm $\mathcal{O}$. Hence, smaller $\epsilon$ implies stronger privacy, and we desire to have $\epsilon$ as small as possible to impose stronger privacy guarantees. In the following, we review an assumption from~\cite{li2019privacy} to discuss a property regarding privacy.

\begin{assumption}[Input vector distribution]\label{assu:invecdist}
For the purpose of privacy analysis, similar to \cite{fergus2006removing,gong2014gradient,levin2007image,parmas2018total}, we suppose that for any input vector $S$ with length $|S|=l$, each element $s_i\in S$ is drawn i.i.d. from a Gaussian distribution: $s_i\sim \mathcal{N}(0,\sigma^2)$, and bounded by a large probability:  $|s_i|\leq C, 1\leq i\leq p$ for some positive constant $C>0$.
\end{assumption}

Based on Assumption~\ref{assu:invecdist}, the reference ~\cite{li2019privacy} proves the following:

\begin{theorem}[$\epsilon-$ differential privacy of count sketch,~\cite{li2019privacy}]\label{thm:privacy}
For a sketching algorithm $\mathcal{O}$ using Count Sketch $\mathbf{S}_{t\times m}$ with $t$ arrays of $m$ bins, for any input vector $S$ with length $l$ satisfying Assumption~\ref{assu:invecdist}, $\mathcal{O}$ achieves $t.\ln \left(1+\frac{\alpha C^2 m(m-1)}{\sigma^2(l-2)}(1+\ln(l-m) )\right)-$differential privacy with high probability, where $\alpha$ is a positive constant satisfying $\frac{\alpha C^2 m(m-1)}{\sigma^2(l-2)}(1+\ln(l-m) )\leq \frac{1}{2}-\frac{1}{\alpha}$.
\end{theorem}
The proof of this theorem can be found in~\cite{li2019privacy}.

Theorem~\ref{thm:privacy} implies that if we use smaller hash table either through using smaller $m$ or $t$, we will obtain stronger differential privacy. On the other hand, smaller hash table means bigger estimation error for a compression based on sketching. Therefore, there is an interesting trade-off between communication complexity and obtained privacy.

\subsection{Biased Compressor}
\begin{definition}[Biased compressor]
A (randomized) function,  ${\text{C}}:\mathbb{R}^{d}\rightarrow\mathbb{R}^{d}$ is called a compression operator with $\alpha>0$ and $\Delta\geq 1$, if we have
\begin{align}\notag
    \mathbb{E}\left[\left\|\alpha\boldsymbol{x}-{\text{C}}(\boldsymbol{x})\right\|^2_2\right]\leq \left(1-\frac{1}{\Delta}\right)\left\|\boldsymbol{x}\right\|^2_2\, ,
\end{align}
then, any biased compression operator $C$ is indicated by $C\in \mathbb{C}(\Delta,\alpha)$.
\end{definition}
The following Lemma links these two definitions:
\begin{lemma}[\cite{horvath2020better}]
We have $\mathbb{U}(\Delta)\subset\mathbb{C}(\Delta)$.
\end{lemma}

An instance of biased compressor based on sketching is given in Algorithm~\ref{alg:heavymix}.
\begin{algorithm}[t]
\caption{\texttt{HEAVYMIX}  }\label{alg:heavymix}
\begin{algorithmic}[1]
\State{\textbf{Inputs:} $\mathbf{S}({\mathbf{g}})$; parameter-$m$}
\State{\textbf{Query the vector $\tilde{\mathbf{g}}\in\mathbb{R}^{d}$ from $\mathbf{S}\left(\mathbf{g}\right)$:}}
\State{Query $\hat{\ell}_2^2=\left(1\pm 0.5\right)\left\|\mathbf{g}\right\|^2$ from sketch $\mathbf{S}(\mathbf{g})$}
\State{$\forall j$ query $\hat{\mathbf{g}}_j^2=\hat{\mathbf{g}}_j^2\pm \frac{1}{2m}\left\|\mathbf{g}\right\|^2$ from sketch $\mathbf{S}_{\mathbf{g}}$}
\State{$H=\{j|\hat{\mathbf{g}}_j\geq \frac{\hat{\ell}_2^2}{m}\}$ and $NH=\{j|\hat{\mathbf{g}}_j<\frac{\hat{\ell}_2^2}{m}\}$}
\State{Top$_m=H\cup rand_\ell(NH)$, where $\ell=m-\left|H\right|$}
\State{Get exact values of Top$_m$ }
\State{\textbf{Output:} $\tilde{\mathbf{g}}:\forall j\in\text{Top}_m:\tilde{\mathbf{g}}_{i}=\mathbf{g}_{i}$ and $\forall j \notin\text{Top}_m: \mathbf{g}_{i}=0$}
\end{algorithmic}
\end{algorithm}

\begin{lemma}[\cite{ivkin2019communication}]
\texttt{HEAVYMIX}, with sketch size $\Theta\left(m\log\left(\frac{d}{\delta}\right)\right)$ is a biased compressor with $\alpha=1$ and  $\Delta=d/m$ with probability $\geq1-\delta$. In other words, with probability $1-\delta$, $\texttt{HEAVYMIX}\in C(\frac{d}{m},1)$.
\end{lemma}

We note that Algorithm~\ref{alg:heavymix} is a variation of the sketching algorithm developed in~\cite{ivkin2019communication} with distinction that \texttt{HEAVYMIX} does not require extra second round of communication to obtain the exact values of top$_m$. Additionally, while sketching algorithm based on \texttt{HEAVYMIX} has smaller estimation error compared to \texttt{PRIVIX}, it requires having access to the exact values of top$_m$, therefore such sketching does not benefit from differentially privacy similar to \texttt{PRIVIX}. In the following we introduce our sketching scheme which enjoys from privacy property as well as smaller estimation error.

\subsection{Sketching Based on Induced Compressor}
The following Lemma from~\cite{horvath2020better} shows that we can convert the biased compressor into an unbiased one:
\begin{lemma}[Induced Compressor ~\cite{horvath2020better}]\label{lemm:induced_compress}
For $C_1\in \mathbb{C}(\Delta_1)$ with $\alpha=1$, choose $C_2\in \mathbb{U}(\Delta_2)$ and define the induced compressor with
\begin{align}\notag
    C(\mathbf{x})=C_1(\mathbf{x})+C_2\left(x-C_1\left(\mathbf{x}\right)\right)\, ,
\end{align}
then, the induced compressor $C$ satisfies $C\in\mathbb{U}(\mathbf{x})$ with $\Delta=\Delta_2+\frac{1-\Delta_2}{\Delta_1}$.
\end{lemma}
\begin{remark}
We note that if $\Delta_2\geq 1$ and $\Delta_1\leq 1$, we have $\Delta=\Delta_2+\frac{1-\Delta_2}{\Delta_1}\leq \Delta_2$\, .
\end{remark}
Using this concept of the induced compressor we introduce \texttt{HEAPRIX}:
\begin{algorithm}[t]
\caption{\texttt{HEAPRIX} }\label{alg:heaprix}
\begin{algorithmic}[1]
\State{\textbf{Inputs:} $\boldsymbol{x}\in\mathbb{R}^{d}, t, m, \mathbf{S}_{m\times t}, h_j (1\leq i\leq t), sign_j (1\leq i\leq t)$, parameter-$m$}
\State{\textbf{Approximate $\mathbf{S}(x)$ using \texttt{HEAVYMIX} }}
\State{\textbf{Approximate $\mathbf{S}\left(x - \texttt{HEAVYMIX}[\mathbf{S}(x)]\right)$ using \texttt{PRIVIX} }}
\State{\textbf{Output:} $\texttt{HEAVYMIX}\left[\mathbf{S}\left(\mathbf{x}\right)\right]+\texttt{PRIVIX}\left[\mathbf{S}\left(\mathbf{x}-\texttt{HEAVYMIX}\left[\mathbf{S}\left(\mathbf{x}\right)\right]\right)\right]$}
\end{algorithmic}
\end{algorithm}

\begin{corollary}
Based on Lemma~\ref{lemm:induced_compress} and using Algorithm~\ref{alg:heaprix}, we have $C(x)\in \mathbb{U}(\mu^2 d)$. This shows that unlike \texttt{PRIVIX} the compression noise can be made as small as possible using large size of hash table.
\end{corollary}
\begin{remark}
We highlight that in this case if $m\rightarrow d$, then $C(x)\rightarrow x$ which means that the algorithm convergence can be improved by decreasing the noise of compression (with choice of bigger $m$).
\end{remark}

In the following we define two general framework for different sketching algorithms for homogeneous and heterogeneous data distributions.

\section{Algorithms for Homogeneous and Heterogeneous Settings}\label{sec:algos}
In the following, we first present two algorithms for the homogeneous setting.
Then, we present two other algorithms to deal with data heterogeneity.
We emphasize that, for the sake of privacy in all of our algorithms, the query step is happening locally and the main task of the parameter server is to perform the average of the received messages from the devices and broadcast the average back to the devices.

\subsection{Homogeneous Setting}
In this section, we propose two algorithms for the setting where data across distributed devices are  identically distributed.
The proposed algorithms for Federated Learning leverage sketching techniques to compress communication.
The main difference between the first suggested algorithm and the \texttt{DiffSketch} algorithm in~\cite{li2019privacy} is that we use distinct local and global learning rates. Additionally, unlike~\cite{li2019privacy}, we do not add local Gaussian noise to ensure privacy.

In \texttt{FedSKETCH}, we denote the number of communication rounds between devices and server by $R$, and the number of local updates at device $j$ by $\tau$, which happens between two consecutive communication rounds. Unlike~\cite{haddadpour2020federated}, server node does not store any global model, instead device $j$ has two models, $\boldsymbol{x}^{(r)}$ and $\boldsymbol{x}^{(\ell,r)}_j$ which are local and global models respectively.
At communication round $r$ and device $j$, the local model $\boldsymbol{x}^{(\ell,r)}_j$ is updated using the rule $$\boldsymbol{x}_j^{(\ell+1,r)}=\boldsymbol{x}_j^{(\ell,r)}-\eta \tilde{\mathbf{g}}_j^{(\ell,r)} \qquad\qquad \text{for}\:\:\ell=0,\ldots,\tau-1\, , $$
where $\tilde{\mathbf{g}}_j^{(\ell,r)}\triangleq\nabla{f}_j(\boldsymbol{x}_j^{(\ell,r)},\Xi_j^{(\ell,r)})\triangleq\frac{1}{b}\sum_{\xi\in\Xi_j^{(\ell,r)}}\nabla{L}_j(\boldsymbol{x}_j^{(\ell,r)},\xi)$ is a stochastic gradient of $f_j$ evaluated using the mini-batch $\Xi_j^{(\ell,r)}=\{\xi^{(\ell,r)}_{j,1},\ldots,\xi^{(\ell,r)}_{j,b_j} \}$ of size $b_j$ and $\eta$ is the local learning rate. After $\tau$ local updates locally, model at device $j$ and communication round $r$ is indicated by $\boldsymbol{x}_j^{(\tau,r)}$. The next step of our algorithm is that device $j$ sends the count sketch $\mathbf{S}_j^{(r)}\triangleq\mathbf{S}_j\left(\boldsymbol{x}_j^{(\tau,r)}-\boldsymbol{x}_j^{(0,r)}\right)$ back to the server. We highlight that $$\mathbf{S}_j^{(r)}\triangleq\mathbf{S}_j\left(\boldsymbol{x}_j^{(\tau,r)}-\boldsymbol{x}_j^{(0,r)}\right)=\mathbf{S}_j\left(\eta\sum_{\ell=0}^{\tau-1}\tilde{\mathbf{g}}_j^{(\ell,r)}\right)=\eta\mathbf{S}_j\left(\sum_{\ell=0}^{\tau-1}\tilde{\mathbf{g}}_j^{(\ell,r)}\right)\, ,$$ which is the aggregation of the consecutive stochastic gradients multiplied with local updates $\eta$.

Upon receiving all $\mathbf{S}_j^{(r)}$ from sampled devices, the server computes \begin{align}\mathbf{S}^{(r)}=\frac{1}{k}\sum_{j\in\mathcal{K}^{(r)}}\mathbf{S}_j^{(r)}\label{eq:average-skestching}
\end{align} and broadcasts it to all devices. Devices after receiving $\mathbf{S}^{(r)}$ from server update global model $\boldsymbol{x}^{(r)}$ using rule $$\boldsymbol{x}^{(r)}=\boldsymbol{x}^{(r-1)}-\gamma \texttt{PRIVIX}\left[\mathbf{S}^{(r-1)}\right]\, .$$
We summarize these steps in \texttt{FedSKETCH}, see Algorithm~\ref{Alg:PFLHom}. A variant of this algorithm which uses a different compression scheme, called \texttt{HEAPRIX} is also described in Algorithm~\ref{Alg:PFLHom}. We note that for this variant we need to have an additional communication round between server and worker $j$ to aggregate $\delta_j^{(r)}\triangleq \mathbf{S}_j\left[\texttt{HEAVYMIX}(\mathbf{S}^{(r)})\right]$. Then, server averages all $\delta^{(r)}_j$ and broadcasts to all devices the following quantity:
\begin{align}
\tilde{\mathbf{S}}^{(r)}\triangleq \frac{1}{k}\sum_{j\in\mathcal{K}^{(r)}}\delta^{(r)}_j \, .\label{eq:glbl-updts}
\end{align}
Upon receiving $\tilde{\mathbf{S}}^{(r)}$, all devices compute
\begin{align}\notag
    {\mathbf{\Phi}}^{(r)}\triangleq \texttt{HEAVYMIX}\left[{\mathbf{S}}^{(r)}\right]+\texttt{PRIVIX}\left[{\mathbf{S}}^{(r)}- \tilde{\mathbf{S}}^{(r)}\right] \, ,
\end{align}
where $\boldsymbol{S}^{(r)}$ is computed using Eq.~(\ref{eq:average-skestching}) and then updates its  global model using $\boldsymbol{x}^{(r+1)}=\boldsymbol{x}^{(r)}-\gamma{\mathbf{\Phi}}^{(r)}$.

\begin{remark}[Improvement over~\cite{haddadpour2020federated}]\label{rmrk:bidirect}
An important feature of our algorithm is that due to a lower dimension of the count sketch, the resulting averages ($\mathbf{S}^{(r)}$ and  $\tilde{\mathbf{S}}^{(r)}$) received by the server, are also of lower dimension.
Therefore, these algorithms exploit bidirectional compression in communication from server to device back and forth.
As a result, due to this bidirectional property of communicating sketching for the case of large quantization error shown by $\omega=\theta(\frac{d}{m})$ in~\cite{haddadpour2020federated}, our algorithms outperform \texttt{FedCOM} and \texttt{FedCOMGATE} developed in~\cite{haddadpour2020federated}.
Furthermore, sketching-based server-devices communication algorithm such as ours also provides privacy as a by-product.
\end{remark}

\begin{algorithm}[H]
\caption{\texttt{FedSKETCH}($R$, $\tau, \eta, \gamma$): Private Federated Learning with Sketching. }\label{Alg:PFLHom}
\begin{algorithmic}[1]
\State{\textbf{Inputs:} $\boldsymbol{x}^{(0)}$ as an initial  model shared by all local devices, the number of communication rounds $R$, the number of local updates $\tau$, and global and local learning rates $\gamma$ and $\eta$, respectively}
\State{\textbf{for $r=0, \ldots, R-1$ do}}
\State{$\qquad$\textbf{parallel for device $j\in \mathcal{K}^{(r)}$ do}:}
\State{$\qquad \quad$ \textbf{if PRIVIX variant:} }
\State{$\qquad\quad \quad$ Computes ${\mathbf{\Phi}}^{(r)}\triangleq  {\texttt{PRIVIX}}\left[{\mathbf{S}}^{(r-1)}\right]$ }
\State{$\qquad \quad$ \textbf{if HEAPRIX variant:} }
\State{$\qquad\quad \quad$ Computes ${\mathbf{\Phi}}^{(r)}\triangleq \texttt{HEAVYMIX}\left[{\mathbf{S}}^{(r-1)}\right]+\texttt{PRIVIX}\left[{\mathbf{S}}^{(r-1)}- \tilde{\mathbf{S}}^{(r-1)}\right]$}
\State{$\qquad\quad$ Set $\boldsymbol{x}^{(r)}=\boldsymbol{x}^{(r-1)}-\gamma{\mathbf{\Phi}}^{(r)}$}
\State{$\qquad\quad$ Set $\boldsymbol{x}_j^{(0,r)}=\boldsymbol{x}^{(r)}$ }
\State{$\qquad\quad $\textbf{for} $\ell=0,\ldots,\tau-1$ \textbf{do}}
\State{$\qquad\quad\quad$ Sample a mini-batch $\xi_j^{(\ell,r)}$ and compute $\tilde{\mathbf{g}}_{j}^{(\ell,r)}\triangleq\nabla{f}_j(\boldsymbol{x}^{(\ell,r)}_j,\xi_j^{(\ell,r)})$}
\State{$\qquad\quad\quad$ $\boldsymbol{x}^{(\ell+1,r)}_{j}=\boldsymbol{x}^{(\ell,r)}_j-\eta~ \tilde{\mathbf{g}}_{j}^{(\ell,r)}$ \label{eq:update-rule-alg}}
\State{$\qquad\quad$\textbf{end for}}
\State{$\qquad\quad\quad$Device $j$ sends $\mathbf{S}^{(r)}_{j}\triangleq\mathbf{S}_{j}\left(\boldsymbol{x}_j^{(0,r)}-~{\boldsymbol{x}}_{j}^{(\tau,r)}\right)$ back to the server.}

\State{$\qquad$Server \textbf{computes} }
\State{$\qquad\qquad {\mathbf{S}}^{(r)}=\frac{1}{k}\sum_{j\in\mathcal{K}}\mathbf{S}^{(r)}_{j}$ .}
\State {$\qquad$Server samples a subset of devices $\mathcal{K}^{(r)}$ randomly with replacement and \textbf{broadcasts} ${\mathbf{S}}^{(r)}$ to devices in set $\mathcal{K}^{(r)}$.}
\vspace{0.1cm}
\State{$\qquad$ \textbf{if HEAPRIX variant:} }
\State{$\qquad \quad$ Second round of communication to obtain $\delta_j^{(r)} :=  \mathbf{S}_j\left[\texttt{HEAVYMIX}(\mathbf{S}^{(r)})\right]$ }
\State{$\qquad \quad$ Broadcasts $\tilde{\mathbf{S}}^{(r)}\triangleq\frac{1}{k}\sum_{j\in\mathcal{K}}\delta_j^{(r)}$ to devices in set $\mathcal{K}^{(r)}$}

\State{$\qquad$\textbf{end parallel for}}
\State{\textbf{end}}
\State{\textbf{Output:} ${\boldsymbol{x}}^{(R-1)}$}
\vspace{- 0.1cm}
\end{algorithmic}
\end{algorithm}

\subsection{Heterogeneous Setting}
In this section, we focus on the optimization problem in Eq.~(\ref{eq:main}) in special case of $q_1=\ldots=q_p=\frac{1}{p}$ with full device participation ($k=p$). We also note that these results can be extended to the scenario where devices are sampled, but for simplicity we do not analyze it in this section. In the previous section, we discussed algorithm \texttt{FedSKETCH}, which is originally designed for homogeneous setting where data distribution available at devices are identical. However, in a heterogeneous setting where data distribution could be different, the aforementioned algorithms may fail to perform well in practice. The main reason to cause this issue is that in Federated learning devices are using local stochastic descent direction which could be different than global descent direction when the data distribution are non-identical.

Therefore, to mitigate the effect of data heterogeneity, we introduce new algorithm \texttt{FedSKETCHGATE} based on sketching. This algorithm uses the idea of gradient tracking introduced in~\cite{haddadpour2020federated} (with compression) and a special case of $\gamma=1$ and without compression~\cite{liang2019variance}. The main idea is that using an approximation of global gradient, $\mathbf{c}_j^{(r)}$, we correct the local gradient direction. For the \texttt{FedSKETCHGATE} with \texttt{PRIVIX} variant, the correction vector $\mathbf{c}_j^{(r)}$ at device $j$ and communication round $r$ is computed using the update rule $\mathbf{c}_j^{(r)}=\mathbf{c}_j^{(r-1)}-\frac{1}{\tau}\left({\texttt{PRIVIX}}\left(\mathbf{S}^{(r-1)}\right)-{\texttt{PRIVIX}}\left(\mathbf{S}^{(r-1)}_{j}\right)\right)$ where $\mathbf{S}^{(r-1)}_{j}\triangleq\mathbf{S}\left(\boldsymbol{x}_j^{(0,r-1)}-~{\boldsymbol{x}}_{j}^{(\tau,r-1)}\right)$ is computed and stored at device $j$ from previous communication round $r-1$. The term $\mathbf{S}^{(r-1)}$ is computed similar to \texttt{FedSKETCH} in \eqref{eq:average-skestching}.
For \texttt{FedSKETCHGATE}, the server needs to compute $\tilde{\mathbf{S}}^{(r)}$ using \eqref{eq:glbl-updts}.
Then, device $j$ computes $\mathbf{\Phi}_j\triangleq \texttt{HEAPRIX}[\mathbf{S}_j^{(r)}]$ and $  {\mathbf{\Phi}}\triangleq \texttt{HEAPRIX}(\mathbf{S}^{(r-1)})$ and updates the correction vector $\mathbf{c}_j^{(r)}$ using the recursion $\mathbf{c}_j^{(r)}=\mathbf{c}_j^{(r-1)}-\frac{1}{\tau}\left(\mathbf{\Phi}-\mathbf{\Phi}_j\right)$.

\begin{algorithm}[t]
\caption{\texttt{FedSKETCHGATE}($R$, $\tau, \eta, \gamma$): Private Federated Learning with Sketching and gradient tracking. }\label{Alg:PFLHet}
\begin{algorithmic}[1]
\State{\textbf{Inputs:} $\boldsymbol{x}^{(0)}=\boldsymbol{x}^{(0)}_j$ shared by all local devices, communication rounds $R$, local updates $\tau$, global and local learning rates $\gamma$ and $\eta$.}
\State{\textbf{for $r=0, \ldots, R-1$ do}}
\State{$\qquad$\textbf{parallel for device $j=1,\ldots,p$ do}:}
\State{$\qquad \quad$ \textbf{if PRIVIX variant:} }
\State{$\qquad\qquad$ Set $\mathbf{c}_j^{(r)}=\mathbf{c}_j^{(r-1)}-\frac{1}{\tau}\left({\texttt{PRIVIX}}\left(\mathbf{S}^{(r-1)}\right)-{\texttt{PRIVIX}}\left(\mathbf{S}^{(r-1)}_{j}\right)\right)$}

\State{$\qquad\qquad$ Computes ${\mathbf{\Phi}}^{(r)}\triangleq \texttt{PRIVIX}(\mathbf{S}^{(r-1)})$}

\State{$\qquad \quad$ \textbf{if HEAPRIX variant:} }
\State{$\qquad\qquad$ Set $\mathbf{c}_j^{(r)}=\mathbf{c}_j^{(r-1)}-\frac{1}{\tau}\left(\mathbf{\Phi}^{(r)}-\mathbf{\Phi}^{(r)}_j\right)$}
\State{$\qquad\quad \quad$ Computes ${\mathbf{\Phi}}^{(r)}\triangleq \texttt{HEAVYMIX}\left[{\mathbf{S}}^{(r-1)}\right]+\texttt{PRIVIX}\left[{\mathbf{S}}^{(r-1)}- \tilde{\mathbf{S}}^{(r-1)}\right]$}

\State{$\qquad\quad$ Set $\boldsymbol{x}^{(r)}=\boldsymbol{x}^{(r-1)}-\gamma\mathbf{\Phi}^{(r)}$ and $\boldsymbol{x}_j^{(0,r)}=\boldsymbol{x}^{(r)}$ }
\State{$\qquad\quad $\textbf{for} $\ell=0,\ldots,\tau-1$ \textbf{do}}
\State{$\qquad\quad\quad$ Sample a mini-batch $\xi_j^{(\ell,r)}$ and compute $\tilde{\mathbf{g}}_{j}^{(\ell,r)}\triangleq\nabla{f}_j(\boldsymbol{x}^{(\ell,r)}_j,\xi_j^{(\ell,r)})$}
\State{$\qquad\quad\quad$ $\boldsymbol{x}^{(\ell+1,r)}_{j}=\boldsymbol{x}^{(\ell,r)}_j-\eta \left( \tilde{\mathbf{g}}_{j}^{(\ell,r)}-\mathbf{c}_j^{(r)}\right)$ \label{eq:update-rule-alg-heter1}}
\State{$\qquad\quad$\textbf{end for}}
\State{$\qquad\quad\quad$Device $j$ sends $\mathbf{S}^{(r)}_{j}\triangleq\mathbf{S}\left(\boldsymbol{x}_j^{(0,r)}-~{\boldsymbol{x}}_{j}^{(\tau,r)}\right)$ back to the server.}
\State{$\qquad$Server \textbf{computes} }
\State{$\qquad\qquad {\mathbf{S}}^{(r)}=\frac{1}{p}\sum_{j=1}\mathbf{S}^{(r)}_{j}$ and  \textbf{broadcasts} ${\mathbf{S}}^{(r)}$ to all devices.}
\vspace{0.1cm}
\State{$\qquad$ \textbf{if HEAPRIX variant:} }
\State{$\qquad\quad\quad$ Device $j$ computes $\mathbf{\Phi}^{(r)}_j\triangleq \texttt{HEAPRIX}[\mathbf{S}_j^{(r)}]$}
\State{$\qquad \qquad$ Second round of communication to obtain $\delta_j^{(r)} :=  \mathbf{S}_j\left(\texttt{HEAVYMIX}[\mathbf{S}^{(r)}]\right)$ }
\State{$\qquad\qquad$ Broadcasts $\tilde{\mathbf{S}}^{(r)}\triangleq\frac{1}{p}\sum_{j=1}^p\delta_j^{(r)}$ to devices}

\State{$\qquad$\textbf{end parallel for}}
\State{\textbf{end}}
\State{\textbf{Output:} ${\boldsymbol{x}}^{(R-1)}$}
\vspace{- 0.1cm}
\end{algorithmic}
\end{algorithm}

\section{Convergence Analysis}\label{sec:cnvg-an}
In this section we start with a few common assumptions, then we provide the convergence results.

\subsection{Common Assumptions}

\begin{assumption}[Smoothness and Lower Boundedness]\label{Assu:1}
The local objective function $f_j(\cdot)$ of $j$th device is differentiable for $j\in [p]$ and $L$-smooth, i.e., $\|\nabla f_j(\boldsymbol{x})-\nabla f_j(\mathbf{y})\|\leq L\|\boldsymbol{x}-\mathbf{y}\|,\: \forall \;\boldsymbol{x},\mathbf{y}\in\mathbb{R}^d$. Moreover, the optimal objective function $f(\cdot)$ is bounded below by ${f^*} = \min_{\boldsymbol{x}} f(\boldsymbol{x})>-\infty$.
\end{assumption}

\begin{assumption}[\pl]\label{assum:pl}
A function $f(\boldsymbol{x})$ satisfies the \pl (PL)~ condition with constant $\mu$ if $\frac{1}{2}\|\nabla f(\boldsymbol{x})\|_2^2\geq \mu\big(f(\boldsymbol{x})-f(\boldsymbol{x}^*)\big),\: \forall \boldsymbol{x}\in\mathbb{R}^d $ with $\boldsymbol{x}^*$ is an optimal solution.
\end{assumption}

We note that Assumption~\ref{Assu:1} is a common assumption in the literature of stochastic optimization. Additionally, it is shown in \cite{karimi2016linear} that PL condition implies strong convexity property with same module. Additionally, PL objectives could also be nonconvex, hence strong convexity does not imply PL condition necessarily.

\subsection{Convergence of  \texttt{FEDSKETCH} for Homogeneous Setting}
Now we focus on the homogeneous case where data is distributed i.i.d. among local devices. In this case, the stochastic local gradient of each worker is an unbiased estimator of the global gradient. We will need the following additional common assumption on the stochastic gradients.

\begin{assumption}[Bounded Variance]\label{Assu:1.5}
For all $j\in [m]$, we can sample an independent mini-batch $\ell_j$   of size $|\Xi_j^{(\ell,r)}| = b$ and compute an unbiased stochastic gradient  $\tilde{\mathbf{g}}_j = \nabla f_j(\boldsymbol{w}; \Xi_j), \mathbb{E}_{\xi_j}[\tilde{\mathbf{g}}_j] = \nabla f(\boldsymbol{w})=\mathbf{g}$ with  the variance bounded is bounded by a constant $\sigma^2$, i.e., $
\mathbb{E}_{\Xi_j}\left[\|\tilde{\mathbf{g}}_j-\mathbf{g}\|^2\right]\leq \sigma^2$.
\end{assumption}

\begin{theorem}\label{thm:homog_case}
  Suppose that the conditions in Assumptions~\ref{Assu:1}-\ref{Assu:1.5} hold. Given $0<m=O\left(\frac{e}{\mu^2}\right)\leq d$, and Consider \texttt{FedSKETCH} in Algorithm~\ref{Alg:PFLHom} with sketch size $B=O\left(m\log\left(\frac{d R}{\delta}\right)\right)$. If the local data distributions of all users are identical (homogeneous setting), then with probability $1-\delta$ we have
 \begin{itemize}
     \item \textbf{Nonconvex:}
     \begin{itemize}
         \item [1)] For the  \texttt{FedSKETCH-PRIVIX} algorithm, by choosing stepsizes as $\eta=\frac{1}{L\gamma}\sqrt{\frac{k}{R\tau\left(\frac{\mu^2d}{k}+1\right)}}$ and $\gamma\geq k$, the sequence of iterates satisfies  $\frac{1}{R}\sum_{r=0}^{R-1}\left\|\nabla f({\boldsymbol{w}}^{(r)})\right\|_2^2\leq {\epsilon}$ if we set
     $R=O\left(\frac{1}{\epsilon}\right)$ and $ \tau=O\left(\frac{\mu^2d+1}{{k}\epsilon}\right)$.
         \item[2)] For \texttt{FedSKETCH-HEAPRIX} algorithm, by choosing stepsizes as $\eta=\frac{1}{L\gamma}\sqrt{\frac{k}{R\tau\left(\frac{\mu^2d-1}{k}+1\right)}}$ and $\gamma\geq k$, the sequence of iterates satisfies  $\frac{1}{R}\sum_{r=0}^{R-1}\left\|\nabla f({\boldsymbol{w}}^{(r)})\right\|_2^2\leq {\epsilon}$ if we set
     $R=O\left(\frac{1}{\epsilon}\right)$ and $ \tau=O\left(\frac{\mu^2d}{{k}\epsilon}\right)$.
     \end{itemize}

     \item \textbf{PL or Strongly convex:}
      \begin{itemize}
          \item[1)] For \texttt{FedSKETCH-PRIVIX} algorithm, by choosing stepsizes as $\eta=\frac{1}{2L\left(\frac{\mu^2d}{k}+1\right)\tau\gamma}$ and $\gamma\geq k$, we obtain that the iterates satisfy $\mathbb{E}\Big[f({\boldsymbol{w}}^{(R)})-f({\boldsymbol{w}}^{(*)})\Big]\leq \epsilon$ if  we set
     $R=O\left(\left(\frac{\mu^2d}{k}+1\right)\kappa\log\left(\frac{1}{\epsilon}\right)\right)$ and $ \tau=O\left(\frac{\mu^2d+1}{k\left(\frac{\mu^2d}{k}+1\right)\epsilon}\right)$.

          \item[2)] For \texttt{FedSKETCH-HEAPRIX} algorithm
by choosing stepsizes as $\eta=\frac{1}{2L\left(\frac{\mu^2d-1}{k}+1\right)\tau\gamma}$ and $\gamma\geq k$, we obtain that the iterates satisfy $\mathbb{E}\Big[f({\boldsymbol{w}}^{(R)})-f({\boldsymbol{w}}^{(*)})\Big]\leq \epsilon$ if  we set
     $R=O\left(\left(\frac{\mu^2d-1}{k}+1\right)\kappa\log\left(\frac{1}{\epsilon}\right)\right)$ and $ \tau=O\left(\frac{\mu^2d}{k\left(\frac{\mu^2d-1}{k}+1\right)\epsilon}\right)$.
      \end{itemize}

     \item \textbf{Convex:}
     \begin{itemize}
         \item[1)]For the \texttt{FedSKETCH-PRIVIX} algorithm, by choosing stepsizes as $\eta=\frac{1}{2L\left(\frac{\mu^2d}{k}+1\right)\tau\gamma}$ and $\gamma\geq k$, we obtain that the iterates satisfy $ \mathbb{E}\Big[f({\boldsymbol{w}}^{(R)})-f({\boldsymbol{w}}^{(*)})\Big]\leq \epsilon$ if we set
     $R=O\left(\frac{L\left(1+\frac{\mu^2d}{k}\right)}{\epsilon}\log\left(\frac{1}{\epsilon}\right)\right)$ and $ \tau=O\left(\frac{\left(\mu^2d+1\right)^2}{k\left(\frac{\mu^2d}{k}+1\right)^2\epsilon^2}\right).$
         \item[2)] For the \texttt{FedSKETCH-HEAPRIX} algorithm,
by choosing stepsizes as $\eta=\frac{1}{2L\left(\frac{\mu^2d-1}{k}+1\right)\tau\gamma}$ and $\gamma\geq k$, we obtain that the iterates satisfy $ \mathbb{E}\Big[f({\boldsymbol{w}}^{(R)})-f({\boldsymbol{w}}^{(*)})\Big]\leq \epsilon$ if we set
     $R=O\left(\frac{L\left(\frac{\mu^2d-1}{k}+1\right)}{\epsilon}\log\left(\frac{1}{\epsilon}\right)\right)$ and $ \tau=O\left(\frac{\left(\mu^2d\right)^2}{k\left(\frac{\mu^2d-1}{k}+1\right)^2\epsilon^2}\right).$
     \end{itemize}
 \end{itemize}
\end{theorem}

\begin{corollary}[Total communication cost]
As a consequence of Remark~\ref{rmk:cnd-lr}, the total communication cost per-worker becomes \begin{align}
O\left(RB\right)&=O\left(Rm\log \left(\frac{d R}{\delta}\right)\right)=O\left(\frac{m }{\epsilon}\log \left(\frac{d }{\epsilon\delta}\right)\right) \, .
\end{align}
We note that this result in addition to improving over the communication complexity of federated learning of the state-of-the-art from $O\left(\frac{d}{\epsilon}\right)$ in~\cite{karimireddy2019scaffold,wang2018cooperative,liang2019variance} to $O\left(\frac{m k}{\epsilon}\log \left(\frac{d k}{\epsilon\delta}\right)\right)$, it also implies differential privacy. As a result, total communication cost is
$$BkR=O\left(\frac{m k}{\epsilon}\log \left(\frac{d }{\epsilon\delta}\right)\right).$$

We note that the state-of-the-art in~\cite{karimireddy2019scaffold} the total communication cost is
\begin{align}\notag
    BkR&=O\left(kd\left(\frac{1}{\epsilon}\right)\frac{p^{2/3}}{k^{2/3}} \right)=O\left(\frac{kd}{\epsilon}\frac{p^{2/3}}{k^{2/3}}\right) \, .
\end{align}
We improve this result, in terms of dependency on $d$, to
\begin{align}\notag
    BkR=O\left(\frac{m k}{\epsilon}\log \left(\frac{d }{\epsilon\delta}\right)\right) \, .
\end{align}
In comparison to~\cite{ivkin2019communication}, we improve the total communication per worker from $RB=O\left(\frac{m }{\epsilon^2}\log \left(\frac{d }{\epsilon^2\delta}\right)\right)$ to $RB=O\left(\frac{m }{\epsilon}\log \left(\frac{d }{\epsilon\delta}\right)\right)$.
\end{corollary}

\begin{remark}
It is worthy to note that most of the available communication-efficient algorithm with quantization or compression only consider communication-efficiency from devices to server. However, Algorithm~\ref{Alg:PFLHom} also improves the communication efficiency from server to devices as well because of using lower dimensional sketching size and the fact that the average of sketching has also small dimension.
\end{remark}
We note that it is not fair to compare our algorithms with algorithms without compression. However, in the following Corollary we share an interesting observation regarding our algorithm for PL and thus strongly convex objectives in homogeneous setting.

\begin{corollary}[Total communication cost for PL or strongly convex]
To achieve the convergence error of $\epsilon$, we need to have $R=O\left(\kappa(\frac{\mu^2d}{k}+1)\log\frac{1}{\epsilon}\right)$ and $\tau=\left(\frac{(\mu^2d+1)}{(\frac{\mu^2d}{k}+1)k\epsilon}\right)$. This leads to the total communication cost per worker of
\begin{align}\notag
BR&=O\left(m\kappa(\frac{\mu^2d}{k}+1)\log\left(\frac{\kappa(\frac{\mu^2d^2}{k}+d)\log\frac{1}{\epsilon}}{\delta}\right)\log\frac{1}{\epsilon} \right) \, .
\end{align}
As a consequence, the total communication cost becomes:
\begin{align}\notag
BkR&=O\left(m\kappa(\mu^2d+k)\log\left(\frac{\kappa(\frac{\mu^2d^2}{k}+d)\log\frac{1}{\epsilon}}{\delta}\right)\log\frac{1}{\epsilon} \right) \, .
\end{align}
We note that the state-of-the-art in~\cite{karimireddy2019scaffold} the total communication cost is
\begin{align}\notag
    BkR=O\left(\kappa kd\log\left(\frac{p}{k\epsilon}\right) \right)=O\left(\kappa kd\log\left(\frac{p}{k\epsilon}\right)\right)  \, .
\end{align}
We improve this result, in terms of dependency on $d$, to
\begin{align}\notag
    BkR=O\left(m\kappa(\mu^2d+k)\log\left(\frac{\kappa(\frac{\mu^2d}{k}+d)\log\frac{1}{\epsilon}}{\delta}\right)\log\frac{1}{\epsilon} \right)\, ,
\end{align}
leading to an improvement from $kd$ to $k+d$. These results are summarized in Table~\ref{table:1}.
\end{corollary}



\begin{table}[h]
    \centering
    \caption{Comparison of results with compression and periodic averaging in the homogeneous setting. Here, $m$ is the number of devices, $\mu$ is the PL constant, $m$ is the number of bins of hash tables, $d$ is the dimension of the model, $\kappa$ is the condition number, $\epsilon$ is the target accuracy, $R$ is the number of communication rounds, and $\tau$ is the number of local updates. UG and PP stand for Unbounded Gradient and Privacy Property respectively.}
\label{table:1}
    \resizebox{0.95\linewidth}{!}{
    \begin{tabular}{lllll}
        \toprule
                    &  \multicolumn{3}{c}{Objective function} &
        \\ \cmidrule(r){2-4}
        Reference        & Nonconvex      & PL/Strongly Convex                                 & UG & PP
        \\

        \midrule
        \makecell{\textbf{Ivkin et al.~\cite{ivkin2019communication}}}  & \makecell[l]{$-$}   & \makecell[l]{$R=O\left(\frac{\mu^2 d}{\epsilon}\right)$\\  $\tau=1$\\ $B=O\left(m\log\left(\frac{dR}{\delta}\right)\right)$\\
        $pRB=O\left(\frac{p\mu^2 d}{\epsilon}m\log\left(\frac{\mu^2d^2}{\epsilon\delta}\right)\right)$}                                                                                           & \makecell{\ding{55}} & \makecell{\ding{55}}
        \\

        \midrule
       \makecell{\textbf{Theorem~\ref{thm:homog_case}}} & \makecell[l]{$\boldsymbol{R=O\left(\frac{1}{\epsilon}\right)}$ \\[3pt] $\boldsymbol{\tau=O\left(\frac{\mu^2d+1}{k\epsilon}\right)}$\\[3pt]
       $\boldsymbol{B=O\left(m\log\left(\frac{dR}{\delta}\right)\right)}$\\[3pt]
       $\boldsymbol{kBR=O\left(\frac{mk}{\epsilon}\log\left(\frac{d}{\epsilon\delta}\right)\right)}$}   & \makecell[l]{$\boldsymbol{R=O\left(\kappa\left(\frac{\mu^2 d}{k}+1\right)\log\left(\frac{1}{\epsilon}\right)\right)}$ \\[3pt] $\boldsymbol{\tau=O\left(\frac{\left(\mu^2 d+1\right)}{k\left(\frac{\mu^2 d}{k}+1\right)\epsilon}\right)}$\\$\boldsymbol{B=O\left(m\log\left(\frac{dR}{\delta}\right)\right)}$\\[3pt]
       $\boldsymbol{kBR=O\left({m}\kappa(\mu^2d+k)\log\frac{1}{\epsilon}\log\left(\frac{\kappa(\frac{\mu^2d^2}{k}+d)\log\frac{1}{\epsilon}}{\delta}\right)\right)}$}                                               & \makecell{\ding{52}} & \makecell{\ding{52}}
   \\
        \midrule
              \makecell{\textbf{Theorem~\ref{thm:homog_case}}} & \makecell[l]{$\boldsymbol{R=O\left(\frac{1}{\epsilon}\right)}$ \\[3pt] $\boldsymbol{\tau=O\left(\frac{\mu^2d}{k\epsilon}\right)}$\\[3pt]
       $\boldsymbol{B=O\left(m\log\left(\frac{dR}{\delta}\right)\right)}$\\[3pt]
       $\boldsymbol{kBR=O\left(\frac{mk}{\epsilon}\log\left(\frac{d}{\epsilon\delta}\right)\right)}$}   & \makecell[l]{$\boldsymbol{R=O\left(\kappa\left(\frac{\mu^2 d-1}{k}+1\right)\log\left(\frac{1}{\epsilon}\right)\right)}$ \\[3pt] $\boldsymbol{\tau=O\left(\frac{\left({\mu^2 d}\right)}{k\left(\frac{\mu^2 d}{k}+1\right)\epsilon}\right)}$\\$\boldsymbol{B=O\left(m\log\left(\frac{dR}{\delta}\right)\right)}$\\[3pt]
       $\boldsymbol{kBR=O\left({m}\kappa(\mu^2d-1+k)\log\frac{1}{\epsilon}\log\left(\frac{\kappa(d\frac{\mu^2d-1}{k}+d)\log\frac{1}{\epsilon}}{\delta}\right)\right)}$}                                                                                   & \makecell{\ding{52}} & \makecell{{\color{red}\ding{52}}}
   \\
        \bottomrule
    \end{tabular}
    }
\end{table}

\subsection{Convergence of  \texttt{FedSKETCHGATE} in Data Heterogeneous Setting}

\begin{assumption}[Bounded Local Variance]\label{Assu:2}
For all $j\in [p]$, we can sample an independent mini-batch $\Xi_j$   of size $|{\xi}_j| = b$ and compute an unbiased stochastic gradient $\tilde{\mathbf{g}}_j = \nabla f_j(\boldsymbol{w}; \Xi_j)$ with $\mathbb{E}_{\xi}[\tilde{\mathbf{g}}_j] = \nabla f_{j}(\boldsymbol{w})={\mathbf{g}}_j$. Moreover, the variance of local stochastic gradients is bounded from above by a constant $\sigma^2$, i.e., $
\mathbb{E}_{\Xi}\left[\|\tilde{\mathbf{g}}_j-{\mathbf{g}}_j\|^2\right]\leq \sigma^2$.
\end{assumption}

\begin{theorem}\label{thm:hetreg_case}
  Suppose that the conditions in Assumptions~\ref{Assu:1} and~\ref{Assu:2} hold. Given $0<m=O\left(\frac{e}{\mu^2}\right)\leq d$, and Consider \texttt{FedSKETCHGATE} in Algorithm~\ref{Alg:PFLHet} with sketch size $B=O\left(m\log\left(\frac{d R}{\delta}\right)\right)$. If the local data distributions of all users are identical (homogeneous setting), then with probability $1-\delta$ we have
 \begin{itemize}
     \item \textbf{Nonconvex:}
     \begin{itemize}
         \item [1)] For the \texttt{FedSKETCHGATE-PRIVIX} algorithm, by choosing stepsizes as $\eta=\frac{1}{L\gamma}\sqrt{\frac{p}{R\tau\left(\mu^2d\right)}}$ and $\gamma\geq p$, the sequence of iterates satisfies  $\frac{1}{R}\sum_{r=0}^{R-1}\left\|\nabla f({\boldsymbol{w}}^{(r)})\right\|_2^2\leq {\epsilon}$ if we set
     $R=O\left(\frac{\mu^2d+1}{\epsilon}\right)$ and $ \tau=O\left(\frac{1}{{p}\epsilon}\right)$.
         \item[2)] For \texttt{FedSKETCHGATE-HEAPRIX} algorithm, by choosing stepsizes as $\eta=\frac{1}{L\gamma}\sqrt{\frac{p}{R\tau\left(\mu^2d\right)}}$ and $\gamma\geq p$, the sequence of iterates satisfies  $\frac{1}{R}\sum_{r=0}^{R-1}\left\|\nabla f({\boldsymbol{w}}^{(r)})\right\|_2^2\leq {\epsilon}$ if we set
     $R=O\left(\frac{\mu^2d}{\epsilon}\right)$ and $ \tau=O\left(\frac{1}{{p}\epsilon}\right)$.
     \end{itemize}

     \item \textbf{PL or Strongly convex:}
      \begin{itemize}
          \item[1)] For the \texttt{FedSKETCHGATE-PRIVIX} algorithm, by choosing stepsizes as $\eta=\frac{1}{2L\left({\mu^2d}+1\right)\tau\gamma}$ and $\gamma\geq p$, we obtain that the iterates satisfy $\mathbb{E}\Big[f({\boldsymbol{w}}^{(R)})-f({\boldsymbol{w}}^{(*)})\Big]\leq \epsilon$ if  we set
     $R=O\left(\left(\mu^2d+1\right)\kappa\log\left(\frac{1}{\epsilon}\right)\right)$ and $ \tau=O\left(\frac{1}{p\epsilon}\right)$.

          \item[2)] For the case of
         \texttt{FedSKETCHGATE-HEAPRIX} algorithm,
by choosing stepsizes as $\eta=\frac{1}{2L\left(\mu^2d\right)\tau\gamma}$ and $\gamma\geq p$, we obtain that the iterates satisfy $\mathbb{E}\Big[f({\boldsymbol{w}}^{(R)})-f({\boldsymbol{w}}^{(*)})\Big]\leq \epsilon$ if  we set
     $R=O\left(\left(\mu^2d\right)\kappa\log\left(\frac{1}{\epsilon}\right)\right)$ and $ \tau=O\left(\frac{1}{p\epsilon}\right)$.
      \end{itemize}

     \item \textbf{Convex:}
     \begin{itemize}
         \item[1)]For the \texttt{FedSKETCHGATE-PRIVIX} algorithm, by choosing stepsizes as $\eta=\frac{1}{2L\left(\mu^2d+1\right)\tau\gamma}$ and $\gamma\geq p$, we obtain that the iterates satisfy $ \mathbb{E}\Big[f({\boldsymbol{w}}^{(R)})-f({\boldsymbol{w}}^{(*)})\Big]\leq \epsilon$ if we set
     $R=O\left(\frac{L\left(1+\mu^2d\right)}{\epsilon}\log\left(\frac{1}{\epsilon}\right)\right)$ and $ \tau=O\left(\frac{1}{p\epsilon^2}\right).$
         \item[2)] For the \texttt{FedSKETCHGATE-HEAPRIX} algorithm,
by choosing stepsizes as $\eta=\frac{1}{2L\left(\mu^2d\right)\tau\gamma}$ and $\gamma\geq p$, we obtain that the iterates satisfy $ \mathbb{E}\Big[f({\boldsymbol{w}}^{(R)})-f({\boldsymbol{w}}^{(*)})\Big]\leq \epsilon$ if we set
     $R=O\left(\frac{L\left(\mu^2d\right)}{\epsilon}\log\left(\frac{1}{\epsilon}\right)\right)$ and $ \tau=O\left(\frac{1}{p\epsilon^2}\right).$
     \end{itemize}
 \end{itemize}
\end{theorem}

\begin{table}[h]
    \centering
    \caption{Comparison of results with compression and periodic averaging in the heterogeneous setting. Here, $p$ is the number of devices, $\mu$ is compression of hash table, $d$ is the dimension of the model, $\kappa$ is condition number, $\epsilon$ is target accuracy, $R$ is  the number of communication rounds, and $\tau$ is the number of local updates. UG and PP stand for Unbounded Gradient and Privacy Property respectively.}
\label{table:2}
    \resizebox{0.9\linewidth}{!}{
    \begin{tabular}{lllll}
        \toprule
                    &  \multicolumn{3}{c}{Objective function} &
        \\ \cmidrule(r){2-4}
        Reference        & Nonconvex                                        & General Convex   & UG & PP
        \\
        \midrule
        \makecell{\textbf{Li et al.~\cite{li2019privacy}}}  & \makecell[l]{$-$}                & \makecell[l]{$R\!=\!O\left(\frac{\mu^2 d}{\epsilon^{2}}\right)$ \\ $\tau\!=\!1$\\
        $B=O\left(m\log\left(\frac{\mu^2d^2}{\epsilon^2\delta}\right)\right)$}                                                                            & \makecell{\ding{55}} & \makecell{\ding{52}}
        \\


        \midrule
        \makecell{\textbf{Rothchild et al.~\cite{rothchild2020fetchsgd}}}  & \makecell[l]{$R=O\left(\max(\frac{1}{\epsilon^2},\frac{d^2-md}{m^2\epsilon})\right)$ \\ $\tau=1$\\
        $B=O\left(m\log\left(\frac{d}{\epsilon^2\delta}\right)\right)$\\
        $BR=O\left(\frac{m}{\epsilon^2}\max(\frac{1}{\epsilon^2},\frac{d^2-md}{m^2\epsilon})\log\left(\frac{d}{\delta}\max(\frac{1}{\epsilon^2},\frac{d^2-md}{m^2\epsilon})\right)\right)$}       & \makecell[l]{$-$}                                                                            & \makecell{\ding{55}} & \makecell{\ding{55}}
        \\
        \midrule
        \makecell{\textbf{Rothchild et al.~\cite{rothchild2020fetchsgd}}}  & \makecell[l]{$R=O\left(\frac{\max(I^{2/3},2-\alpha)}{\epsilon^3}\right)$ \\ $\tau=1$\\
        $B=O\left(\frac{m}{\alpha}\log\left(\frac{d\max(I^{2/3},2-\alpha)}{\epsilon^3\delta}\right)\right)$\\
        $BR=O\left(\frac{m\max(I^{2/3},2-\alpha)}{\epsilon^3\alpha}\log\left(\frac{d\max(I^{2/3},2-\alpha)}{\epsilon^3\delta}\right)\right)$
        }       & \makecell[l]{$-$}                                                                            & \makecell{\ding{55}} & \makecell{\ding{55}}
        \\
        \midrule
       \makecell{\textbf{Theorem~\ref{thm:hetreg_case}}} & \makecell[l]{$\boldsymbol{R=O\left(\frac{\mu^2d+1}{\epsilon}\right)}$ \\[3pt] $\boldsymbol{\tau=O\left(\frac{1}{p\epsilon}\right)}$\\[3pt]
       $\boldsymbol{B=O\left(m\log\left(\frac{\mu^2d^2+d}{\epsilon\delta}\right)\right)}$\\[3pt]
       $\boldsymbol{BR=O\left(\frac{m\left(\mu^2d+1\right)}{\epsilon}\log\left(\frac{\mu^2d^2+d}{\epsilon\delta}\log\left(\frac{1}{\epsilon}\right)\right)\right)}$
       }   &
       \makecell[l]{$\boldsymbol{R\!=\!O\left(\frac{1+\mu^2d}{\epsilon}{\color{black}\log\left(\frac{1}{\epsilon}\right)}\right)}$\\[3pt]
       $\boldsymbol{\tau\!=\!O\left(\frac{1}{p\epsilon^2}\right)}$\\[3pt]
       $\boldsymbol{B=O\left(m\log\left(\frac{\mu^2d^2+d}{\epsilon\delta}\log\left(\frac{1}{\epsilon}\right)\right)\right)}$
}                                                                            & \makecell{\ding{52}} & \makecell{\ding{52}}
   \\
        \midrule
              \makecell{\textbf{Theorem~\ref{thm:hetreg_case}}} & \makecell[l]{$\boldsymbol{R=O\left(\frac{\mu^2d}{\epsilon}\right)}$ \\[3pt] $\boldsymbol{\tau=O\left(\frac{1}{p\epsilon}\right)}$\\[3pt]
       $\boldsymbol{B=O\left(m\log\left(\frac{\mu^2d^2}{\epsilon\delta}\right)\right)}$\\[3pt]
       $\boldsymbol{BR=O\left(\frac{m\left(\mu^2d\right)}{\epsilon}\log\left(\frac{\mu^2d^2}{\epsilon\delta}\log\left(\frac{1}{\epsilon}\right)\right)\right)}$}   & \makecell[l]{$\boldsymbol{R\!=\!O\left(\frac{\mu^2d}{\epsilon}{\color{black}\log\left(\frac{1}{\epsilon}\right)}\right)}$\\[3pt]
       $\boldsymbol{\tau\!=\!O\left(\frac{1}{p\epsilon^2}\right)}$\\[3pt]
       $\boldsymbol{B=O\left(m\log\left(\frac{\mu^2d^2}{\epsilon\delta}\right)\right)}$}                                                                            & \makecell{\ding{52}} & \makecell{{\color{red}\ding{52}}}
   \\
        \bottomrule
    \end{tabular}
    }
\end{table}

\subsection{Comparison with Prior Methods~\cite{li2019privacy},~\cite{rothchild2020fetchsgd} and~\cite{philippenko2020artemis}}

\vspace{0.1in}\noindent\textbf{Comparison to~\cite{li2019privacy}.} We note that our convergence analysis does not rely on the bounded gradient assumption and it can be seen that we improve both the number of communication rounds $R$ and the size of vector $B$ per communication round while preserving the privacy property.
Additionally, we highlight that, while~\cite{li2019privacy} provides a convergence analysis for convex  objectives, our analysis holds for PL (thus strongly convex case), general convex and general nonconvex objectives.

\vspace{0.1in}\noindent\textbf{Comparison with~\cite{rothchild2020fetchsgd}.}
Consider two versions of \texttt{FetchSGD} in this reference. First while in our schemes we do not to have access to the exact entries of gradients, since the approaches in~\cite{rothchild2020fetchsgd} is based on $top_m$ queries, both of the proposed algorithms (in~\cite{rothchild2020fetchsgd}) require to have access to the exact value of $top_k$ gradients, hence they do not preserve privacy. Second, both of the convergence results in~\cite{rothchild2020fetchsgd} rely on the bounded gradient assumption and it is known that this assumption is not in consistent with $L$-smoothness when data distribution is heterogeneous which is the case in Federated Learning (see~\cite{bayoumi2020tighter} for more detail). However, our convergence results do not need any bounded gradient assumption. Third, Theorem 1~\cite{rothchild2020fetchsgd} is based on an Assumption that \emph{Contraction Holds} for the sequence of gradients encountered during the optimization which may not hold necessarily in practice, yet based on this strong assumption their total communication cost ($RB$) to achieve $\epsilon$ error is  $BR=O\left(m\max(\frac{1}{\epsilon^2},\frac{d^2-dm}{m^2\epsilon})\log\left(\frac{d}{\delta}\max(\frac{1}{\epsilon^2},\frac{d^2-dm}{m^2\epsilon})\right)\right)$ (Note for the sake of comparison we let the compression ration in~\cite{rothchild2020fetchsgd} to be $\frac{m}{d}$). In contrast, without any extra assumptions, our results in Theorem~\ref{thm:hetreg_case} for \texttt{PRIVIX} and \texttt{HEAPRIX} are respectively $BR=O\left(\frac{m\left(\mu^2d+1\right)}{\epsilon}\log\left(\frac{\mu^2d^2+d}{\epsilon\delta}\log\left(\frac{1}{\epsilon}\right)\right)\right)$ and $BR=O\left(\frac{m\left(\mu^2d\right)}{\epsilon}\log\left(\frac{\mu^2d^2}{\epsilon\delta}\log\left(\frac{1}{\epsilon}\right)\right)\right)$ which improves total communication cost in Theorem 1 in~\cite{rothchild2020fetchsgd} in regimes where $\frac{1}{\epsilon}\geq d$ or $d>>m$. Theorem 2 in~\cite{rothchild2020fetchsgd} is based on another assumption of Sliding Window Heavy Hitters, which is similar to gradient diversity assumption in~\cite{li2018federated,haddadpour2019convergence} (but it is weaker assumption of contraction in Theorem 1 in~\cite{rothchild2020fetchsgd}), and they showed that the total communication cost is $BR=O\left(\frac{m\max(I^{2/3},2-\alpha)}{\epsilon^3\alpha}\log\left(\frac{d\max(I^{2/3},2-\alpha)}{\epsilon^3\delta}\right)\right)$ ($I$ is constant comes from the extra assumption over the window of gradients which similar to bounded gradient diversity) which is again worse than obtained result in this paper with weaker assumptions in a regime where $\frac{I^{2/3}}{\epsilon^2}\geq d$. Next, unlike~\cite{rothchild2020fetchsgd} which only focuses on nonconvex objectives, in this work we provide the convergence analysis for PL (thus strongly convex case), general convex and general nonconvex objectives. Finally, although the algorithm in~\cite{rothchild2020fetchsgd} requires additional memory for the server to store the compression error correction vector, our algorithm does not need such additional storage. \\

These results are summarized in Table~\ref{table:2}.

\vspace{0.1in}\noindent\textbf{Comparison with~\cite{philippenko2020artemis}.} The reference~\cite{philippenko2020artemis} considers two-way compression from parameter server to devices and vice versa. They provide the convergence rate of $R=O\left(\frac{\omega^{\text{Up}}\omega^{\text{Down}}}{\epsilon^2}\right)$ for strongly-objective functions where $\omega^{\text{Up}}$ and $\omega^{\text{Down}}$ are uplink and downlink's compression noise (specializing to our case for the sake of comparison $\omega^{\text{Up}}=\omega^{\text{Down}}=\theta\left(d\right)$) for general heterogeneous data distribution. In contrast, while as pointed out in Remark~\ref{rmrk:bidirect} that our algorithms are using bidirectional compression due to use of sketching for communication, our convergence rate for strongly-convex objective is $R=O(\kappa\mu^2d\log\left(\frac{1}{\epsilon}\right))$ with probability $1-\delta$.

\section{Numerical Example}\label{sec:experimnt}

In this section, we provide empirical results on MNIST dataset to demonstrate the effectiveness of our proposed algorithms. The model we use is the LeNet-5 Convolutional Neural Network (CNN) architecture introduced in~\cite{lecun1998gradient}, with $60\,000$ model parameters in total.

Four methods are compared in our experiments: Federated SGD (FedSGD), SketchSGD~\cite{ivkin2019communication}, FedSketch-PRIVIX (FS-PRIVIX) and FedSketch-HEAPRIX (FS-HEAPRIX). We implement the algorithms by simulating the distributed and federated environment.
Note that in Algorithm~\ref{Alg:PFLHom}, FS-PRIVIX with global learning rate $\gamma=1$ is equivalent to the DiffSketch algorithm proposed in~\cite{li2018federated}.
In the following experiments, we set the number of workers to $50$.
For federated learning algorithms, we use different number of local updates $\tau$.
For SketchedSGD which is under synchronous distributed learning framework, $\tau$ is fixed and equal to $1$.
For all methods, we tune the learning rates (both local, i.e. $\eta$ and global, i.e. $\gamma$, if applicable) over the log-scale and report the best results.

In each round of local update, we randomly choose half of the local devices to be active, which is the common practice in real-world applications.
For the data distribution on each device, we test both \emph{homogeneous} and \emph{heterogeneous} setting.
In the former case, each device receives uniformly drawn data samples (each class has equal probability to be selected).
In the latter case, each device only receives samples from one or two classes among ten digits in the MNIST dataset.
Since data is not distributed i.i.d. among local devices, training is expected to be harder in the heterogeneous case.

\clearpage

\begin{figure}[H]
	\begin{center}
		\mbox{			    \includegraphics[width=2.1in]{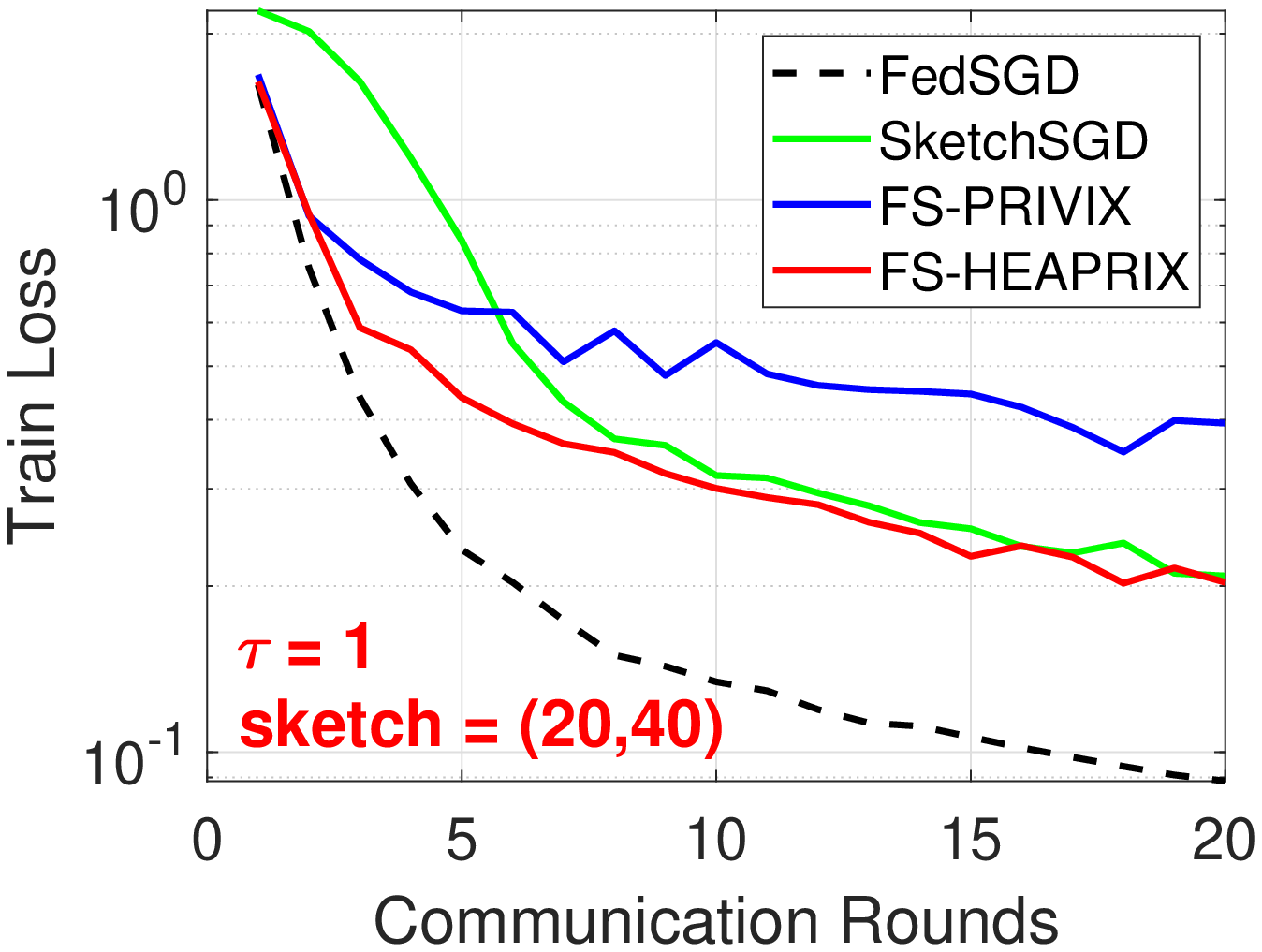} \hspace{-0.2in}
		\includegraphics[width=2.1in]{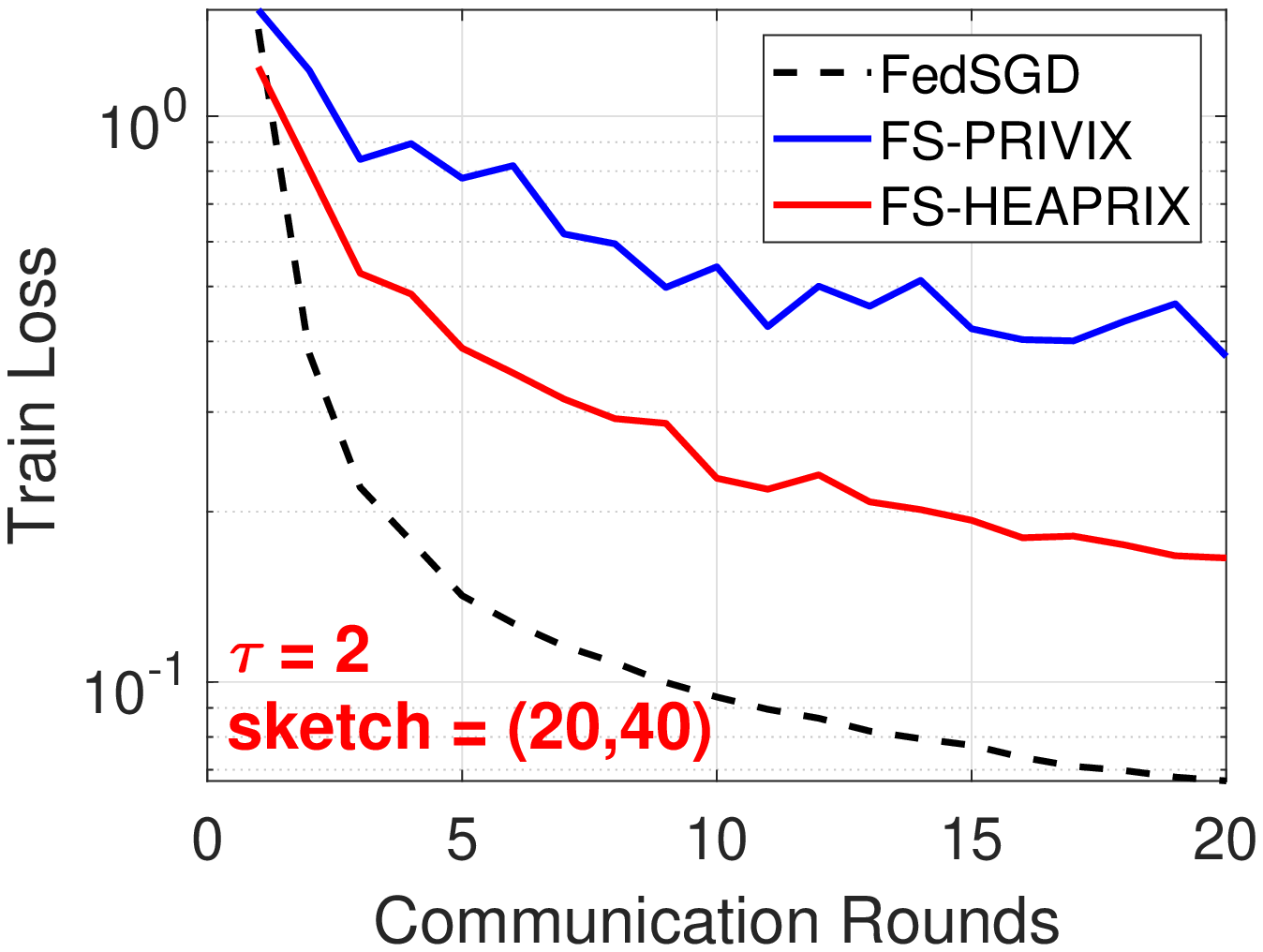} \hspace{-0.2in}
		\includegraphics[width=2.1in]{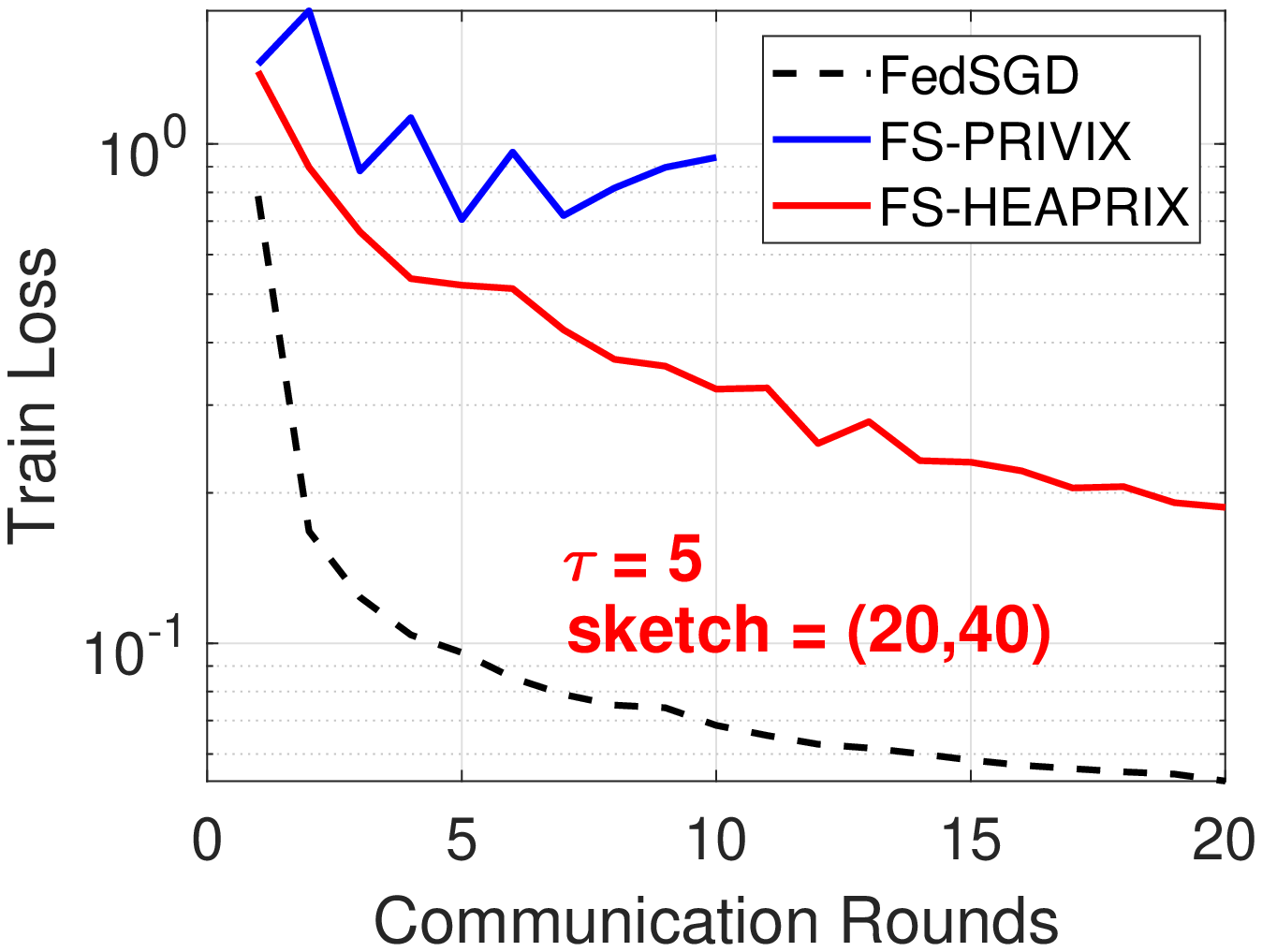}
		}
		\mbox{
		\includegraphics[width=2.1in]{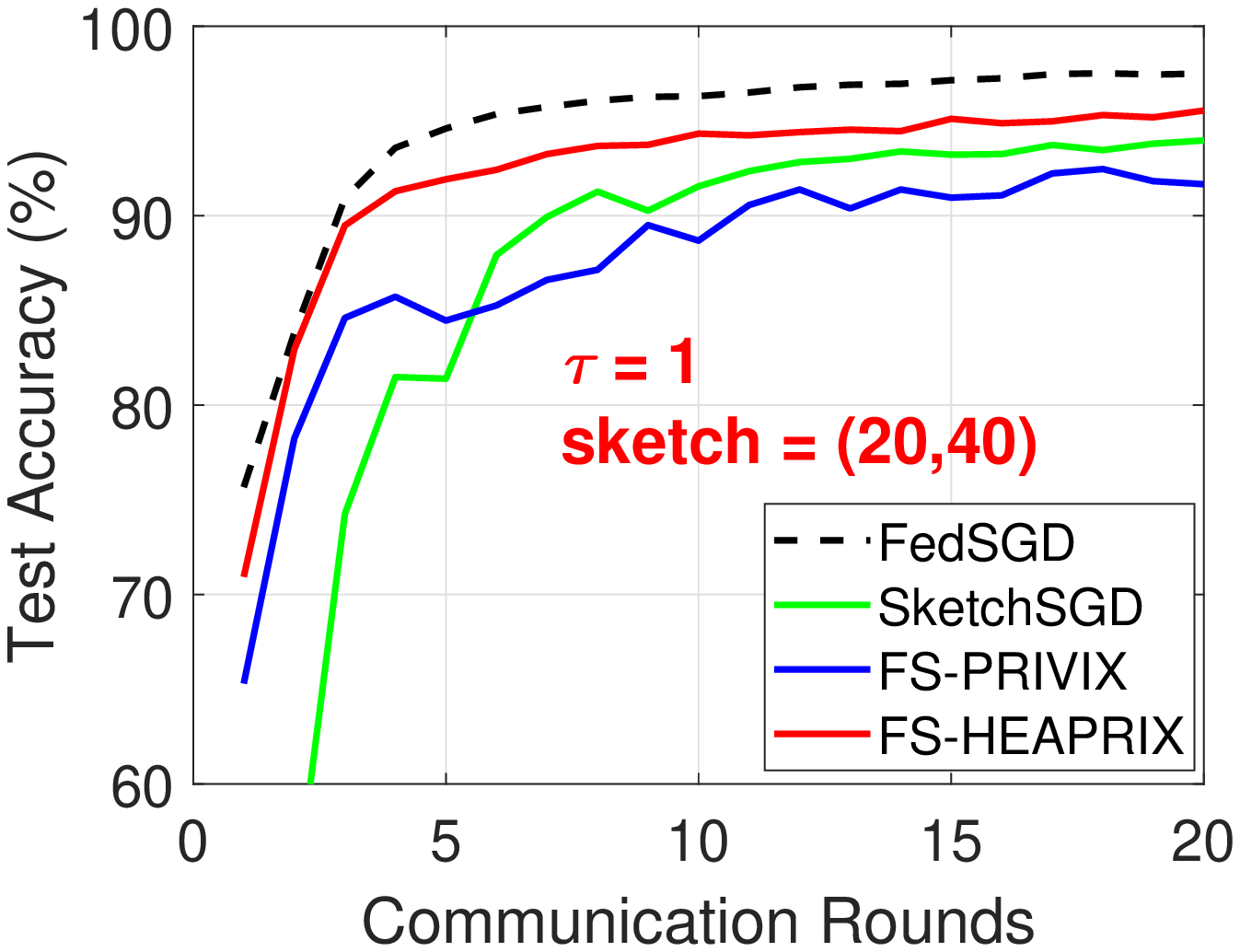} \hspace{-0.2in}
		\includegraphics[width=2.1in]{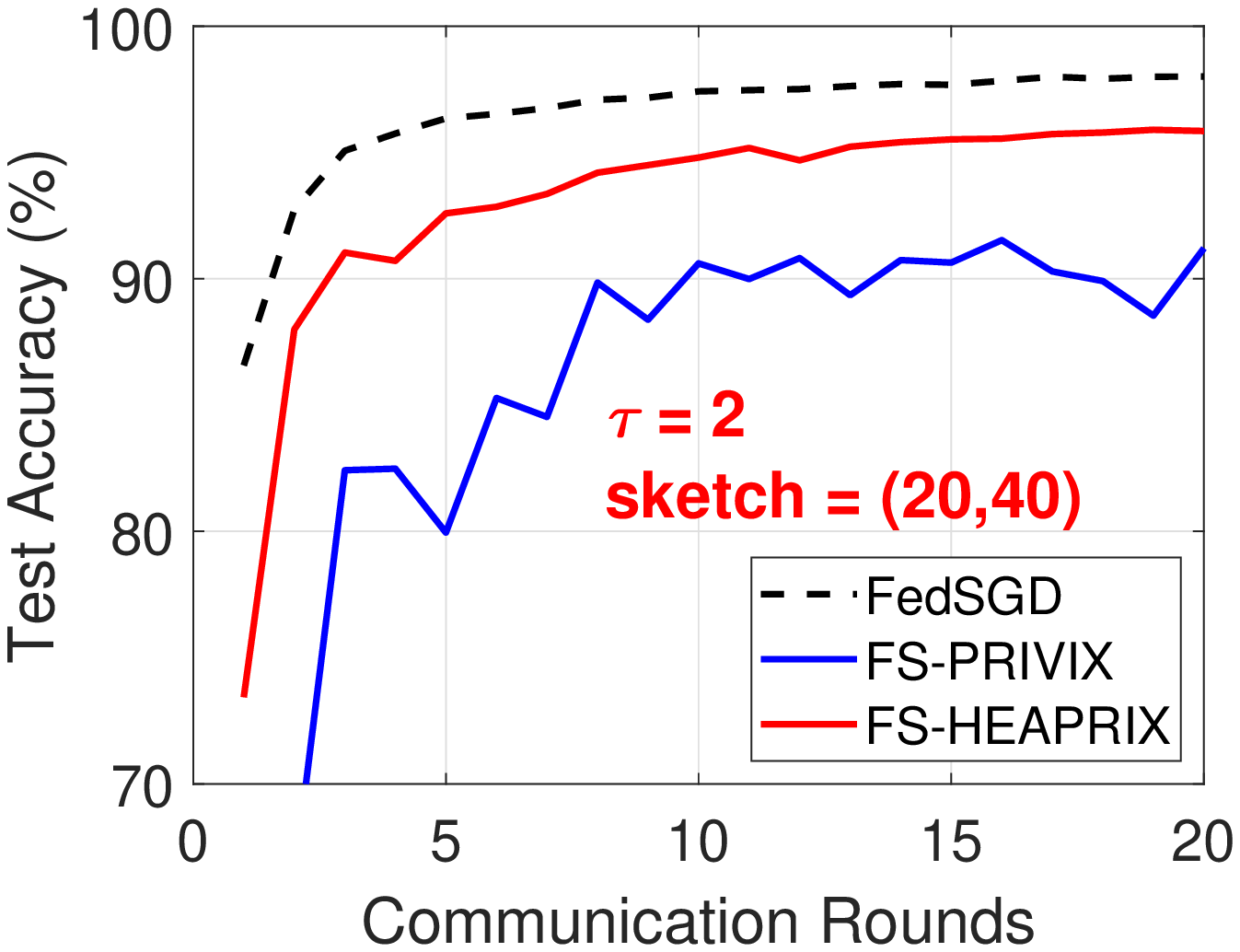} \hspace{-0.2in}
		\includegraphics[width=2.1in]{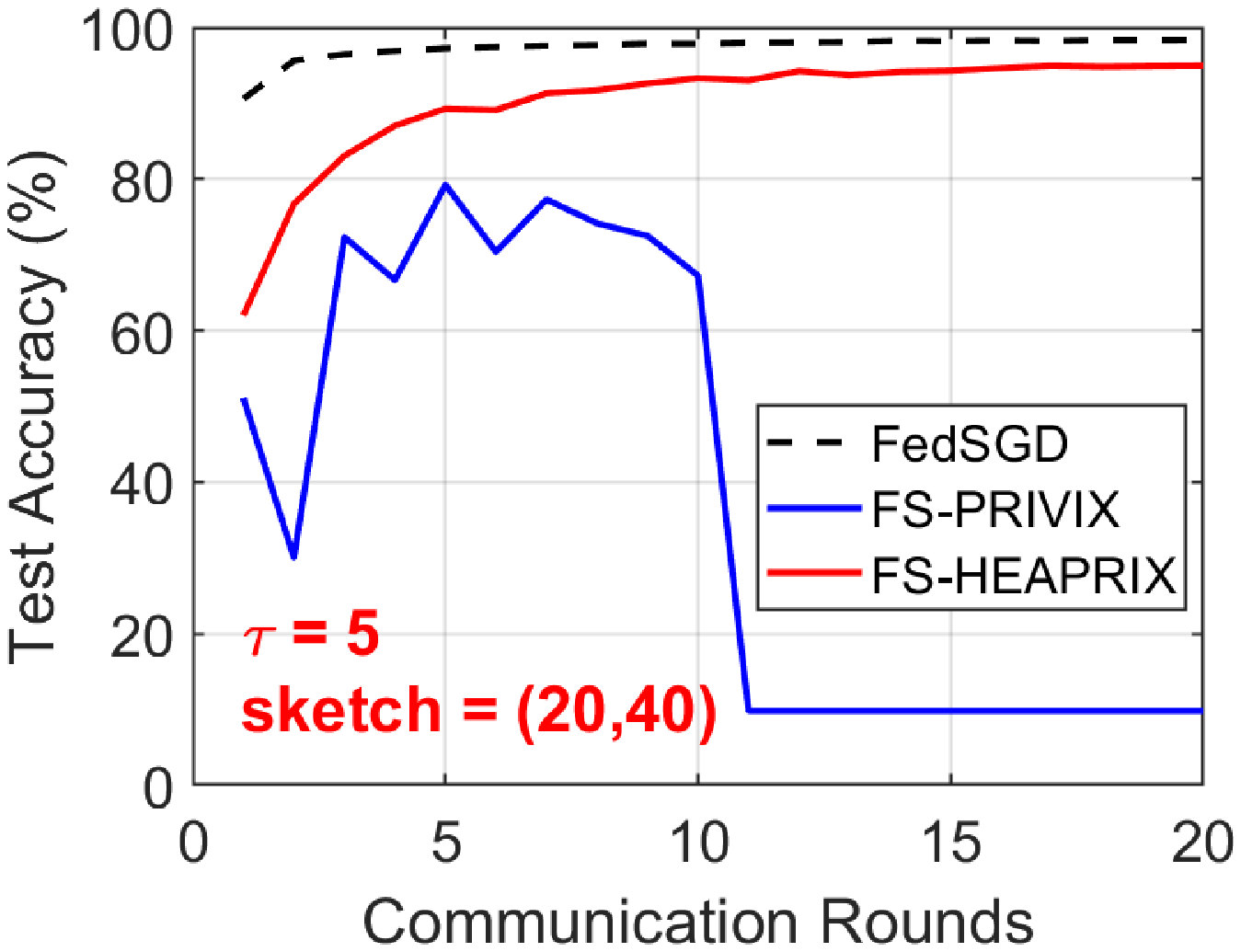}
		}
		\mbox{			    \includegraphics[width=2.1in]{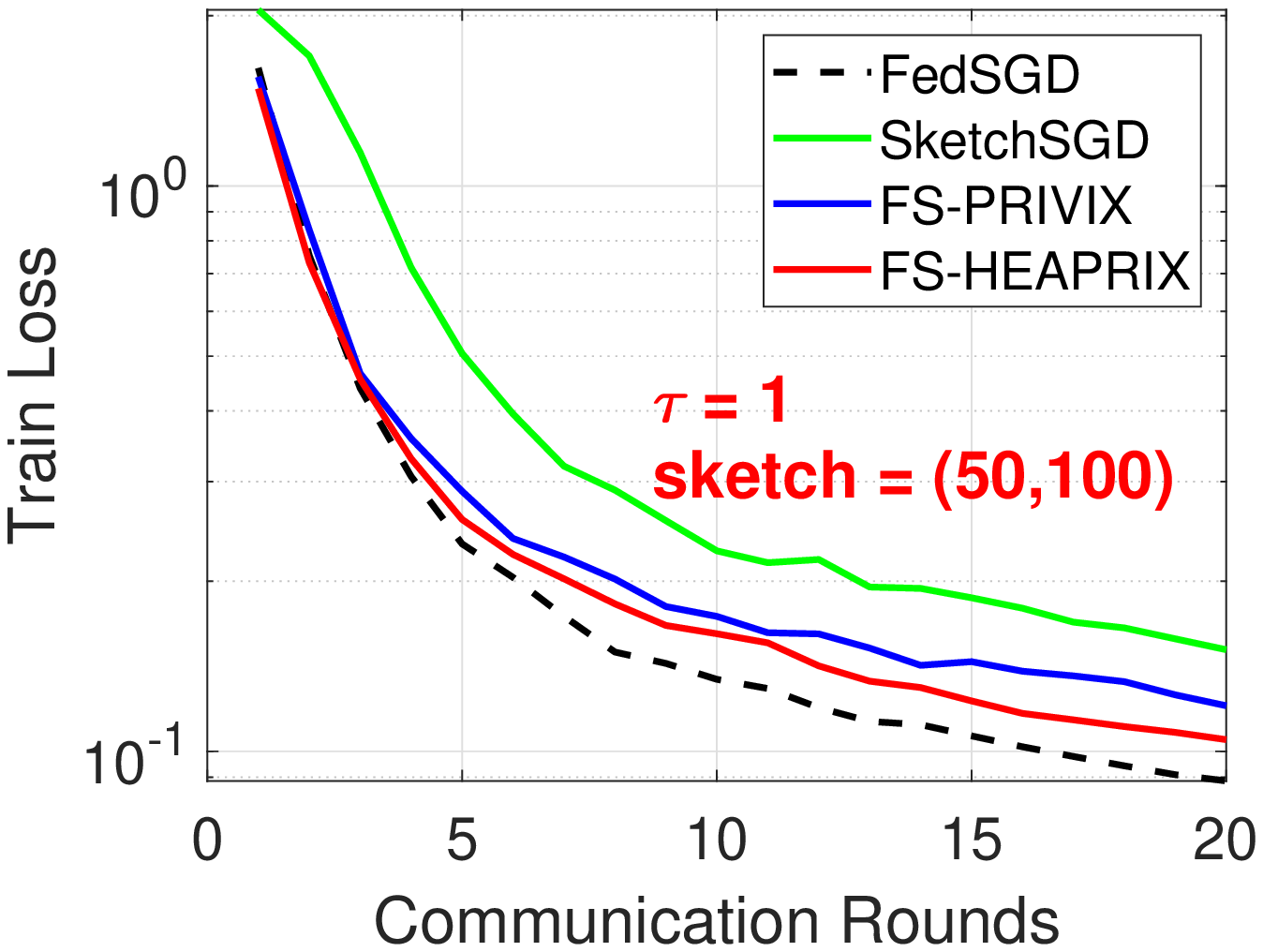} \hspace{-0.2in}
		\includegraphics[width=2.1in]{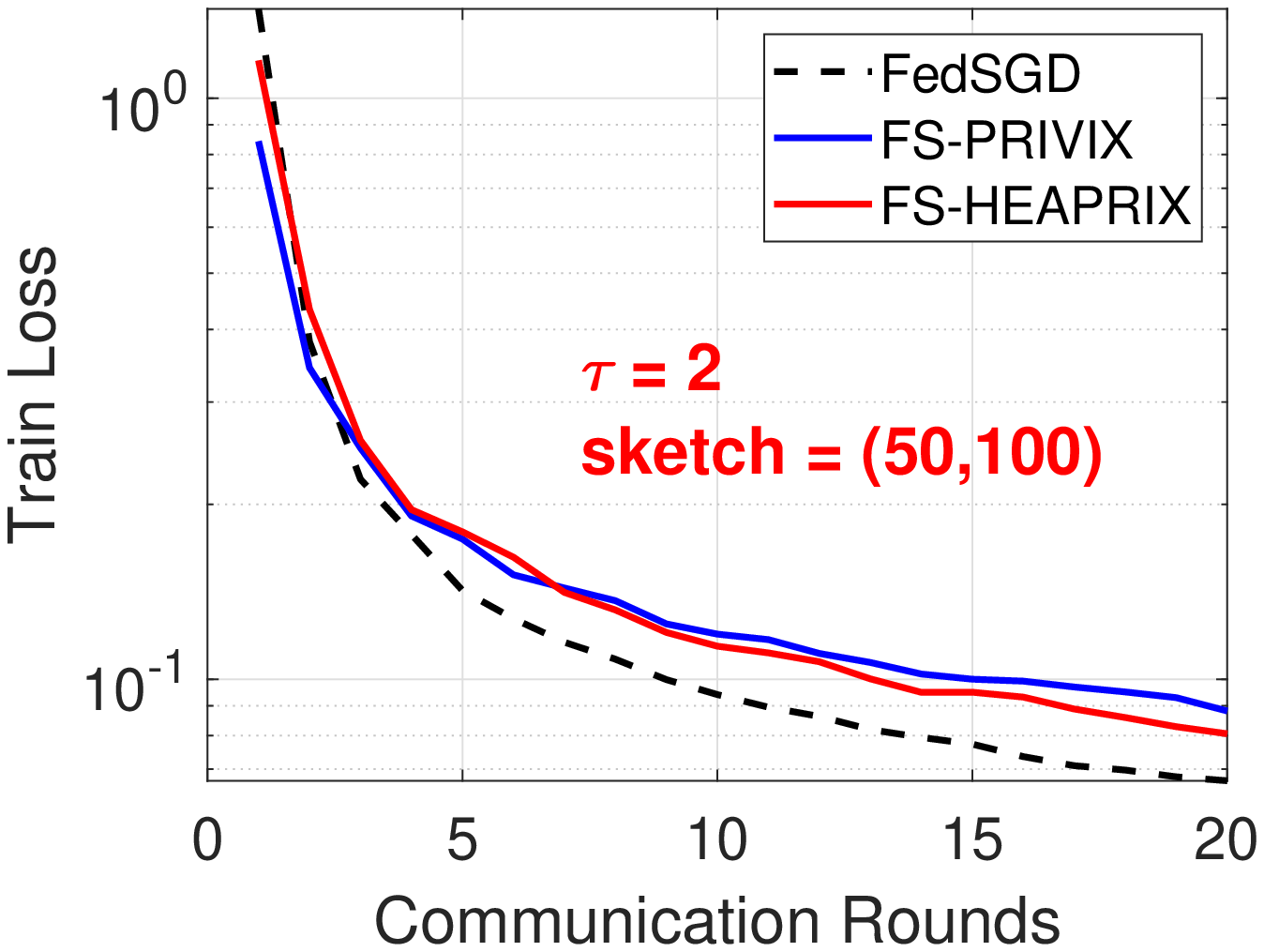} \hspace{-0.2in}
		\includegraphics[width=2.1in]{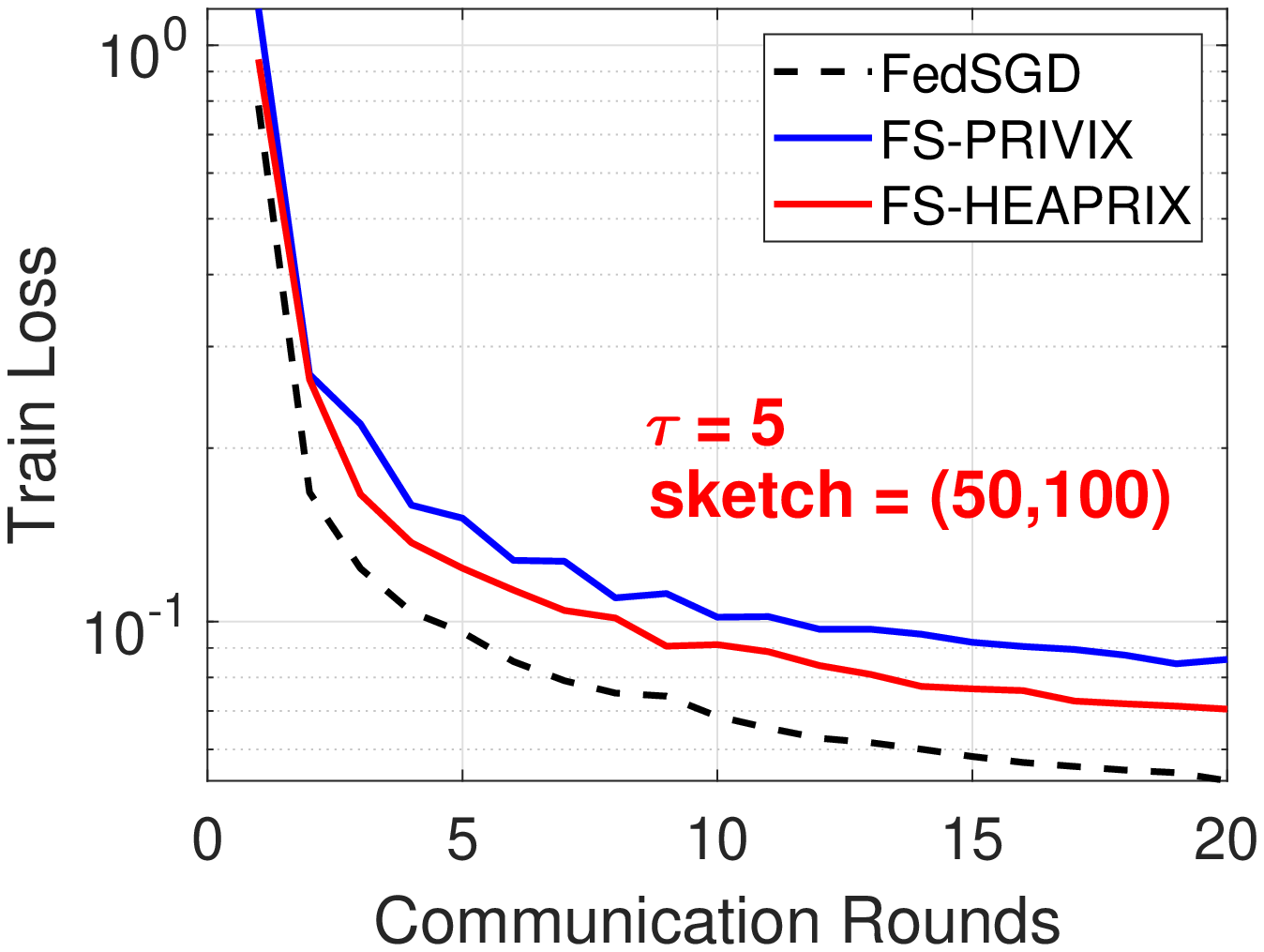}
		}
		\mbox{
		\includegraphics[width=2.1in]{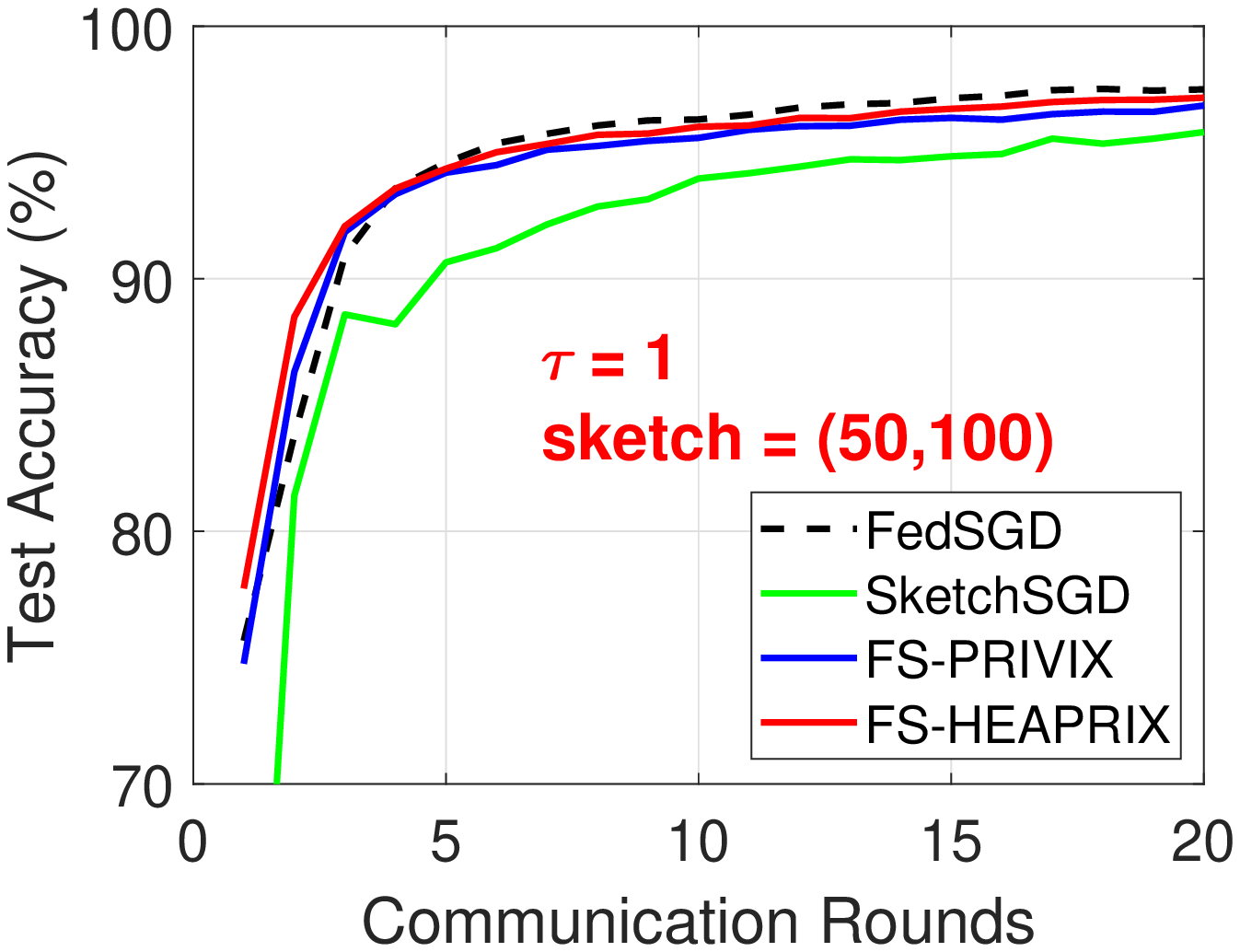} \hspace{-0.2in}
		\includegraphics[width=2.1in]{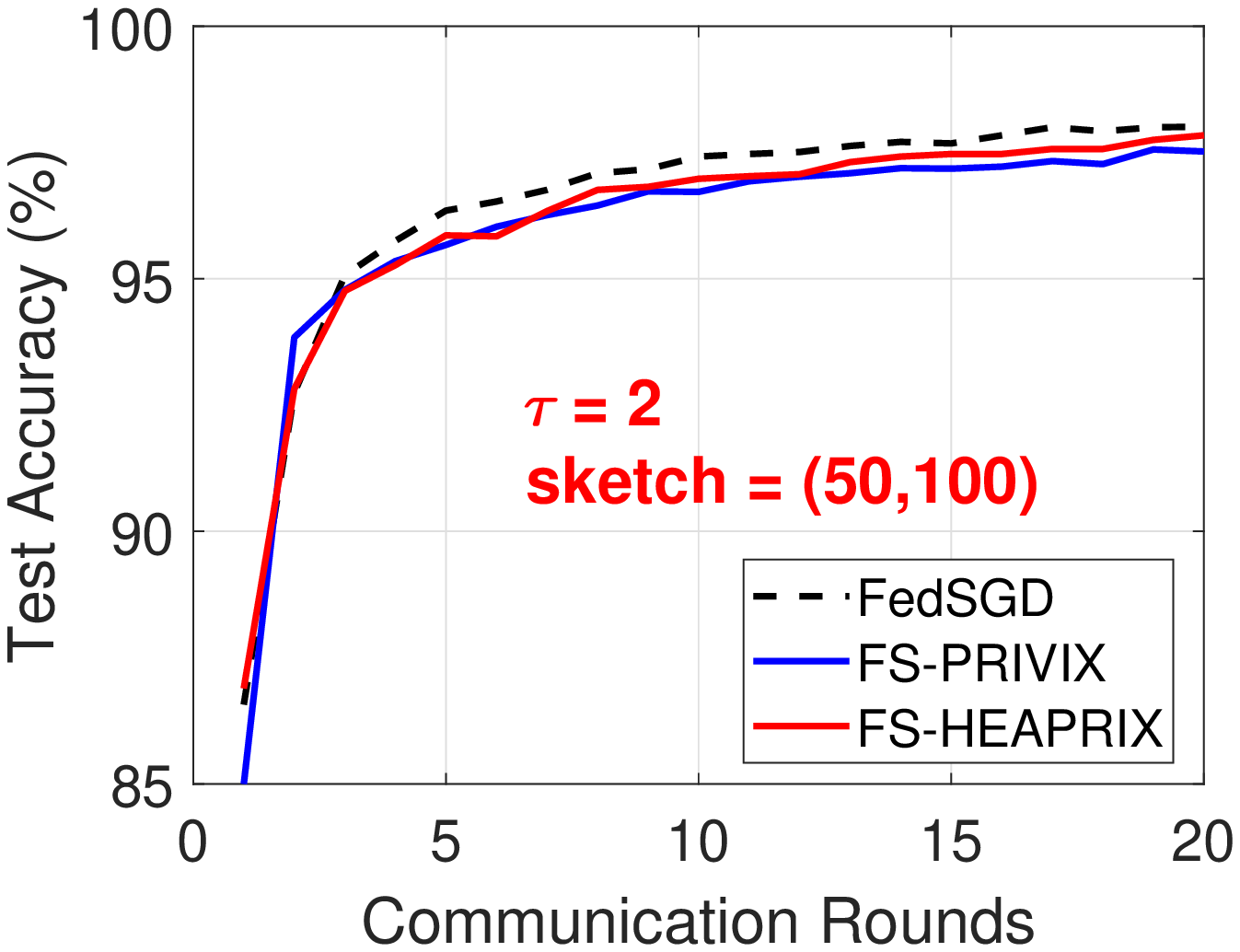} \hspace{-0.2in}
		\includegraphics[width=2.1in]{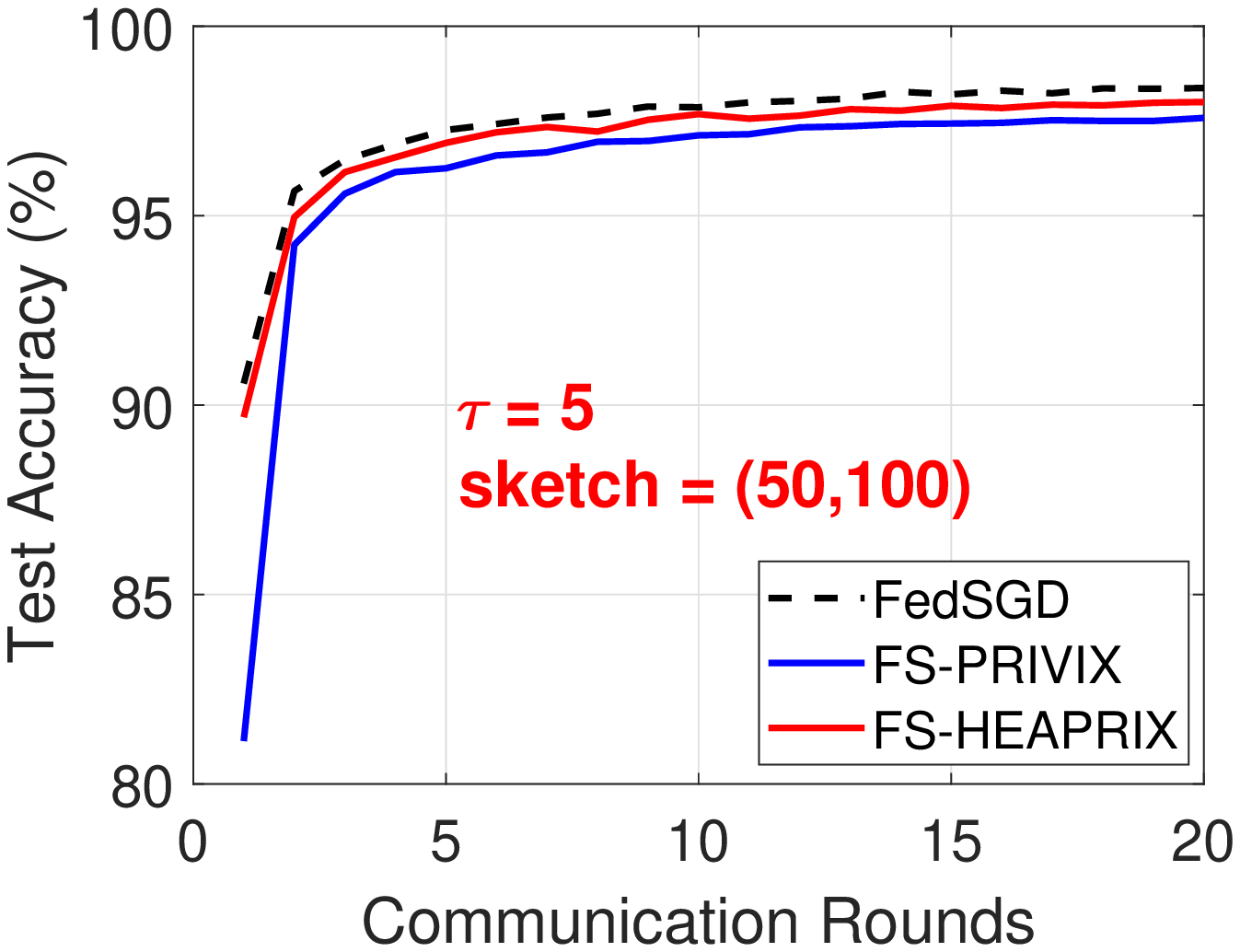}
		}
	\end{center}
	\caption{Homogeneous case: Comparison of compressed optimization methods on LeNet CNN architecture.}
    \label{fig:MNIST-iid1}
\end{figure}

\textbf{Homogeneous case.} In Figure~\ref{fig:MNIST-iid1} first column, we provide the training loss and test accuracy for the four algorithms mentioned above, with $\tau=1$ (since SketchSGD requires single local update per round).
We also test different sizes of sketching matrix, $(t,k)=(20,40)$ and $(50,100)$. Note that these two choices of sketch size correspond to a $75\times$ and $12\times$ compression ratio, respectively.
In general, as one would expect, higher compression ratio leads to worse learning performance. In both cases, FS-HEAPRIX performs the best in terms of both training objective and test accuracy. FS-PRIVIX is better when sketch size is large (i.e. when the estimation from sketches are more accurate), while SketchSGD performs better with small sketch size.

The results for multiple local updates are given in column 2 and column 3 in Figure~\ref{fig:MNIST-iid1}, where we set $\tau=2,5$. We see that FS-HEAPRIX is significantly better than FS-PRIVIX, either with small or large sketching matrix. In both cases, FS-HEAPRIX yields acceptable extra test error compared to FedSGD, especially when considering the high compression ratio (e.g. $75\times$). However, FS-PRIVIX performs poorly with small sketch size $(20,40)$, and even diverges with $\tau=5$.
We also observe that the performances of FS-HEAPRIX improve when the number of local updates increases. That is, the proposed method is able to further reduce the communication cost by reducing the number of rounds required for communication. This is also consistent with our theoretical claims established in this paper. For $\tau=1,2,5$, we see that a sketch size of $(50,100)$ is sufficient to give similar test accuracy as the Federated SGD (FedSGD) algorithm.

\begin{figure}[h]
	\begin{center}
		\mbox{			    \includegraphics[width=2.1in]{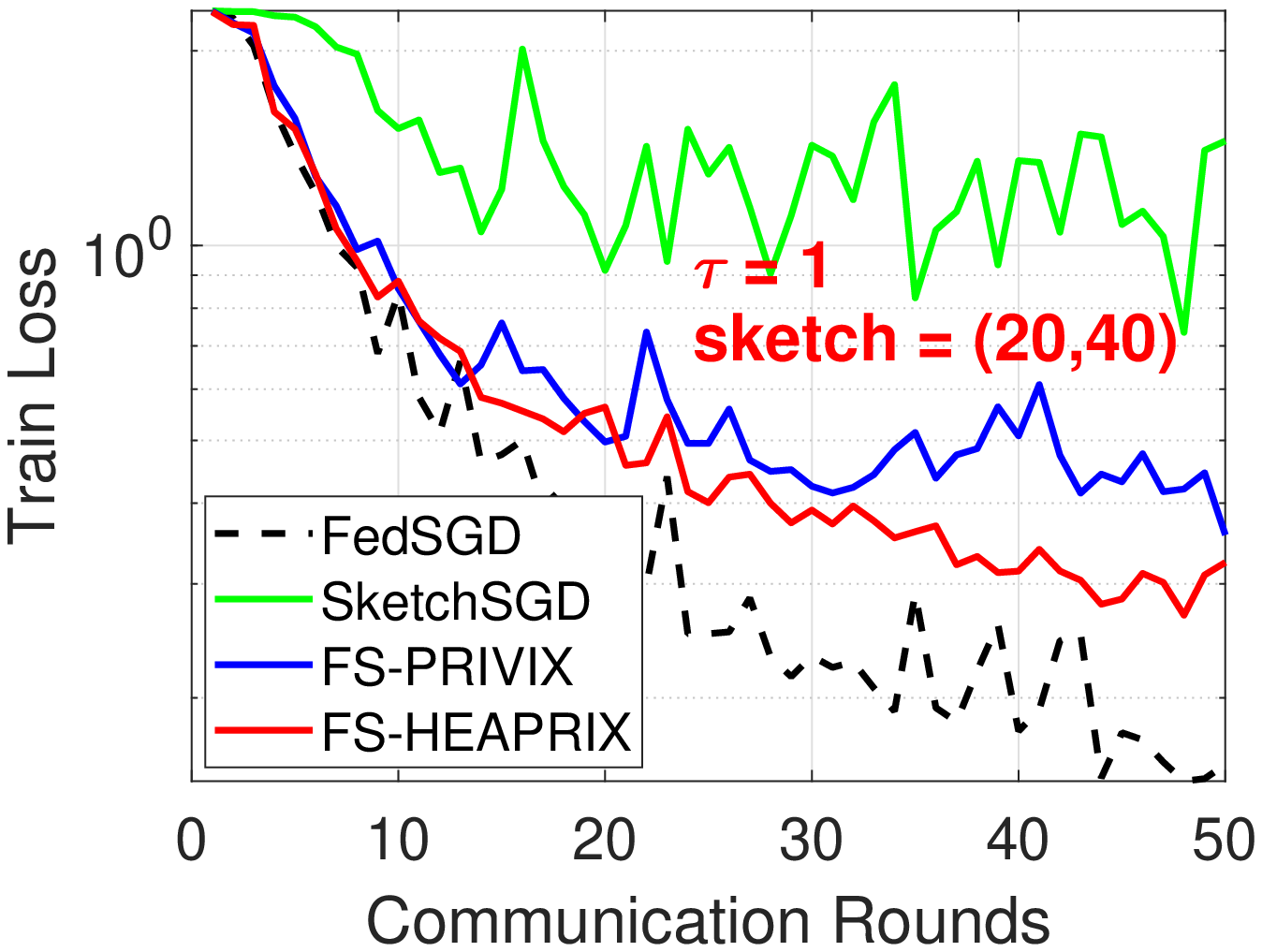} \hspace{-0.2in}
		\includegraphics[width=2.1in]{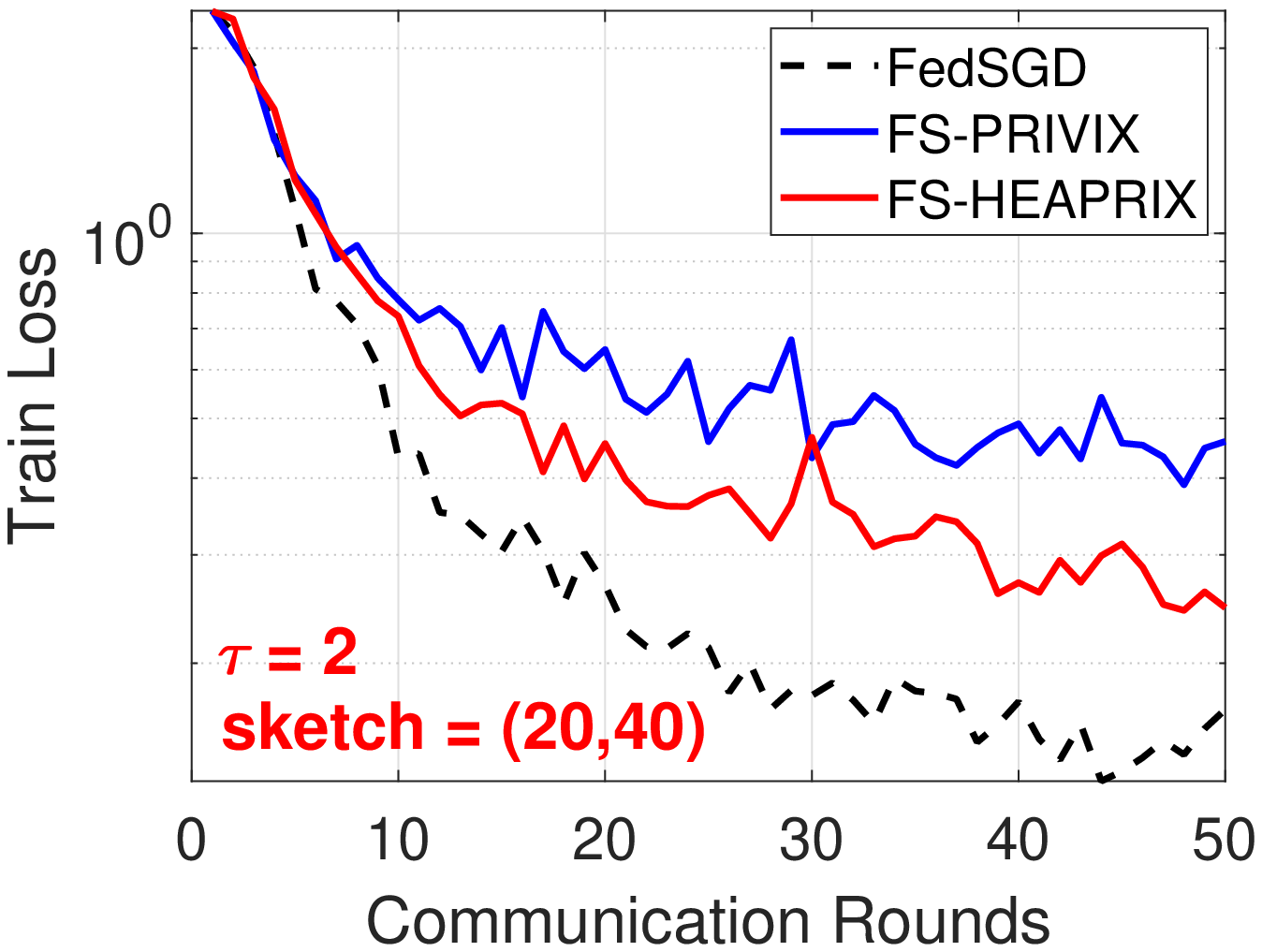} \hspace{-0.2in}
		\includegraphics[width=2.1in]{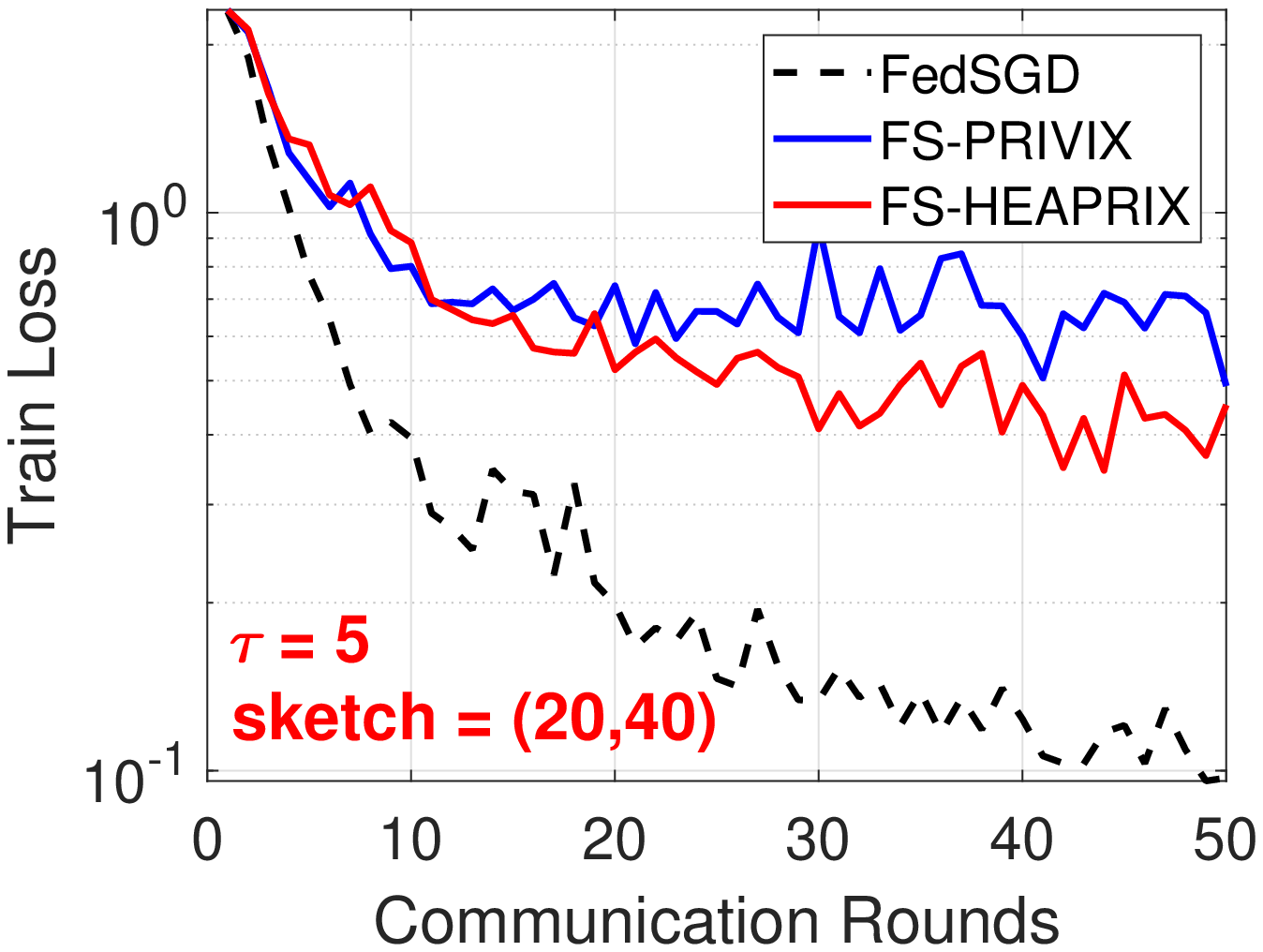}
		}
		\mbox{
		\includegraphics[width=2.1in]{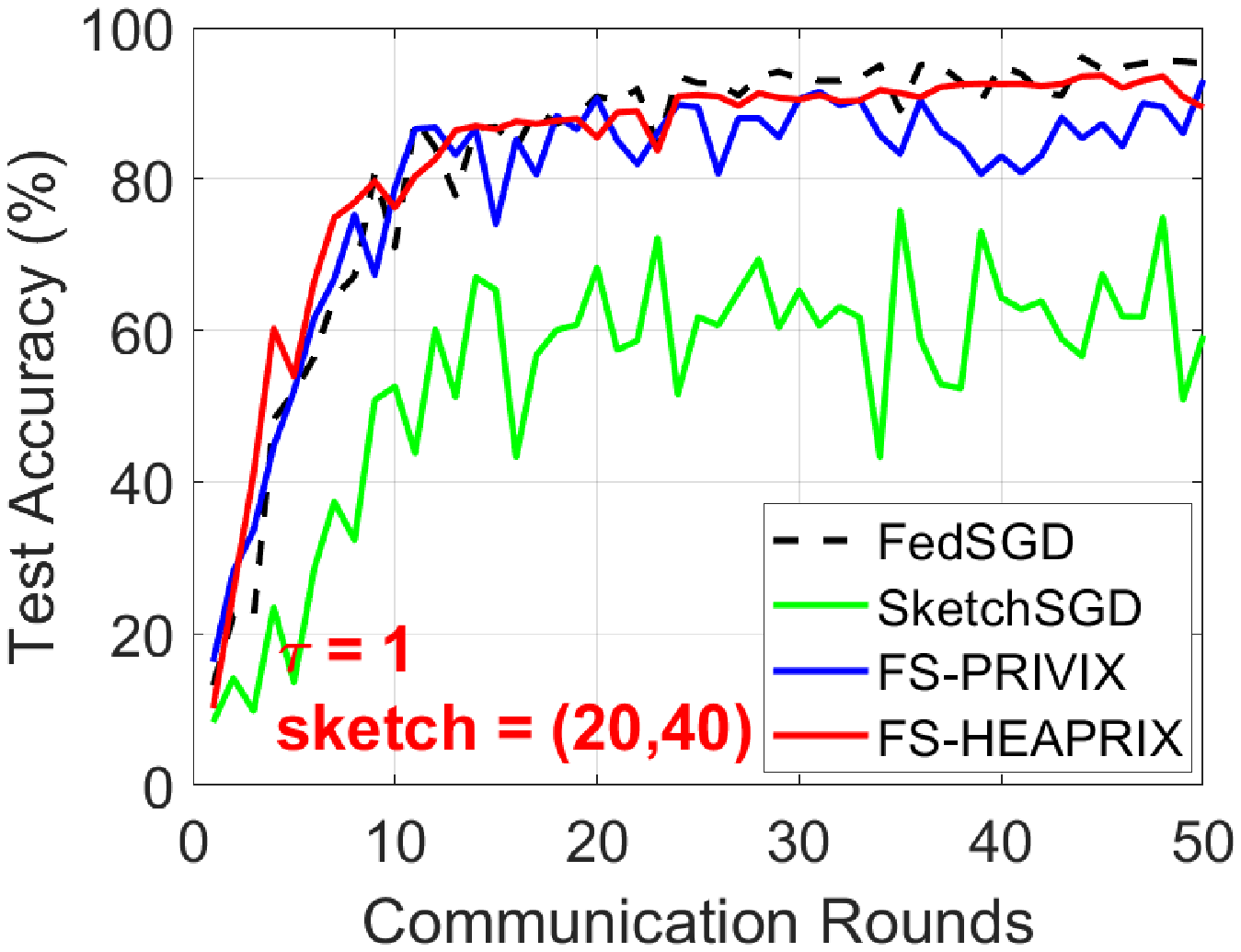} \hspace{-0.2in}
		\includegraphics[width=2.1in]{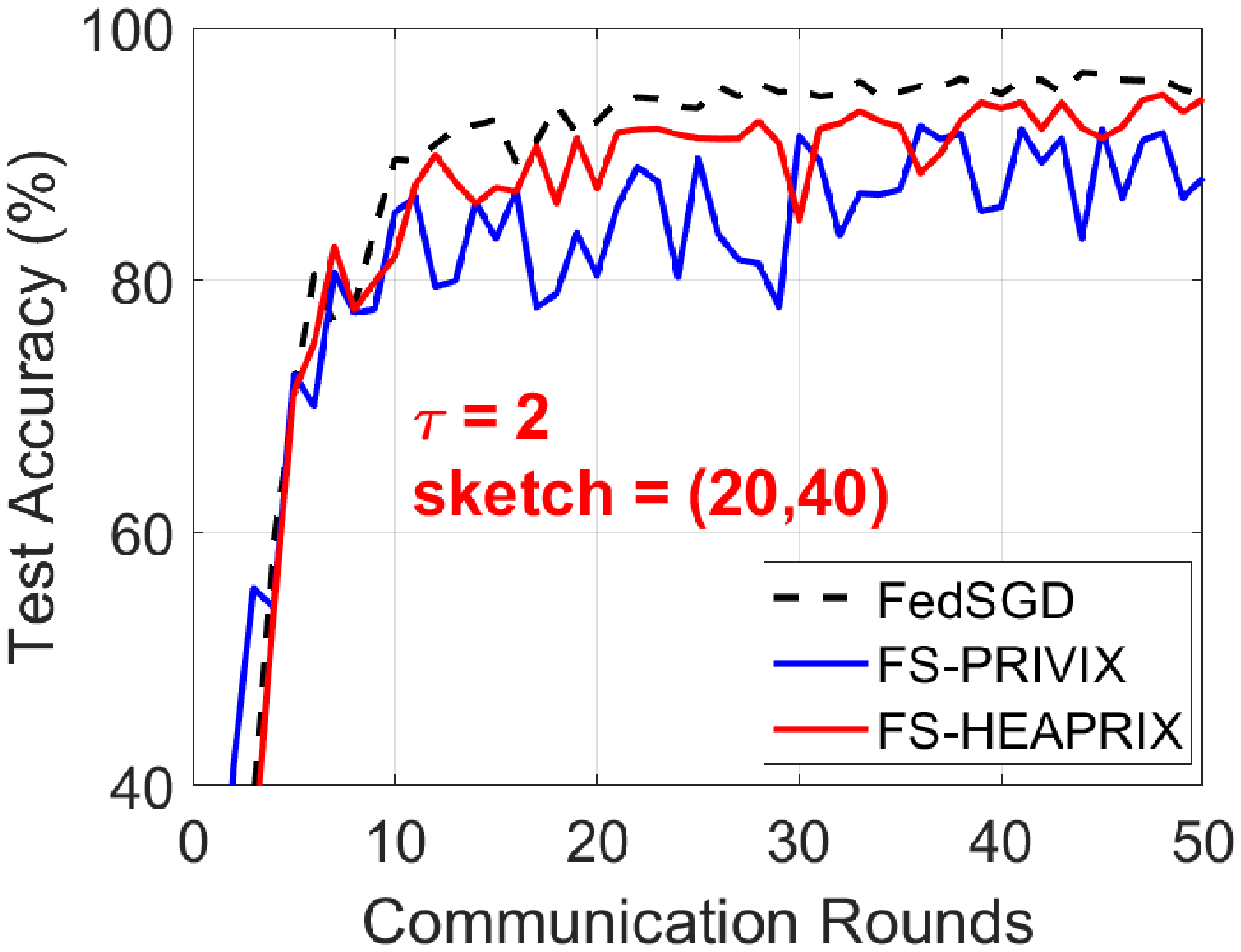} \hspace{-0.2in}
		\includegraphics[width=2.1in]{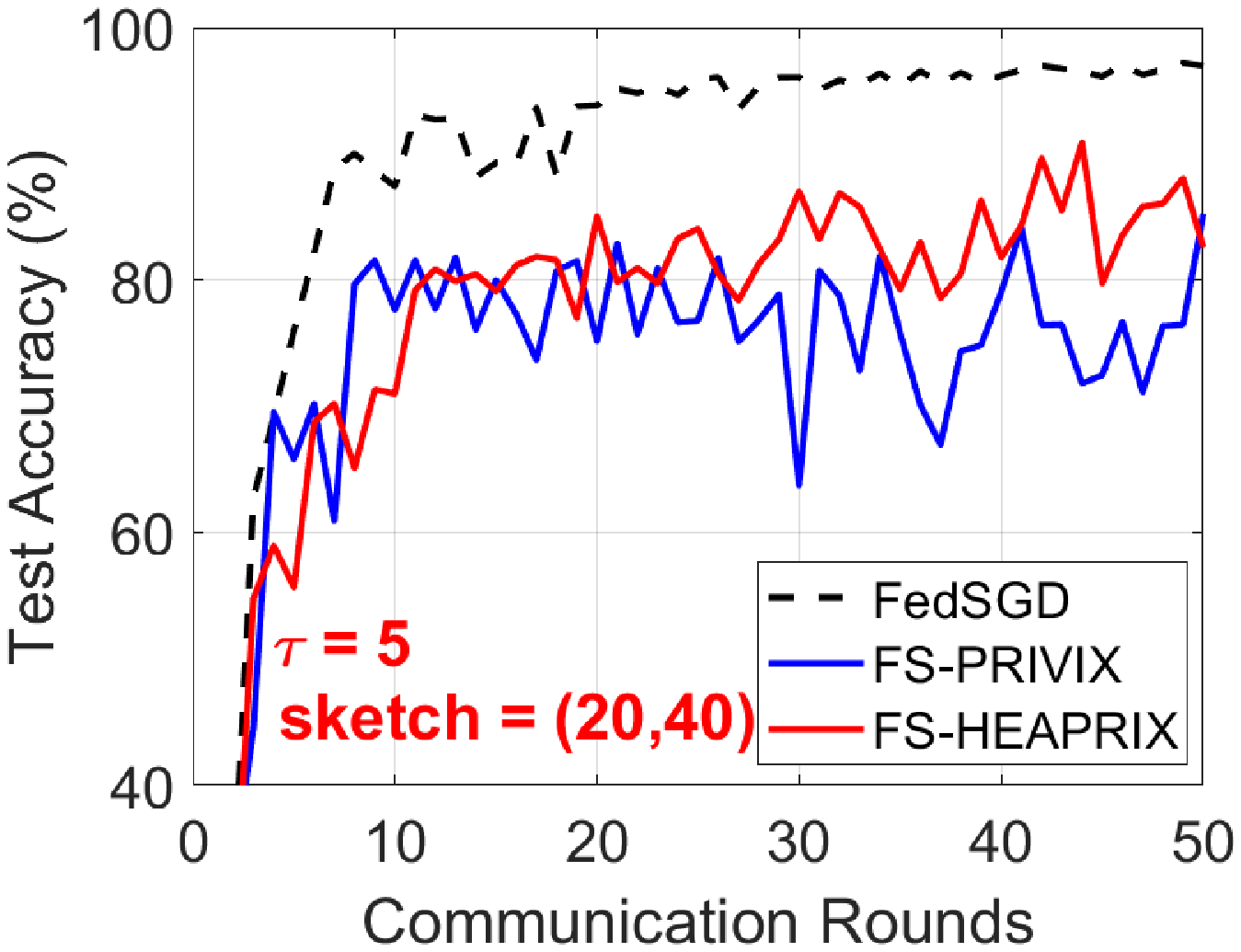}
		}
		\mbox{			    \includegraphics[width=2.1in]{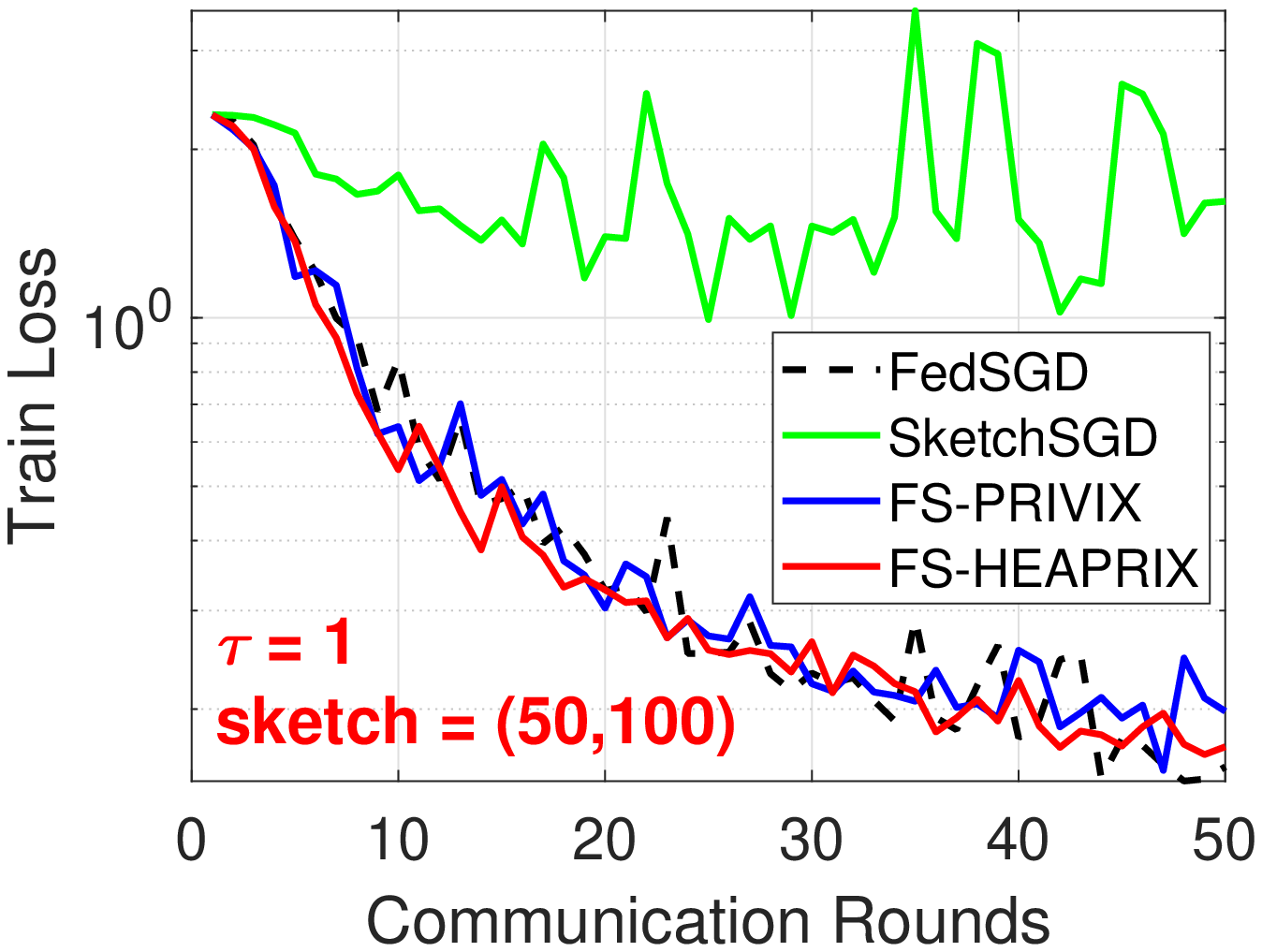} \hspace{-0.2in}
		\includegraphics[width=2.1in]{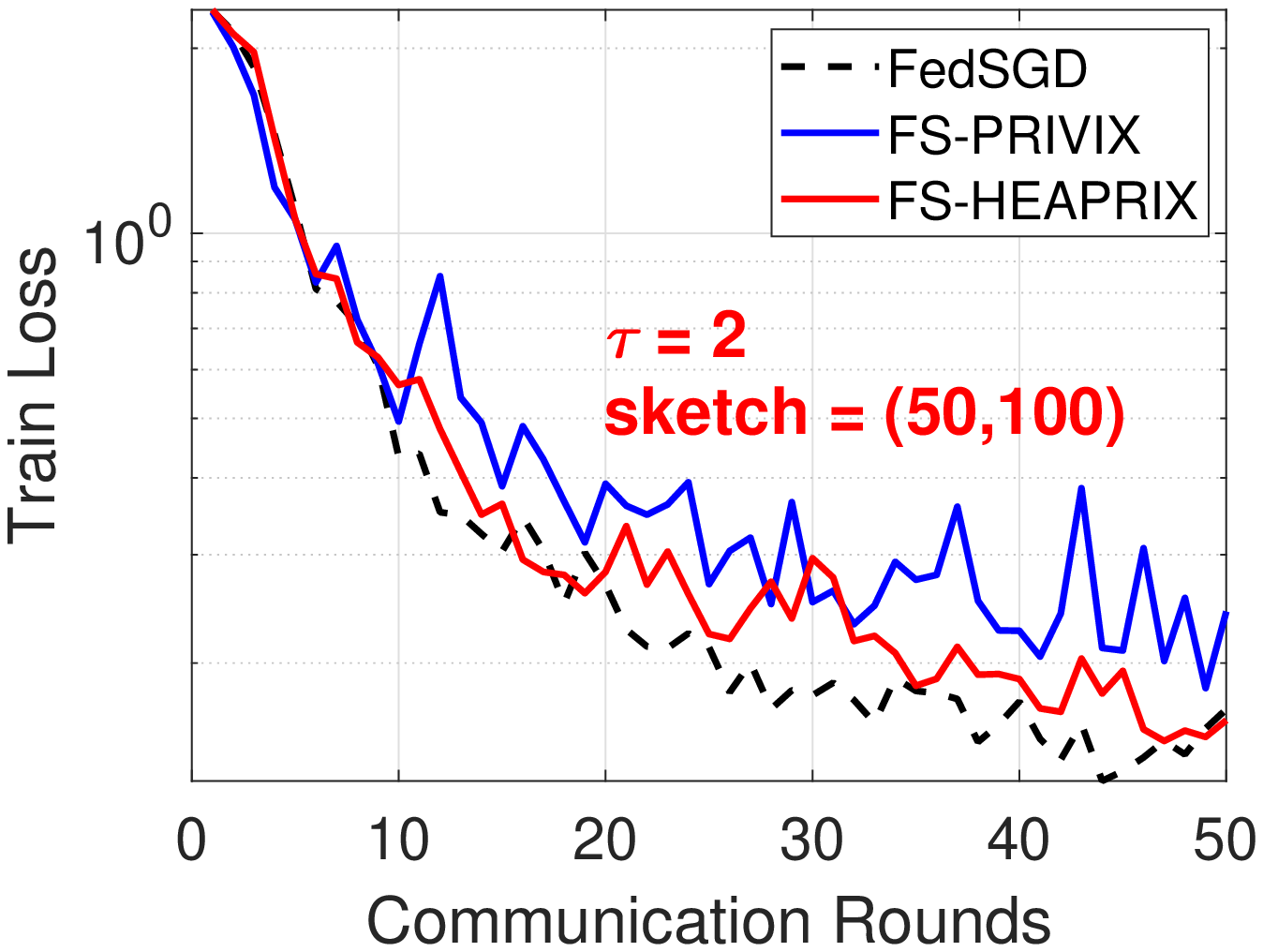} \hspace{-0.2in}
		\includegraphics[width=2.1in]{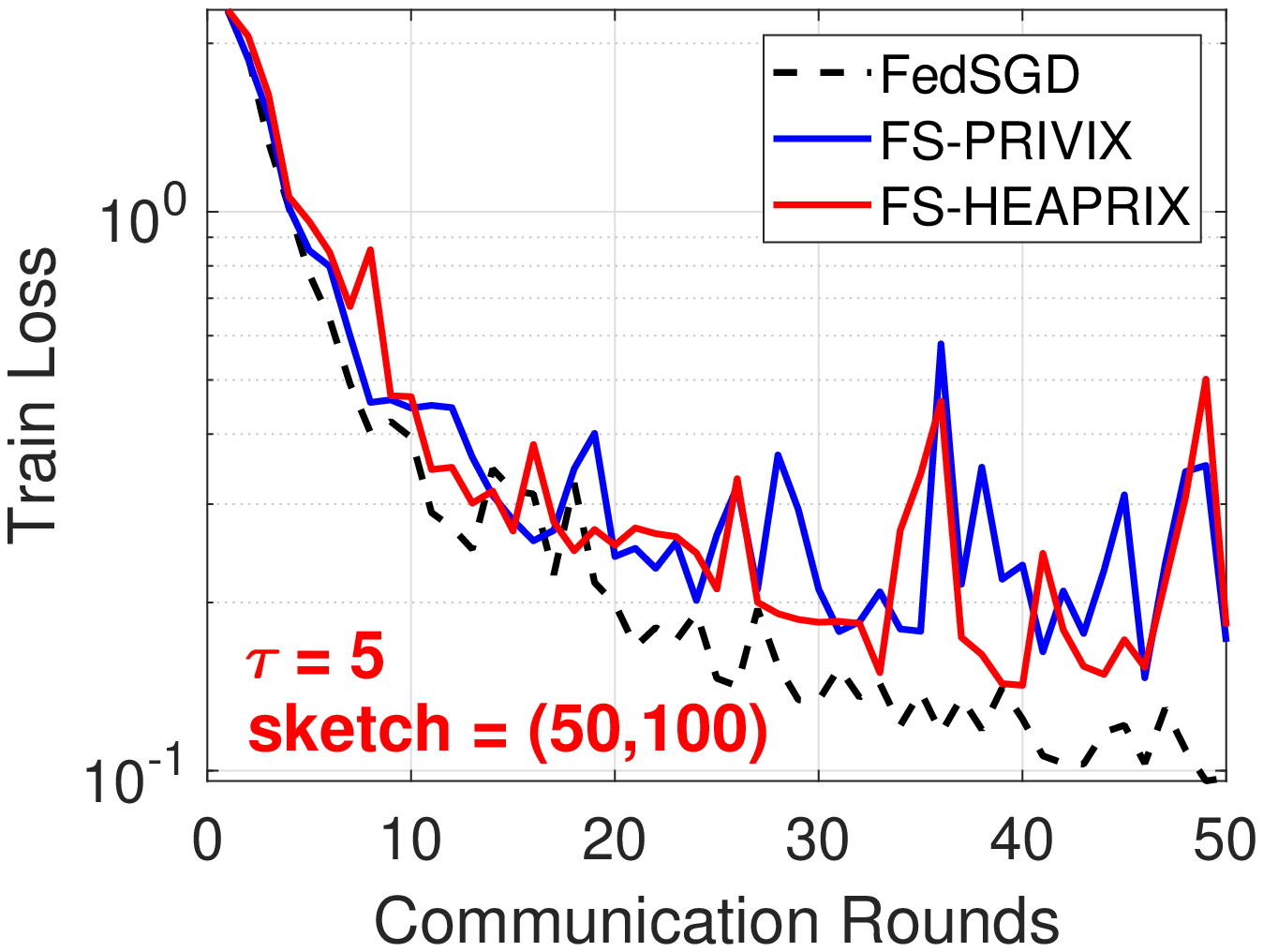}
		}
		\mbox{
		\includegraphics[width=2.1in]{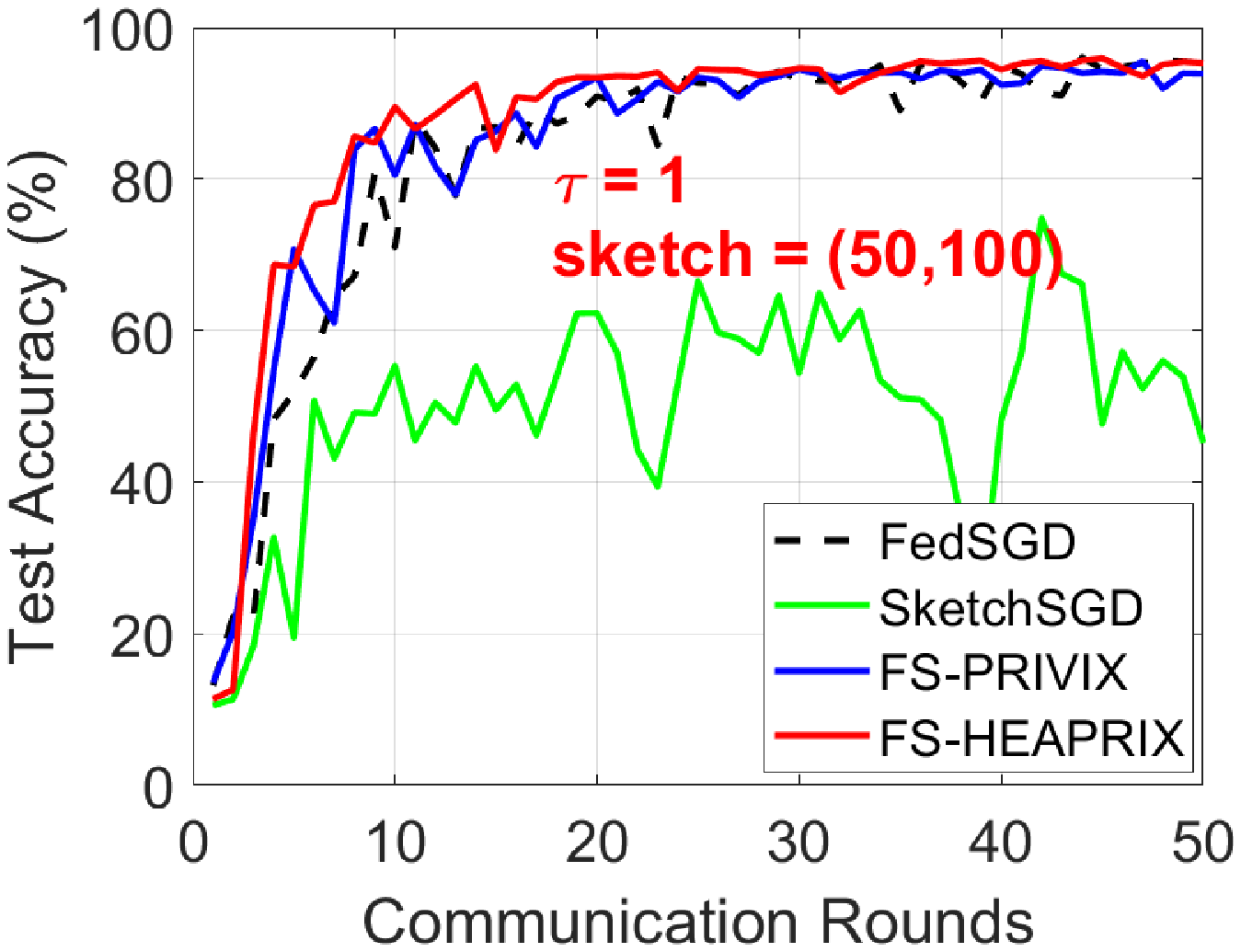} \hspace{-0.2in}
		\includegraphics[width=2.1in]{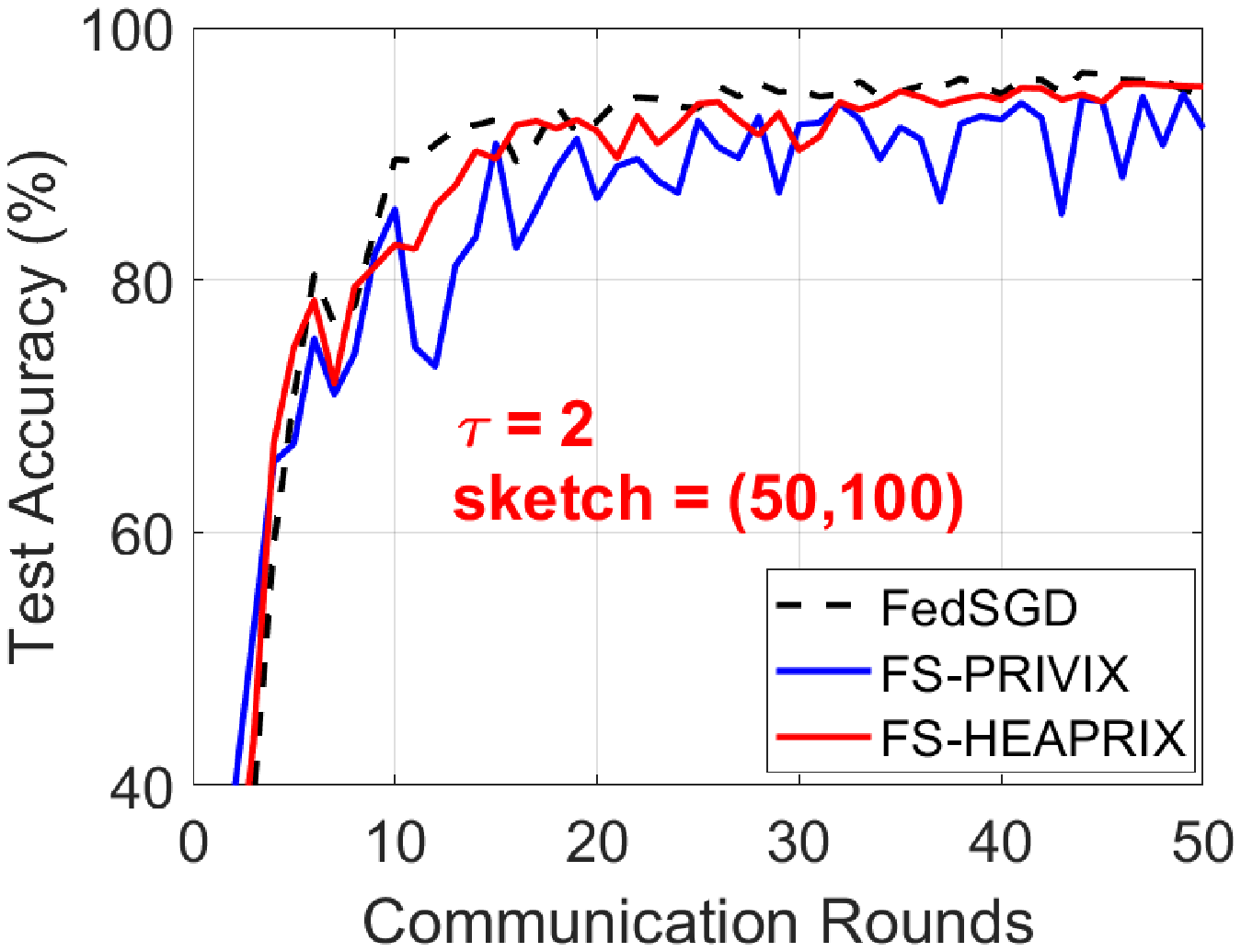} \hspace{-0.2in}
		\includegraphics[width=2.1in]{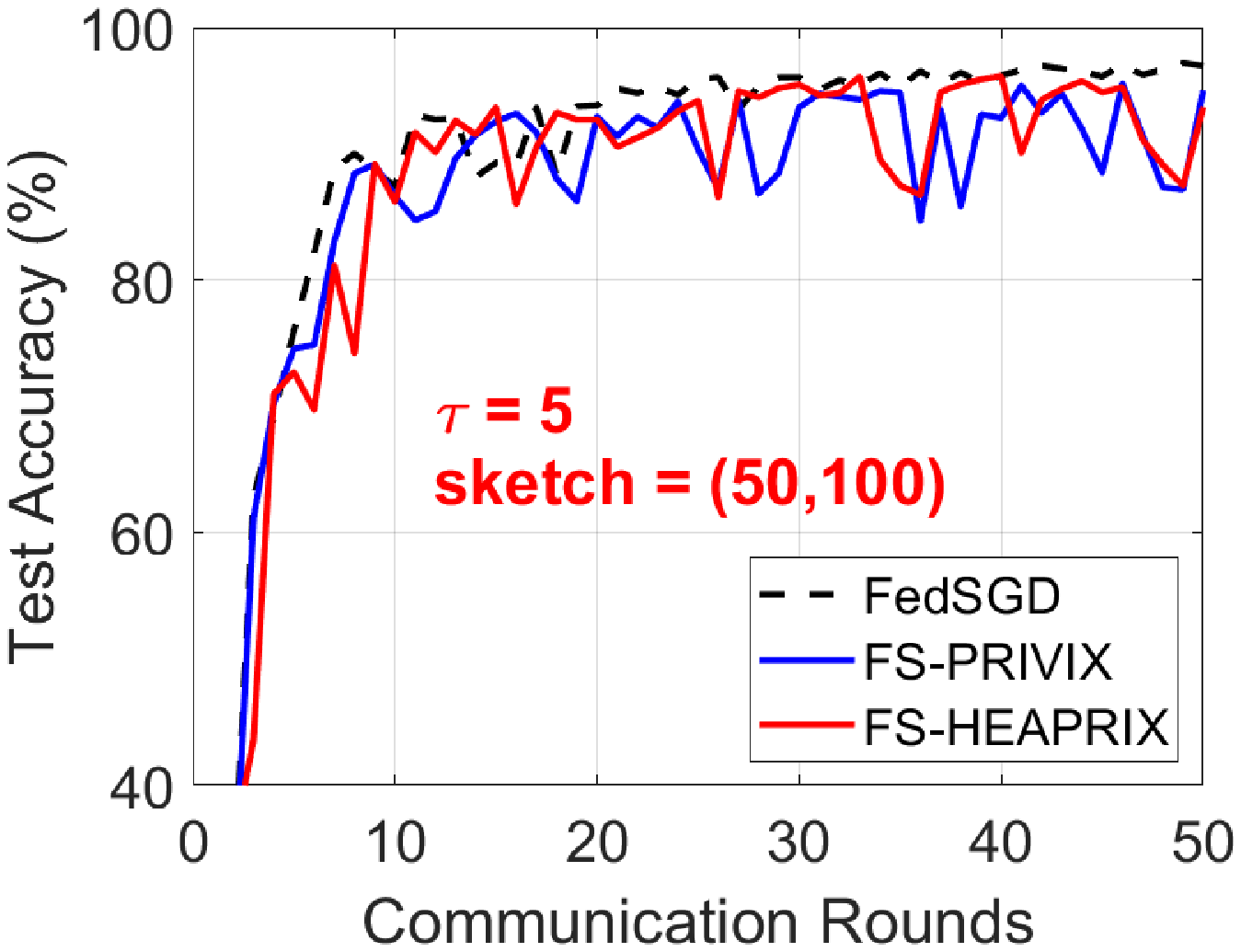}
		}
	\end{center}
	\caption{Heterogeneous case: Comparison of compressed optimization algorithms on LeNet CNN architecture.}
    \label{fig:MNIST-iid0}
\end{figure}

\textbf{Heterogeneous case.} We plot similar sets of results in Figure~\ref{fig:MNIST-iid0} for non-i.i.d. data distribution (heterogeneous setting). This setting leads to more twists and turns in the training curves.
From the first column ($\tau=1$), we see that SketchSGD performs very poorly in the heterogeneous case, while both our proposed FedSketchGATE methods, see Algorithm~\ref{Alg:PFLHet}, achieve similar generalization accuracy as the Federated SGD (FedSGD) algorithm, even with fairly small sketch size (i.e. $75\times$ compression ratio). Note that, the slow convergence of federated SGD in non-i.i.d. data distribution case has also been reported in literature, e.g.~\cite{mcmahan2016communication,chen2020toward}. In addition, FS-HEAPRIX is again better than FS-PRIVIX in terms of both training loss and test accuracy.

Furthermore, we notice in column 2 and 3 of Figure~\ref{fig:MNIST-iid0} the advantage of FS-HEAPRIX over FS-PRIVIX with multiple local updates. However, empirically we see that in the heterogeneous setting, more local updates $\tau$ tend to undermine the learning performance, especially with small sketch size.  Nevertheless, we see that when sketch size is large, i.e. $(50,100)$, FS-HEAPRIX can still provide comparable test accuracy as FedSGD with $\tau=5$.

Our empirical study demonstrates that our proposed FedSketch (and FedSketchGATE) frameworks are able to perform well in homogeneous (resp. heterogeneous) learning setting, with high compression rate. In particular, FedSketch methods are advantageous over prior SketchedSGD~\cite{ivkin2019communication} method in both cases. FS-HEAPRIX performs the best among all the tested compressed optimization algorithms, which in many cases achieves similar generalization accuracy as Federated SGD with small sketch size. In general, in any tested case, we can at least achieve $12\times$ compression ratio with very little loss in test accuracy.

\section{Conclusion}

In this paper, we introduced \texttt{FedSKETCH} and \texttt{FedSKETCHGATE} algorithms for homogeneous and heterogeneous data distribution setting respectively for Federated Learning wherein communication between server and devices is only performed using count sketch.
Our algorithms, thus, provide communication-efficiency and privacy.
We analyze the convergence error for \emph{nonconvex}, \emph{\pl} and \emph{general convex} objective functions in the scope of Federated Optimization.
We provide insightful numerical experiments showcasing the advantages of our \texttt{FedSKETCH} and \texttt{FedSKETCHGATE} methods over current federated optimization algorithm. The proposed algorithms outperform competing compression method and can achieve comparable test accuracy as Federated SGD, with high compression ratio.


\bibliographystyle{plain}
\bibliography{references}
\newpage
\appendix


\noindent\textbf{\LARGE Appendix}\\

\paragraph{Notation.} Here we indicate the count sketch of the vector $\boldsymbol{x}$ with $\mathbf{S}(\boldsymbol{x})$ and with abuse of notation we indicate the expectation over the randomness of count sketch with $\mathbb{E}_{\mathbf{S}}[.]$. We illustrate the random subset of the devices selected by server with $\mathcal{K}$ with size $|\mathcal{K}|=k\leq p$, and we represent the expectation over the device sampling with $\mathbb{E}_{\mathcal{K}}[.]$.

We will use the following fact (which is also used in \cite{li2019convergence,haddadpour2019convergence}) in proving results.
\begin{fact}[\cite{li2019convergence,haddadpour2019convergence}]\label{fact:1}
Let
$\{x_i\}_{i=1}^p$ denote any fixed deterministic sequence. We sample a multiset $\mathcal{P}$ (with size $K$) uniformly at random where $x_j$ is sampled  with probability $q_j$ for $1\leq j\leq p$ with replacement.  Let $\mathcal{P} = \{i_1,\ldots, i_K\} \subset[p]$ (some $i_j$’s may have the same value). Then
\begin{align}
    \mathbb{E}_{\mathcal{P}}\left[\sum_{i\in \mathcal{P}}x_i\right]=\mathbb{E}_{\mathcal{P}}\left[\sum_{k=1}^Kx_{i_k}\right]=K\mathbb{E}_{\mathcal{P}}\left[x_{i_k}\right]=K\left[\sum_{j=1}^pq_jx_j\right]
\end{align}
\end{fact}

\section{Results for the Homogeneous Setting}
\label{sec:app:sgd:undrr-pl}

In this section, we study the convergence properties of our  \texttt{FedSKETCH} method presented in Algorithm~\ref{Alg:PFLHom}. Before stating the proofs for \texttt{FedSKETCH} in the homogeneous setting, we first mention the following intermediate lemmas.

\begin{lemma}\label{lemma:tasbih1-iid}
Using unbiased compression and under Assumption~\ref{Assu:1.5}, we have the following bound:
\begin{align}
\mathbb{E}_{\mathcal{K}}\left[\mathbb{E}_{{\mathbf{S},\xi^{(r)}}}\Big[\|\tilde{\mathbf{g}}_{\mathbf{S}}^{(r)}\|^2\Big]\right]&=\mathbb{E}_{{\xi}^{(r)}}\mathbb{E}_{\mathbf{S}}\Big[\|\tilde{\mathbf{g}}_\mathbf{S}^{(r)}\|^2\Big]\leq \tau(\frac{\omega}{k}+1)\sum_{j=1}^mq_j\left[\sum_{c=0}^{\tau-1}\|\mathbf{g}_j^{(c,r)}\|^2+\sigma^2\right] \label{eq:lemma1}
\end{align}
\end{lemma}

\begin{proof}
\begin{align}
&\mathbb{E}_{{\xi^{(r)}|\boldsymbol{w}^{(r)}}}\mathbb{E}_{\mathcal{K}}\left[\mathbb{E}_{{\mathbf{S}}}\Big[\|\frac{1}{k}\sum_{j\in \mathcal{K}} \mathbf{S}\left(\sum_{c=0}^{\tau-1}\tilde{\mathbf{g}}^{(c,r)}_j\right)\|^2\Big]\right]\nonumber\\
=&\mathbb{E}_{{\xi}^{(r)}}\left[\mathbb{E}_{\mathcal{K}}\left[\mathbb{E}_{\mathbf{S}}\Big[\|\frac{1}{k}\sum_{j\in\mathcal{K}}\underbrace{\mathbf{S}\left(\overbrace{\sum_{c=0}^{\tau-1}\tilde{\mathbf{g}}^{(c,r)}_j}^{\tilde{\mathbf{g}}_j^{(r)}}\right)}_{\tilde{\mathbf{g}}_{\mathbf{S}j}^{(r)}}\|^2\Big]\right]\right]\nonumber\\
\stackrel{\text{\ding{192}}}{=}&\mathbb{E}_{{\xi}^{(r)}}\left[\mathbb{E}_{\mathcal{K}}\left[\left[\|\frac{1}{k}\sum_{j\in\mathcal{K}}\tilde{\mathbf{g}}_{\mathbf{S}j}^{(r)}-\frac{1}{k}\sum_{j\in\mathcal{K}}\mathbb{E}_{\mathbf{S}}\left[\tilde{\mathbf{g}}_{\mathbf{S}j}^{(r)}\right]\|^2\right]+\|\mathbb{E}_{\mathbf{S}}\left[\frac{1}{k}\sum_{j\in\mathcal{K}}\tilde{\mathbf{g}}_{\mathbf{S},j}^{(r)}\right]\|^2\right]\right]\nonumber\\
\stackrel{\text{\ding{193}}}{=}&\mathbb{E}_{{\xi}^{(r)}}\left[\mathbb{E}_{\mathcal{K}}\left[\mathbb{E}_{\mathbf{S}}\left[\|\frac{1}{k}\left[\sum_{j\in\mathcal{K}}\tilde{\mathbf{g}}_{\mathbf{S}j}^{(r)}-\sum_{j\in\mathcal{K}}\tilde{\mathbf{g}}_{j}^{(r)}\right]\|^2\right]+\|\frac{1}{k}\sum_{j\in\mathcal{K}}\tilde{\mathbf{g}}_{j}^{(r)}\|^2\right]\right]\nonumber\\
\stackrel{}{=}& \mathbb{E}_{{\xi}^{(r)}}\left[\mathbb{E}_{\mathcal{K}}\left[\left[\text{Var}_{\mathbf{S}}\left[\frac{1}{k}\sum_{j\in\mathcal{K}}\tilde{\mathbf{g}}_{\mathbf{S}j}^{(r)}\right]\right]+\|\frac{1}{k}\sum_{j\in\mathcal{K}}\tilde{\mathbf{g}}_{j}^{(r)}\|^2\right]\right]\nonumber\\
\stackrel{}{=}& \mathbb{E}_{{\xi}^{(r)}}\left[\mathbb{E}_{\mathcal{K}}\left[\frac{1}{k^2}\sum_{j\in\mathcal{K}}\text{Var}_{\mathbf{S}_j}\left[\tilde{\mathbf{g}}_{\mathbf{S}j}^{(r)}\right]+\|\frac{1}{k}\sum_{j\in\mathcal{K}}\tilde{\mathbf{g}}_{j}^{(r)}\|^2\right]\right]\nonumber\\
\stackrel{}{\leq}& \mathbb{E}_{{\xi}^{(r)}}\left[\mathbb{E}_{\mathcal{K}}\left[\frac{1}{k^2}\sum_{j\in\mathcal{K}}\omega\left\|\tilde{\mathbf{g}}_{j}^{(r)}\right\|^2+\|\frac{1}{k}\sum_{j\in\mathcal{K}}\tilde{\mathbf{g}}_{j}^{(r)}\|^2\right]\right]\nonumber\\
\stackrel{}{=}& \left[\mathbb{E}_{\xi}\left[\frac{1}{k}\sum_{j\in\mathcal{K}}\omega\left\|\tilde{\mathbf{g}}_{j}^{(r)}\right\|^2+\mathbb{E}_{\mathcal{K}}\mathbb{E}_{{\xi}^{(r)}}\|\frac{1}{k}\sum_{j\in\mathcal{K}}\tilde{\mathbf{g}}_{j}^{(r)}\|^2\right]\right]\nonumber\\
\stackrel{}{=} &\left[\mathbb{E}_{\xi}\left[\frac{\omega}{k}\sum_{j=1}^pq_j\left\|\tilde{\mathbf{g}}_{j}^{(r)}\right\|^2+\mathbb{E}_{\mathcal{K}}\left[\text{Var}\left(\frac{1}{k}\sum_{j\in\mathcal{K}}\tilde{\mathbf{g}}_{j}^{(r)}\right)+\|\frac{1}{k}\sum_{j\in\mathcal{K}}{\mathbf{g}}_{j}^{(r)}\|^2\right]\right]\right]\nonumber\\
\stackrel{}{=}& \frac{\omega}{k}\sum_{j=1}^pq_j\mathbb{E}_{\xi}\left\|\tilde{\mathbf{g}}_{j}^{(r)}\right\|^2+\mathbb{E}_{\mathcal{K}}\left[\frac{1}{k^2}\sum_{j\in\mathcal{K}}\text{Var}\left(\tilde{\mathbf{g}}_{j}^{(r)}\right)+\|\frac{1}{k}\sum_{j\in\mathcal{K}}{\mathbf{g}}_{j}^{(r)}\|^2\right]\nonumber\\
\stackrel{\ding{195}}{\leq}&\frac{\omega}{k}\sum_{j=1}^pq_j\mathbb{E}_{\xi}\left\|\tilde{\mathbf{g}}_{j}^{(r)}\right\|^2+\mathbb{E}_{\mathcal{K}}\left[\frac{1}{k^2}\sum_{j\in\mathcal{K}}\tau\sigma^2+\frac{1}{k}\sum_{j\in\mathcal{K}}\|{\mathbf{g}}_{j}^{(r)}\|^2\right]\nonumber\\
=&\frac{\omega}{k}\sum_{j=1}^pq_j\left[\text{Var}\left(\tilde{\mathbf{g}}_{j}^{(r)}\right)+\left\|\mathbf{g}_{j}^{(r)}\right\|^2\right]+\left[\frac{\tau\sigma^2}{k}+\sum_{j=1}^pq_j\|{\mathbf{g}}_{j}^{(r)}\|^2\right]\nonumber\\
\stackrel{\ding{196}}{\leq}&\frac{\omega}{k}\sum_{j=1}^pq_j\left[\tau\sigma^2+\left\|\mathbf{g}_{j}^{(r)}\right\|^2\right]+\left[\frac{\tau\sigma^2}{k}+\sum_{j=1}^pq_j\|{\mathbf{g}}_{j}^{(r)}\|^2\right]\nonumber\\
=&(\omega+1)\frac{\tau\sigma^2}{k}+(\frac{\omega}{k}+1)\left[\sum_{j=1}^pq_j\|{\mathbf{g}}_{j}^{(r)}\|^2\right]\label{eq:lemma111}
\end{align}
where \text{\ding{192}} holds due to $\mathbb{E}\left[\left\|\mathbf{x}\right\|^2\right]=\text{Var}[\mathbf{x}]+\left\|\mathbb{E}[\mathbf{x}]\right\|^2$, \text{\ding{193}} is due to $\mathbb{E}_{\mathbf{S}}\left[\frac{1}{p}\sum_{j=1}^p\tilde{\mathbf{g}}_{\mathbf{S}j}^{(r)}\right]=\frac{1}{p}\sum_{j=1}^m\tilde{\mathbf{g}}_{j}^{(r)}$.

Next we show that from Assumptions~\ref{Assu:2}, we have
\begin{align}\label{eq:100000}
    \mathbb{E}_{\xi^{(r)}}\left[\Big[\|{\tilde{\mathbf{g}}_j^{(r)}}-{\mathbf{g}_j^{(r)}}\|^2\Big]\right]\leq \tau \sigma^2
\end{align}
To do so, note that
\begin{align}
    \text{Var}\left(\tilde{\mathbf{g}}_{j}^{(r)}\right)&=\mathbb{E}_{\xi^{(r)}}\left[\left\|{\tilde{\mathbf{g}}_j^{(r)}}-{\mathbf{g}_j^{(r)}}\right\|^2\right]\nonumber\\
    &\stackrel{\text{\ding{192}}}{=}\mathbb{E}_{\xi^{(r)}}\left[\left\|\sum_{c=0}^{\tau-1}\left[\tilde{\mathbf{g}}_j^{(c,r)}-\mathbf{g}_j^{(c,r)}\right]\right\|^2\right]\nonumber\\
    &{=}\text{Var}\left(\sum_{c=0}^{\tau-1}\tilde{\mathbf{g}}_j^{(c,r)}\right)\nonumber\\
    &\stackrel{\text{\ding{193}}}{=}\sum_{c=0}^{\tau-1}\text{Var}\left(\tilde{\mathbf{g}}_j^{(c,r)}\right)\nonumber\\
    &{=}\sum_{c=0}^{\tau-1}\mathbb{E}\left[\left\|\tilde{\mathbf{g}}_j^{(c,r)}-\mathbf{g}_j^{(c,r)}\right\|^2\right]\nonumber\\
    &\stackrel{\text{\ding{194}}}{\leq}\tau\sigma^2\label{eq:var_b_mid}
    \end{align}
where in \text{\ding{192}} we use the definition of ${\tilde{\mathbf{g}}}_j^{(r)}$ and ${{\mathbf{g}}}_j^{(r)}$, in \text{\ding{193}} we use the fact that mini-batches are chosen in i.i.d. manner at each local machine, and \text{\ding{194}} immediately follows from Assumptions~\ref{Assu:1.5}.

Replacing $\mathbb{E}_{\xi^{(r)}}\left[\|{\tilde{\mathbf{g}}_j^{(r)}}-{\mathbf{g}_j^{(r)}}\|^2\right]$ in \eqref{eq:lemma111} by its upper bound in \eqref{eq:100000} implies that
\begin{align}
\mathbb{E}_{{\xi^{(r)}|\boldsymbol{w}^{(r)}}}\mathbb{E}_{\mathbf{S},\mathcal{K}}\Big[\|\frac{1}{k}\sum_{j\in\mathcal{K}} \mathbf{S}\left(\sum_{c=0}^{\tau-1}\tilde{\mathbf{g}}^{(c,r)}_j\right)\|^2\Big]
\leq (\omega+1)\frac{\tau\sigma^2}{k}+(\frac{\omega}{k}+1)\sum_{j=1}^pq_j\|{\mathbf{g}}_{j}^{(r)}\|^2\label{eq:lemma112}
\end{align}

Further note that we have
\begin{align}
\left\|{\mathbf{g}}_j^{(r)}\right\|^2&=\|\sum_{c=0}^{\tau-1}\mathbf{g}_j^{(c,r)}\|^2\stackrel{}{\leq} \tau\sum_{c=0}^{\tau-1}\|\mathbf{g}_j^{(c,r)}\|^2\label{eq:mid-bounding-absg}
\end{align}
where the last inequality is due to $\left\|\sum_{j=1}^n\mathbf{a}_i\right\|^2\leq n\sum_{j=1}^n\left\|\mathbf{a}_i\right\|^2$, which together with \eqref{eq:lemma112} leads to the following bound:
\begin{align}
    \mathbb{E}_{{\xi^{(r)}|\boldsymbol{w}^{(r)}}}\mathbb{E}_{\mathbf{S}}\Big[\|\frac{1}{k}\sum_{j\in\mathcal{K}} \mathbf{S}\left(\sum_{c=0}^{\tau-1}\tilde{\mathbf{g}}^{(c,r)}_j\right)\|^2\Big]\leq(\omega+1)\frac{\tau\sigma^2}{k}+\tau(\frac{\omega}{k}+1)\sum_{j=1}^pq_j\|{\mathbf{g}}_{j}^{(c,r)}\|^2,
\end{align}
and the proof is complete.
\end{proof}

\begin{lemma}\label{lemma:cross-inner-bound-unbiased}
  Under Assumption~\ref{Assu:1}, and according to the \texttt{FedCOM} algorithm the expected inner product between stochastic gradient and full batch gradient can be bounded with:
\begin{align}
    - \mathbb{E}_{\xi,\mathbf{S},\mathcal{K}}\left[\left\langle\nabla f({\boldsymbol{w}}^{(r)}),{{\tilde{\mathbf{g}}}^{(r)}}\right\rangle\right]&\leq \frac{1}{2}\eta\frac{1}{m}\sum_{j=1}^m\sum_{c=0}^{\tau-1}\left[-\|\nabla f({\boldsymbol{w}}^{(r)})\|_2^2-\|\nabla{f}(\boldsymbol{w}_j^{(c,r)})\|_2^2+L^2\|{\boldsymbol{w}}^{(r)}-\boldsymbol{w}_j^{(c,r)}\|_2^2\right]\label{eq:lemma3-thm2}
\end{align}

\end{lemma}
\begin{proof}
We have:
\begin{align}
    &-\mathbb{E}_{\{{\xi}^{(t)}_{1}, \ldots, {\xi}^{(t)}_{m}|{\boldsymbol{w}}^{(t)}_{1},\ldots,  {\boldsymbol{w}}^{(t)}_{m}\}} \mathbb{E}_{\mathbf{S},\mathcal{K}}\left[ \big\langle\nabla f({\boldsymbol{w}}^{(r)}),\tilde{\mathbf{g}}_{\mathbf{S},\mathcal{K}}^{(r)}\big\rangle\right]\nonumber\\
    =&-\mathbb{E}_{\{{\xi}^{(t)}_{1}, \ldots, {\xi}^{(t)}_{m}|{\boldsymbol{w}}^{(t)}_{1},\ldots,  {\boldsymbol{w}}^{(t)}_{m}\}}\left[\left\langle \nabla f({\boldsymbol{w}}^{(r)}),\eta\sum_{j\in\mathcal{K}}q_j\sum_{c=0}^{\tau-1}\tilde{\mathbf{g}}_j^{(c,r)}\right\rangle\right]\nonumber\\
    =&-\left\langle \nabla f({\boldsymbol{w}}^{(r)}),\eta\sum_{j=1}^mq_j\sum_{c=0}^{\tau-1}\mathbb{E}_{\xi,\mathbf{S}}\left[\tilde{\mathbf{g}}_{j,\mathbf{S}}^{(c,r)}\right]\right\rangle\nonumber\\
        &=-\eta\sum_{c=0}^{\tau-1}\sum_{j=1}^mq_j\left\langle \nabla f({\boldsymbol{w}}^{(r)}),{\mathbf{g}}_j^{(c,r)}\right\rangle\nonumber\\
     \stackrel{\text{\ding{192}}}{=}&\frac{1}{2}\eta\sum_{c=0}^{\tau-1}\sum_{j=1}^mq_j\left[-\|\nabla f({\boldsymbol{w}}^{(r)})\|_2^2-\|{{\nabla{f}}}(\boldsymbol{w}_j^{(c,r)})\|_2^2+\|\nabla f({\boldsymbol{w}}^{(r)})-\nabla{f}(\boldsymbol{w}_j^{(c,r)})\|_2^2\right]\nonumber\\
    \stackrel{\text{\ding{193}}}{\leq}&\frac{1}{2}\eta\sum_{c=0}^{\tau-1}\sum_{j=1}^mq_j\left[-\|\nabla f({\boldsymbol{w}}^{(r)})\|_2^2-\|\nabla{f}(\boldsymbol{w}_j^{(c,r)})\|_2^2+L^2\|{\boldsymbol{w}}^{(r)}-\boldsymbol{w}_j^{(c,r)}\|_2^2\right]
   \label{eq:bounding-cross-no-redundancy}
\end{align}

where \ding{192} is due to $2\langle \mathbf{a},\mathbf{b}\rangle=\|\mathbf{a}\|^2+\|\mathbf{b}\|^2-\|\mathbf{a}-\mathbf{b}\|^2$, and \ding{193} follows from Assumption \ref{Assu:1}.
\end{proof}

The following lemma bounds the distance of local solutions from global solution at $r$th communication round.
\begin{lemma}\label{lemma:dif-under-pl-sgd-iid}
Under Assumptions~\ref{Assu:1.5} we have:
\begin{align}\notag
      \mathbb{E}\left[\|{\boldsymbol{w}}^{(r)}-\boldsymbol{w}_j^{(c,r)}\|_2^2\right]&\leq\eta^2\tau\sum_{c=0}^{\tau-1}\left\|{\mathbf{g}}_j^{(c,r)}\right\|_2^2+\eta^2\tau\sigma^2
\end{align}

\end{lemma}

\begin{proof}
Note that
\begin{align}\notag
 \mathbb{E}\left[\left\|{\boldsymbol{w}}^{(r)}-\boldsymbol{w}_j^{(c,r)}\right\|_2^2\right]&=\mathbb{E}\left[\left\|{\boldsymbol{w}}^{(r)}-\left({\boldsymbol{w}}^{(r)}-\eta\sum_{k=0}^{c}\tilde{\mathbf{g}}_j^{(k,r)}\right)\right\|_2^2\right]\nonumber\\
 &=\mathbb{E}\left[\left\|\eta\sum_{k=0}^{c}\tilde{\mathbf{g}}_j^{(k,r)}\right\|_2^2\right]\nonumber\\
 &\stackrel{\text{\ding{192}}}{=}\mathbb{E}\left[\left\|\eta\sum_{k=0}^{c}\left(\tilde{\mathbf{g}}_j^{(k,r)}-{\mathbf{g}}_j^{(k,r)}\right)\right\|_2^2\right]+\left[\left\|\eta\sum_{k=0}^{c}{\mathbf{g}}_j^{(k,r)}\right\|_2^2\right]\nonumber\\
 &\stackrel{\text{\ding{193}}}{=}\eta^2\sum_{k=0}^{c}\mathbb{E}\left[\left\|\left(\tilde{\mathbf{g}}_j^{(k,r)}-{\mathbf{g}}_j^{(k,r)}\right)\right\|_2^2\right]+\left(c+1\right)\eta^2\sum_{k=0}^{c}\left[\left\|{\mathbf{g}}_j^{(k,r)}\right\|_2^2\right]\nonumber\\
  &{\leq}\eta^2\sum_{k=0}^{\tau-1}\mathbb{E}\left[\left\|\left(\tilde{\mathbf{g}}_j^{(k,r)}-{\mathbf{g}}_j^{(k,r)}\right)\right\|_2^2\right]+\tau\eta^2\sum_{k=0}^{\tau-1}\left[\left\|{\mathbf{g}}_j^{(k,r)}\right\|_2^2\right]\nonumber\\
  &\stackrel{\text{\ding{194}}}{\leq}\eta^2\sum_{k=0}^{\tau-1}\sigma^2+\tau\eta^2\sum_{k=0}^{\tau-1}\left[\left\|{\mathbf{g}}_j^{(k,r)}\right\|_2^2\right]\nonumber\\
 &{=}\eta^2\tau\sigma^2+\eta^2\sum_{k=0}^{\tau-1}\tau\left\|{\mathbf{g}}_j^{(k,r)}\right\|_2^2
\end{align}
where \ding{192} comes from $\mathbb{E}\left[\mathbf{x}^2\right]=\text{Var}\left[\mathbf{x}\right]+\left[\mathbb{E}\left[\mathbf{x}\right]\right]^2$ and \ding{193} holds because $\text{Var}\left(\sum_{j=1}^n\mathbf{x}_j\right)=\sum_{j=1}^n\text{Var}\left(\mathbf{x}_j\right)$ for i.i.d. vectors $\mathbf{x}_i$ (and i.i.d. assumption comes from i.i.d. sampling), and finally \ding{194} follows from Assumption~\ref{Assu:1.5}.
\end{proof}

\subsection{Main result for the nonconvex setting}
Now we are ready to present our result for the homogeneous setting. We first state and prove the result for the general nonconvex objectives.
\begin{theorem}[Nonconvex]\label{thm:lsgwd-lr} For \texttt{FedSKETCH}$(\tau, \eta, \gamma)$, for all $0\leq t\leq R\tau-1$,  under Assumptions \ref{Assu:1} to \ref{Assu:1.5}, if the learning rate satisfies \begin{align}
   1\geq {\tau^2 L^2\eta^2}+\left(\frac{\omega}{k}+1\right){\eta\gamma L}{\tau}
\label{eq:cnd-thm4.3}
\end{align}
and all local model parameters are initialized at the same point ${\boldsymbol{w}}^{(0)}$, then the average-squared gradient after $\tau$ iterations is bounded as follows:
\begin{align}
        \frac{1}{R}\sum_{r=0}^{R-1}\left\|\nabla f({\boldsymbol{w}}^{(r)})\right\|_2^2\leq \frac{2\left(f(\boldsymbol{w}^{(0)})-f(\boldsymbol{w}^{(*)})\right)}{\eta\gamma\tau R}+\frac{L\eta\gamma{\left(\omega+1\right)}}{k}\sigma^2+{L^2\eta^2\tau }\sigma^2\label{eq:thm1-result}
\end{align}
where $\boldsymbol{w}^{(*)}$ is the global optimal solution with  function value $f(\boldsymbol{w}^{(*)})$.
\end{theorem}

\begin{proof}
Before proceeding to the proof of Theorem~\ref{thm:lsgwd-lr}, we would like to highlight that
\begin{align}
    \boldsymbol{w}^{(r)}- ~{\boldsymbol{w}}_{j}^{(\tau,r)}=\eta\sum_{c=0}^{\tau-1}\tilde{\mathbf{g}}_j^{(c,r)}.\label{eq:decent-smoothe}
\end{align}

From the updating rule of Algorithm~\ref{Alg:PFLHom} we have

{
\begin{align}\notag
     {\boldsymbol{w}}^{(r+1)}=\boldsymbol{w}^{(r)}-\gamma\eta\left(\frac{1}{k}\sum_{j\in\mathcal{K}}\mathbf{S}\Big(\sum_{c=0,r}^{\tau-1}\tilde{\mathbf{g}}_{j}^{(c,r)}\Big)\right)=\boldsymbol{w}^{(r)}-\gamma\left[\frac{\eta}{k}\sum_{j\in\mathcal{K}}\mathbf{S}\left(\sum_{c=0}^{\tau-1}\tilde{\mathbf{g}}_{j}^{(c,r)}\right)\right]
\end{align}
}
In what follows, we use the following notation to denote the stochastic gradient used to update the global model at $r$th communication round $$\tilde{\mathbf{g}}_{\mathbf{S},\mathcal{K}}^{(r)}\triangleq\frac{\eta}{p}\sum_{j=1}^{p}\mathbf{S}\left(\frac{\boldsymbol{w}^{(r)}- ~{\boldsymbol{w}}_{j}^{(\tau,r)}}{\eta}\right)=\frac{1}{k}\sum_{j\in\mathcal{K}}\mathbf{S}\left(\sum_{c=0}^{\tau-1}\tilde{\mathbf{g}}_j^{(c,r)}\right).$$
and notice that $\boldsymbol{w}^{(r)} = \boldsymbol{w}^{(r-1)} - \gamma \tilde{\mathbf{g}}^{(r)}$.

Then using the unbiased estimation property of sketching we have:
\begin{align}\notag
  \mathbb{E}_\mathbf{S}\left[\tilde{\mathbf{g}}_\mathbf{S}^{(r)}\right]=\frac{1}{k}\sum_{j\in\mathcal{K}}\left[-\eta\mathbb{E}_\mathbf{S}\left[ \mathbf{S}\left(\sum_{c=0}^{\tau-1}\tilde{\mathbf{g}}_j^{(c,r)}\right)\right]\right]=\frac{1}{k}\sum_{j\in\mathcal{K}}\left[-\eta\left(\sum_{c=0}^{\tau-1}\tilde{\mathbf{g}}_j^{(c,r)}\right)\right]\triangleq \tilde{\mathbf{g}}_{\mathbf{S},\mathcal{K}}^{(r)}
\end{align}


From the $L$-smoothness gradient assumption on global objective, by using  $\tilde{\mathbf{g}}^{(r)}$ in inequality Eq.~(\ref{eq:decent-smoothe}) we have:
\begin{align}
    f({\boldsymbol{w}}^{(r+1)})-f({\boldsymbol{w}}^{(r)})\leq -\gamma \big\langle\nabla f({\boldsymbol{w}}^{(r)}),\tilde{\mathbf{g}}^{(r)}\big\rangle+\frac{\gamma^2 L}{2}\|\tilde{\mathbf{g}}^{(r)}\|^2\label{eq:Lipschitz-c1}
\end{align}
By taking expectation on both sides of above inequality over sampling, we get:
\begin{align}
    \mathbb{E}\left[\mathbb{E}_\mathbf{S}\Big[f({\boldsymbol{w}}^{(r+1)})-f({\boldsymbol{w}}^{(r)})\Big]\right]&\leq -\gamma\mathbb{E}\left[\mathbb{E}_\mathbf{S}\left[ \big\langle\nabla f({\boldsymbol{w}}^{(r)}),\tilde{\mathbf{g}}_\mathbf{S}^{(r)}\big\rangle\right]\right]+\frac{\gamma^2 L}{2}\mathbb{E}\left[\mathbb{E}_\mathbf{S}\|\tilde{\mathbf{g}}_\mathbf{S}^{(r)}\|^2\right]\nonumber\\
    &\stackrel{(a)}{=}-\gamma\underbrace{\mathbb{E}\left[\left[ \big\langle\nabla f({\boldsymbol{w}}^{(r)}),\tilde{\mathbf{g}}^{(r)}\big\rangle\right]\right]}_{(\mathrm{I})}+\frac{\gamma^2 L}{2}\underbrace{\mathbb{E}\left[\mathbb{E}_\mathbf{S}\Big[\|\tilde{\mathbf{g}}_\mathbf{S}^{(r)}\|^2\Big]\right]}_{\mathrm{(II)}}\label{eq:Lipschitz-c-gd}
\end{align}
We proceed to use Lemma~\ref{lemma:tasbih1-iid}, Lemma~\ref{lemma:cross-inner-bound-unbiased}, and Lemma~\ref{lemma:dif-under-pl-sgd-iid}, to bound  terms $(\mathrm{I})$ and $(\mathrm{II})$ in right hand side of Eq.~(\ref{eq:Lipschitz-c-gd}), which gives
\begin{align}
     &\mathbb{E}\left[\mathbb{E}_{\mathbf{S}}\Big[f({\boldsymbol{w}}^{(r+1)})-f({\boldsymbol{w}}^{(r)})\Big]\right]\nonumber\\
     \leq& \gamma\frac{1}{2}\eta\sum_{j=1}^pq_j\sum_{c=0}^{\tau-1}\left[-\left\|\nabla f({\boldsymbol{w}}^{(r)})\right\|_2^2-\left\|\mathbf{g}_j^{(c,r)}\right\|_2^2+L^2\eta^2\sum_{c=0}^{\tau-1}\left[\tau\left\|{\mathbf{g}}_j^{(c,r)}\right\|_2^2+\sigma^2\right]\right]\nonumber\\
     &\quad+\frac{\gamma^2 L(\frac{\omega}{k}+1)}{2}\left[{\eta^2\tau}\sum_{j=1}^pq_j\sum_{c=0}^{\tau-1}\|\mathbf{g}^{(c,r)}_{j}\|^2\right]+\frac{\gamma^2\eta^2 L(\omega+1)}{2}\frac{\tau \sigma^2}{k}\nonumber\\
     \stackrel{\text{\ding{192}}}{\leq}&\frac{\gamma\eta}{2}\sum_{j=1}^pq_j\sum_{c=0}^{\tau-1}\left[-\left\|\nabla f({\boldsymbol{w}}^{(r)})\right\|_2^2-\left\|\mathbf{g}_j^{(c,r)}\right\|_2^2+\tau L^2\eta^2\left[\tau\left\|{\mathbf{g}}_j^{(c,r)}\right\|_2^2+\sigma^2\right]\right]\nonumber\\
     &\quad+\frac{\gamma^2 L(\frac{\omega}{k}+1)}{2}\left[{\eta^2\tau}\sum_{j=1}^pq_j\sum_{c=0}^{\tau-1}\|\mathbf{g}^{(c,r)}_{j}\|^2\right]+\frac{\gamma^2\eta^2 L(\omega+1)}{2}\frac{\tau \sigma^2}{k}\nonumber\\
     =&-\eta\gamma\frac{\tau}{2}\left\|\nabla f({\boldsymbol{w}}^{(r)})\right\|_2^2\nonumber\\
     &\quad-\left(1-{\tau L^2\eta^2\tau}-{(\frac{\omega}{k}+1)\eta\gamma L}{\tau}\right)\frac{\eta\gamma}{2}\sum_{j=1}^pq_j\sum_{c=0}^{\tau-1}\|\mathbf{g}^{(c,r)}_{j}\|^2+\frac{L\tau\gamma\eta^2 }{2k}\left(kL\tau\eta+\gamma(\omega+1)\right)\sigma^2\nonumber\\
     \stackrel{\text{\ding{193}}}{\leq}& -\eta\gamma\frac{\tau}{2}\left\|\nabla f({\boldsymbol{w}}^{(r)})\right\|_2^2+\frac{L\tau\gamma\eta^2 }{2k}\left(kL\tau\eta+\gamma(\omega+1)\right)\sigma^2\label{eq:finalll}
\end{align}
where in \ding{192} we incorporate outer summation $\sum_{c=0}^{\tau-1}$, and  \ding{193} follows from condition
\begin{align}\notag
   1\geq {\tau L^2\eta^2\tau}+(\frac{\omega}{k}+1)\eta\gamma L{\tau}.
\end{align}
Summing up for all $R$ communication rounds and  rearranging the terms gives:
\begin{align}\notag
    \frac{1}{R}\sum_{r=0}^{R-1}\left\|\nabla f({\boldsymbol{w}}^{(r)})\right\|_2^2\leq \frac{2\left(f(\boldsymbol{w}^{(0)})-f(\boldsymbol{w}^{(*)})\right)}{\eta\gamma\tau R}+\frac{L\eta\gamma{(\omega+1)}}{k}\sigma^2+{L^2\eta^2\tau }\sigma^2
\end{align}
From above inequality, is it easy to see that in order to achieve a linear speed up, we need to have $\eta\gamma=O\left(\frac{\sqrt{k}}{\sqrt{R \tau}}\right)$.
\end{proof}

\begin{corollary}[Linear speed up]
In Eq.~(\ref{eq:thm1-result}) for the choice of  $\eta\gamma=O\left(\frac{1}{L}\sqrt{\frac{k}{R\tau\left(\omega+1\right)}}\right)$, and $\gamma\geq k$  the  convergence rate reduces to:
\begin{align}
    \frac{1}{R}\sum_{r=0}^{R-1}\left\|\nabla f({\boldsymbol{w}}^{(r)})\right\|_2^2&\leq O\left(\frac{L\sqrt{\left(\omega+1\right)}\left(f(\boldsymbol{w}^{(0)})-f(\boldsymbol{w}^{*})\right)}{\sqrt{kR\tau}}+\frac{\left(\sqrt{\left(\omega+1\right)}\right)\sigma^2}{\sqrt{kR\tau}}+\frac{k\sigma^2}{R\gamma^2}\right).\label{eq:convg-error}
\end{align}
Note that according to Eq.~(\ref{eq:convg-error}), if we pick  a fixed constant value for  $\gamma$, in order to achieve an $\epsilon$-accurate solution, $R=O\left(\frac{1}{\epsilon}\right)$ communication rounds and $\tau=O\left(\frac{\omega+1}{k\epsilon}\right)$ local updates are necessary. We also highlight  that Eq.~(\ref{eq:convg-error}) also allows us to choose $R=O\left(\frac{\omega+1}{\epsilon}\right)$ and $\tau=O\left(\frac{1}{k\epsilon}\right)$ to get the  same convergence rate.
\end{corollary}

\begin{remark}\label{rmk:cnd-lr}

Condition in Eq.~(\ref{eq:cnd-thm4.3}) can be rewritten as
\begin{align}
    \eta&\leq \frac{-\gamma L\tau\left(\frac{\omega}{k}+1\right)+\sqrt{\gamma^2 \left(L\tau\left(\frac{\omega}{k}+1\right)\right)^2+4L^2\tau^2}}{2L^2\tau^2}\nonumber\\
    &= \frac{-\gamma L\tau\left(\frac{\omega}{k}+1\right)+L\tau\sqrt{\left(\frac{\omega}{k}+1\right)^2\gamma^2 +4}}{2L^2\tau^2}\nonumber\\
    &=\frac{\sqrt{\left(\frac{\omega}{k}+1\right)^2\gamma^2 +4}-\left(\frac{\omega}{k}+1\right)\gamma}{2L\tau}\label{eq:lrcnd}
\end{align}

So based on Eq.~(\ref{eq:lrcnd}), if we set $\eta=O\left(\frac{1}{L\gamma}\sqrt{\frac{p}{R\tau\left(\omega+1\right)}}\right)$, it implies that:
\begin{align}
    R\geq \frac{\tau k}{\left(\omega+1\right)\gamma^2\left(\sqrt{\left(\frac{\omega}{k}+1\right)^2\gamma^2+4}-\left(\frac{\omega}{k}+1\right)\gamma\right)^2}\label{eq:iidexact}
\end{align}
We note that $\gamma^2\left(\sqrt{\left(\frac{\omega}{k}+1\right)^2\gamma^2+4}-\left(\frac{\omega}{k}+1\right)\gamma\right)^2=\Theta(1)\leq 5 $ therefore even for $\gamma\geq m$ we need to have
\begin{align}
    R\geq \frac{\tau k}{5\left(\omega+1\right)}=O\left(\frac{\tau k}{\left(\omega+1\right)}\right)\label{eq:lrbnd-homog}
\end{align}

{Therefore, for the choice of $\tau=O\left(\frac{\omega+1}{k\epsilon}\right)$, due to condition in Eq.~(\ref{eq:lrbnd-homog}), we need to have $R=O\left(\frac{1}{\epsilon}\right)$. Similarly, we can have $R=O\left(\frac{\omega+1}{\epsilon}\right)$ and $\tau=O\left(\frac{1}{k\epsilon}\right)$.}

\end{remark}

\begin{corollary}[Special case, $\gamma=1$]
By letting $\gamma=1$, $\omega=0$ and $k=p$ the convergence rate in Eq.~(\ref{eq:thm1-result}) reduces to
\begin{align}\notag
     \frac{1}{R}\sum_{r=0}^{R-1}\left\|\nabla f({\boldsymbol{w}}^{(r)})\right\|_2^2&\leq \frac{2\left(f(\boldsymbol{w}^{(0)})-f(\boldsymbol{w}^{(*)})\right)}{\eta R\tau}+\frac{L\eta }{p}\sigma^2+{L^2\eta^2\tau }\sigma^2
\end{align}
which matches the rate  obtained in~\cite{wang2018cooperative}. In this case the communication complexity and the number of local updates become
\begin{align}\notag
    {R}=O\left(\frac{p}{\epsilon}\right), \:\:\: \tau=O\left(\frac{1}{\epsilon}\right).
\end{align}
This simply implies  that in this special case the convergence rate of our algorithm reduces to the  rate obtained in~\cite{wang2018cooperative}, which indicates the tightness of  our analysis.
\end{corollary}

\subsection{Main result for the PL/Strongly convex setting}

We now turn to stating the convergence rate for the homogeneous setting under PL condition which naturally leads to the same rate for strongly convex functions.
\begin{theorem}[PL or strongly convex]\label{thm:pl-iid}
For \texttt{FedSKETCH}$(\tau, \eta, \gamma)$, for all $0\leq t\leq R\tau-1$,  under Assumptions \ref{Assu:1} to \ref{Assu:1.5} and \ref{assum:pl},if the learning rate satisfies
\begin{align}\notag
   1\geq {\tau^2 L^2\eta^2}+\left(\frac{\omega}{k}+1\right){\eta\gamma L}{\tau}
\end{align}

and if the all the models are initialized with $\boldsymbol{w}^{(0)}$ we obtain:
\begin{align}\notag
        \mathbb{E}\Big[f({\boldsymbol{w}}^{(R)})-f({\boldsymbol{w}}^{(*)})\Big]&\leq \left(1-\eta\gamma{\mu\tau}\right)^R\left(f(\boldsymbol{w}^{(0)})-f(\boldsymbol{w}^{(*)})\right)+\frac{1}{{\mu}}\left[\frac{1}{2} L^2\tau\eta^2\sigma^2+\left(1+\omega\right)\frac{\gamma\eta L\sigma^2}{2k}\right]
\end{align}
\end{theorem}

\begin{proof}
From Eq.~(\ref{eq:finalll}) under condition:
\begin{align}\notag
       1\geq {\tau L^2\eta^2\tau}+{{(\frac{\omega}{k}+1)}\eta\gamma L}{\tau}
\end{align}
we obtain:
\begin{align}\notag
         \mathbb{E}\Big[f({\boldsymbol{w}}^{(r+1)})-f({\boldsymbol{w}}^{(r)})\Big]&\leq -\eta\gamma\frac{\tau}{2}\left\|\nabla f({\boldsymbol{w}}^{(r)})\right\|_2^2+\frac{L\tau\gamma\eta^2 }{2k}\left(kL\tau\eta+\gamma(\omega+1)\right)\sigma^2\nonumber\\
         &\leq -\eta\mu\gamma{\tau} \left(f({\boldsymbol{w}}^{(r)})-f({\boldsymbol{w}}^{(r)})\right)+\frac{L\tau\gamma\eta^2 }{2k}\left(kL\tau\eta+\gamma(\omega+1)\right)\sigma^2
\end{align}
which leads to the following bound:
\begin{align}\notag
            \mathbb{E}\Big[f({\boldsymbol{w}}^{(r+1)})-f({\boldsymbol{w}}^{(*)})\Big]&\leq \left(1-\eta\mu\gamma{\tau}\right) \Big[f({\boldsymbol{w}}^{(r)})-f({\boldsymbol{w}}^{(*)})\Big]+\frac{L\tau\gamma\eta^2 }{2k}\left(kL\tau\eta+{(\omega+1)}\gamma\right)\sigma^2
\end{align}
By setting $\Delta=1-\eta\mu\gamma{\tau}$ we obtain  the following bound:
\begin{align}\notag
            &\mathbb{E}\Big[f({\boldsymbol{w}}^{(R)})-f({\boldsymbol{w}}^{(*)})\Big]\nonumber\\
            \leq& \Delta^R \Big[f({\boldsymbol{w}}^{(0)})-f({\boldsymbol{w}}^{(*)})\Big]+\frac{1-\Delta^R}{1-\Delta}\frac{L\tau\gamma\eta^2 }{2k}\left(kL\tau\eta+{(\omega+1)}\gamma\right)\sigma^2\nonumber\\
            \leq& \Delta^R \Big[f({\boldsymbol{w}}^{(0)})-f({\boldsymbol{w}}^{(*)})\Big]+\frac{1}{1-\Delta}\frac{L\tau\gamma\eta^2 }{2k}\left(kL\tau\eta+{(\omega+1)}\gamma\right)\sigma^2\nonumber\\
            =&{\left(1-\eta\mu\gamma{\tau}\right)}^R \Big[f({\boldsymbol{w}}^{(0)})-f({\boldsymbol{w}}^{(*)})\Big]+\frac{1}{\eta\mu\gamma{\tau}}\frac{L\tau\gamma\eta^2 }{2k}\left(kL\tau\eta+{(\omega+1)}\gamma\right)\sigma^2
\end{align}
\end{proof}

\begin{corollary}
If we  let $\eta\gamma\mu\tau\leq\frac{1}{2}$, $\eta=\frac{1}{2L\left(\frac{\omega}{k}+1\right)\tau\gamma }$ and $\kappa=\frac{L}{\mu}$ the convergence error in Theorem~\ref{thm:pl-iid}, with $\gamma\geq k$ results in:

\begin{align}\notag
&\mathbb{E}\Big[f({\boldsymbol{w}}^{(R)})-f({\boldsymbol{w}}^{(*)})\Big]\nonumber\\
\leq& e^{-\eta\gamma{\mu\tau}R}\left(f(\boldsymbol{w}^{(0)})-f(\boldsymbol{w}^{(*)})\right)+\frac{1}{{\mu}}\left[\frac{1}{2} \tau L^2\eta^2\sigma^2+\left(1+\omega\right)\frac{\gamma\eta L\sigma^2}{2k}\right]\nonumber\\
\leq& e^{-\frac{R}{2\left(\frac{\omega}{k}+1\right)\kappa}}\left(f(\boldsymbol{w}^{(0)})-f(\boldsymbol{w}^{(*)})\right)+\frac{1}{{\mu}}\left[\frac{1}{2} L^2\frac{\tau\sigma^2}{L^2\left(\frac{\omega}{k}+1\right)^2\gamma^2\tau^2}+\frac{\left(1+\omega\right) L\sigma^2}{2\left(\frac{\omega}{k}+1\right)L\tau k}\right]\nonumber\\
=&O\left(e^{-\frac{R}{2\left(\frac{\omega}{k}+1\right)\kappa}}\left(f(\boldsymbol{w}^{(0)})-f(\boldsymbol{w}^{(*)})\right)+\frac{\sigma^2}{\left(\frac{\omega}{k}+1\right)^2\gamma^2\mu\tau}+\frac{\left(\omega+1\right)\sigma^2}{\mu\left(\frac{\omega}{k}+1\right) \tau k}\right)
\nonumber\\
=&O\left(e^{-\frac{R}{2\left(\frac{\omega}{k}+1\right)\kappa}}\left(f(\boldsymbol{w}^{(0)})-f(\boldsymbol{w}^{(*)})\right)+\frac{\sigma^2}{\gamma^2\mu\tau}+\frac{\left(\omega+1\right)\sigma^2}{\mu\left(\frac{\omega}{k}+1\right) \tau k}\right)
\end{align}
which indicates  that to achieve an error of $\epsilon$, we need to have $R=O\left(\left(\frac{\omega}{k}+1\right)\kappa\log\left(\frac{1}{\epsilon}\right)\right)$ and $\tau=\frac{\left(\omega+1\right)}{k\left(\frac{\omega}{k}+1\right)\epsilon}$. {Additionally, we note that if $\gamma\rightarrow\infty$, yet $R=O\left(\left(\frac{\omega}{k}+1\right)\kappa\log\left(\frac{1}{\epsilon}\right)\right)$ and $\tau=\frac{\left(\omega+1\right)}{k\left(\frac{\omega}{k}+1\right)\epsilon}$ will be necessary.}
\end{corollary}

\newpage
\subsection{Main result for the general convex setting}

\begin{theorem}[Convex]\label{thm:cvx-iid}
 For a general convex function $f(\boldsymbol{w})$ with optimal solution $\boldsymbol{w}^{(*)}$, using  \texttt{FedSKETCH}$(\tau, \eta, \gamma)$ to optimize $\tilde{f}(\boldsymbol{w},\phi)=f(\mathbf{\boldsymbol{w}})+\frac{\phi}{2}\left\|\boldsymbol{w}\right\|^2$,  for all $0\leq t\leq R\tau-1$,  under Assumptions \ref{Assu:1} to \ref{Assu:1.5}, if the learning rate satisfies
 \begin{align}\notag
   1\geq {\tau^2 L^2\eta^2}+\left(\frac{\omega}{k}+1\right){\eta\gamma L}{\tau}
\end{align}
and if the all the models initiate with $\boldsymbol{w}^{(0)}$, with $\phi=\frac{1}{\sqrt{k\tau}}$ and $\eta=\frac{1}{2L\gamma\tau\left(1+\frac{\omega}{k}\right)}$ we obtain:
\begin{align}
        \mathbb{E}\Big[f({\boldsymbol{w}}^{(R)})-f({\boldsymbol{w}}^{(*)})\Big]&\leq e^{-\frac{ R}{2L\left(1+\frac{\omega}{k}\right) \sqrt{m\tau}}}\left(f(\boldsymbol{w}^{(0)})-f(\boldsymbol{w}^{(*)})\right)\nonumber\\
        &\qquad +\left[\frac{\sqrt{k}\sigma^2}{8\sqrt{\tau}\gamma^2\left(1+\frac{\omega}{k}\right)^2} +\frac{\left(\omega+1\right)\sigma^2}{4\left(\frac{\omega}{k}+1\right)\sqrt{k\tau}} \right] +\frac{1}{2\sqrt{k\tau}}\left\|\boldsymbol{w}^{(*)}\right\|^2\label{eq:cvx-iid}
\end{align}{{}}
\end{theorem}
We note that above theorem implies that to achieve a convergence error of $\epsilon$ we need to have $R=O\left(L\left(1+\frac{\omega}{k}\right)\frac{1}{\epsilon}\log\left(\frac{1}{\epsilon}\right)\right)$ and $\tau=O\left(\frac{\left(\omega+1\right)^2}{k\left(\frac{\omega}{k}+1\right)^2\epsilon}\right)$.

\begin{proof}
Since $\tilde{f}(\boldsymbol{w}^{(r)},\phi)=f(\boldsymbol{w}^{(r)})+\frac{\phi}{2}\left\|\boldsymbol{w}^{(r)}\right\|^2$ is $\phi$-PL, according to Theorem~\ref{thm:pl-iid}, we have:
\begin{align}
   & \tilde{f}(\boldsymbol{w}^{(R)},\phi)-\tilde{f}(\boldsymbol{w}^{(*)},\phi)\nonumber\\
   =&{f}(\boldsymbol{w}^{(r)})+\frac{\phi}{2}\left\|\boldsymbol{w}^{(r)}\right\|^2-\left({f}(\boldsymbol{w}^{(*)})+\frac{\phi}{2}\left\|\boldsymbol{w}^{(*)}\right\|^2\right)\nonumber\\
    \leq& \left(1-\eta\gamma{\phi\tau}\right)^R\left(f(\boldsymbol{w}^{(0)})-f(\boldsymbol{w}^{(*)})\right)+\frac{1}{{\phi}}\left[\frac{1}{2} L^2\tau\eta^2\sigma^2+\left(1+\omega\right)\frac{\gamma\eta L\sigma^2}{2k}\right]\label{eq:mid-cvx}
\end{align}
Next rearranging Eq.~(\ref{eq:mid-cvx}) and replacing $\mu$ with $\phi$ leads to the following error bound:
\begin{align}\notag
  &  {f}(\boldsymbol{w}^{(R)})-f^*\\\notag
  \leq& \left(1-\eta\gamma{\phi\tau}\right)^R\left(f(\boldsymbol{w}^{(0)})-f(\boldsymbol{w}^{(*)})\right)+\frac{1}{{\phi}}\left[\frac{1}{2} L^2\tau\eta^2\sigma^2+\left(1+\omega\right)\frac{\gamma\eta L\sigma^2}{2k}\right]\\\notag
  &\qquad +\frac{\phi}{2}\left(\left\|\boldsymbol{w}^*\right\|^2-\left\|\boldsymbol{w}^{(r)}\right\|^2\right)\\\notag
    \leq& e^{-\left(\eta\gamma{\phi\tau}\right)R}\left(f(\boldsymbol{w}^{(0)})-f(\boldsymbol{w}^{(*)})\right)+\frac{1}{{\phi}}\left[\frac{1}{2} L^2\tau\eta^2\sigma^2+\left(1+\omega\right)\frac{\gamma\eta L\sigma^2}{2k}\right] +\frac{\phi}{2}\left\|\boldsymbol{w}^{(*)}\right\|^2
\end{align}
Next, if we set $\phi=\frac{1}{\sqrt{k\tau}}$ and $\eta=\frac{1}{2\left(1+\frac{\omega}{k}\right)L\gamma \tau}$, we obtain that
\begin{align}\notag
        &{f}(\boldsymbol{w}^{(R)})-f^*\\\notag
        \leq& e^{-\frac{R}{2\left(1+\frac{\omega}{k}\right)L \sqrt{m\tau}}}\left(f(\boldsymbol{w}^{(0)})-f(\boldsymbol{w}^{(*)})\right)+\sqrt{k\tau}\left[\frac{\sigma^2}{8\tau\gamma^2\left(1+\frac{\omega}{k}\right)^2} +\frac{\left(\omega+1\right)\sigma^2}{4\left(\frac{\omega}{k}+1\right)\tau k}\right] +\frac{1}{2\sqrt{k\tau}}\left\|\boldsymbol{w}^{(*)}\right\|^2 ,
\end{align}
thus the proof is complete.
\end{proof}

\newpage

\newpage
\section{Proof of Main Theorems}
The proof of Theorem~\ref{thm:homog_case} follows directly from the results in~\cite{haddadpour2020federated}. For the sake of the completeness we review an assumptions from this reference for the quantization with their notation.

\begin{assumption}[\cite{haddadpour2020federated}]\label{Assu:quant}
The output of the compression operator $Q(\mathbf{x})$ is an unbiased estimator of its input $\mathbf{x}$, and its variance grows with the squared of the squared of $\ell_2$-norm of its argument, i.e., $\mathbb{E}\left[Q(\mathbf{x})\right]=\mathbf{x}$ and $\mathbb{E}\left[\left\|Q(\mathbf{x})-\mathbf{x}\right\|^2\right]\leq \omega\left\|\mathbf{x}\right\|^2$ .
\end{assumption}

\subsection{Proof of Theorem~\ref{thm:homog_case}}
Based on Assumption~\ref{Assu:quant} we have:
\begin{theorem}[\cite{haddadpour2020federated}]\label{thm:fromhaddad}
 Consider \texttt{FedCOM} in \cite{haddadpour2020federated}. Suppose that the conditions in Assumptions~\ref{Assu:1}, \ref{Assu:1.5} and \ref{Assu:quant} hold. If the local data distributions of all users are identical (homogeneous setting), then we have
 \begin{itemize}
     \item \textbf{Nonconvex:}  By choosing stepsizes as $\eta=\frac{1}{L\gamma}\sqrt{\frac{p}{R\tau\left(\frac{\omega}{p}+1\right)}}$ and $\gamma\geq p$, the sequence of iterates satisfies  $\frac{1}{R}\sum_{r=0}^{R-1}\left\|\nabla f({\boldsymbol{w}}^{(r)})\right\|_2^2\leq {\epsilon}$ if we set
     $R=O\left(\frac{1}{\epsilon}\right)$ and $ \tau=O\left(\frac{\frac{\omega}{p}+1}{{p}\epsilon}\right)$.
     \item \textbf{Strongly convex or PL:}
      By choosing stepsizes as $\eta=\frac{1}{2L\left(\frac{\omega}{p}+1\right)\tau\gamma}$ and $\gamma\geq m$, we obtain that the iterates satisfy $\mathbb{E}\Big[f({\boldsymbol{w}}^{(R)})-f({\boldsymbol{w}}^{(*)})\Big]\leq \epsilon$ if  we set
     $R=O\left(\left(\frac{\omega}{p}+1\right)\kappa\log\left(\frac{1}{\epsilon}\right)\right)$ and $ \tau=O\left(\frac{1}{p\epsilon}\right)$.
     \item \textbf{Convex:} By choosing stepsizes as $\eta=\frac{1}{2L\left(\frac{\omega}{p}+1\right)\tau\gamma}$ and $\gamma\geq p$, we obtain that the iterates satisfy $ \mathbb{E}\Big[f({\boldsymbol{w}}^{(R)})-f({\boldsymbol{w}}^{(*)})\Big]\leq \epsilon$ if we set
     $R=O\left(\frac{L\left(1+\frac{\omega}{p}\right)}{\epsilon}\log\left(\frac{1}{\epsilon}\right)\right)$ and $ \tau=O\left(\frac{1}{p\epsilon^2}\right)$.
 \end{itemize}
\end{theorem}

\begin{proof}
Since the sketching \texttt{PRIVIX} and \texttt{HEAPRIX}, satisfy Assumption~\ref{Assu:quant} with $\omega=\mu^2d$ and $\omega=\mu^2d-1$ respectively with probability $1-\delta$.  Therefore, all the results in Theorem~\ref{thm:homog_case}, conclude from Theorem~\ref{thm:fromhaddad} with probability $1-\delta$ and plugging $\omega=\mu^2d$ and $\omega=\mu^2d-1$ respectively into the corresponding convergence bounds.
\end{proof}

\subsection{Proof of Theorem~\ref{thm:hetreg_case}}
For the heterogeneous setting, the results in~\cite{haddadpour2020federated} requires the following extra assumption that naturally holds for the sketching:

\begin{assumption}[\cite{haddadpour2020federated}]\label{assum:009}
The compression scheme $Q$ for the heterogeneous data distribution setting satisfies the following condition $
    \mathbb{E}_Q[\|\frac{1}{m}\sum_{j=1}^m Q(\boldsymbol{x}_j)\|^2-\|Q(\frac{1}{m}\sum_{j=1}^m \boldsymbol{x}_j)\|^2]\leq G_q$.
\end{assumption}
We note that since sketching is a linear compressor, in the case of our algorithms for heterogeneous setting we have $G_q=0$.

\newpage

Next, we restate the Theorem in \cite{haddadpour2020federated} here as follows:

\begin{theorem}\label{thm:fromhaddad-het}
 Consider \texttt{FedCOMGATE} in \cite{haddadpour2020federated}. If Assumptions~\ref{Assu:1}, \ref{Assu:2}, \ref{Assu:quant}  and \ref{assum:009} hold, then even for the case the local data distribution of users are different  (heterogeneous setting) we have
 \begin{itemize}
     \item \textbf{Nonconvex:} By choosing stepsizes as $\eta=\frac{1}{L\gamma}\sqrt{\frac{p}{R\tau\left(\omega+1\right)}}$ and $\gamma\geq p$, we obtain that the iterates satisfy  $\frac{1}{R}\sum_{r=0}^{R-1}\left\|\nabla f({\boldsymbol{w}}^{(r)})\right\|_2^2\leq \epsilon$ if we set
     $R=O\left(\frac{\omega+1}{\epsilon}\right)$ and $ \tau=O\left(\frac{1}{p\epsilon}\right)$.
     \item \textbf{Strongly convex or PL:}
      By choosing stepsizes as $\eta=\frac{1}{2L\left(\frac{\omega}{p}+1\right)\tau\gamma}$ and ${\gamma\geq \sqrt{p\tau}}$, we obtain that the iterates satisfy $\mathbb{E}\Big[f({\boldsymbol{w}}^{(R)})-f({\boldsymbol{w}}^{(*)})\Big]\leq \epsilon$ if we set
      $R=O\left(\left(\omega+1\right)\kappa\log\left(\frac{1}{\epsilon}\right)\right)$ and $ \tau=O\left(\frac{1}{p\epsilon}\right)$.
     \item \textbf{Convex:}  By choosing stepsizes as $\eta=\frac{1}{2L\left(\omega+1\right)\tau\gamma}$ and ${\gamma\geq \sqrt{p\tau}}$, we obtain that the iterates satisfy $\mathbb{E}\Big[f({\boldsymbol{w}}^{(R)})-f({\boldsymbol{w}}^{(*)})\Big]\leq \epsilon$ if we set
     $R=O\left(\frac{L\left(1+\omega\right)}{\epsilon}\log\left(\frac{1}{\epsilon}\right)\right)$ and $ \tau=O\left(\frac{1}{p\epsilon^2}\right)$.
 \end{itemize}

\end{theorem}
\begin{proof}
Since the sketching methods \texttt{PRIVIX} and \texttt{HEAPRIX}, satisfy the Assumption~\ref{Assu:quant} with $\omega=\mu^2d$ and $\omega=\mu^2d-1$ respectively with probablity $1-\delta$, we conclude the proofs of Theorem~\ref{thm:hetreg_case} using Theorem~\ref{thm:fromhaddad-het} with probability $1-\delta$ and plugging $\omega=\mu^2d$ and $\omega=\mu^2d-1$ respectively into the convergence bounds.
\end{proof}

\end{document}